\documentclass{article}
\usepackage[utf8]{inputenc} %
\usepackage[T1]{fontenc}    %
\usepackage{fullpage}

\usepackage[round]{natbib}

\usepackage[margin = 1in]{geometry}

\usepackage{url}            %
\usepackage{booktabs}       %
\usepackage{amsfonts}       %
\usepackage{microtype}      %
\usepackage{wrapfig}
\usepackage{caption}
\usepackage[mathscr]{euscript}
\usepackage{epstopdf}

\usepackage{times}
\usepackage{algorithm}
\usepackage{adjustbox}
\usepackage[noend]{algpseudocode}
\usepackage{amsthm}
\usepackage{graphicx}
\usepackage{latexsym}
\usepackage{mathtools}
\usepackage{multirow}
\usepackage{paralist}
\usepackage{minitoc} %
\usepackage{xspace}
\usepackage{enumitem}
\usepackage{thm-restate}

\usepackage[dvipsnames]{xcolor} %
\usepackage{tikz}
\usetikzlibrary{shapes}
\usetikzlibrary{positioning}
\usetikzlibrary{plotmarks}
\usetikzlibrary{patterns}
\usetikzlibrary{intersections,shapes.arrows}
\definecolor{darkpink}{rgb}{0.91, 0.33, 0.5}

\definecolor{puorange}{rgb}{0.80,0.20,0}
\definecolor{bluegray}{rgb}{0.04,0,0.7}
\definecolor{greengray}{rgb}{0.05,0.50,0.15}
\definecolor{darkbrown}{rgb}{0.40,0.2,0.05}
\definecolor{darkcyan}{rgb}{0,0.4,1}
\definecolor{black}{rgb}{0,0,0}
\definecolor{grey}{rgb}{0.93,0.93,0.93}
\definecolor{royalazure}{rgb}{0.0, 0.22, 0.66}

\usepackage[colorlinks,citecolor=blue,linkcolor=blue,urlcolor=blue,breaklinks]{hyperref}

\usepackage{cleveref}

\crefname{section}{Sec.}{Sections}
\crefname{appendix}{Appx.}{Appxs}

\crefname{theorem}{Thm.}{Thms.}
\crefname{lemma}{Lem.}{Lems.}
\crefname{lem}{Lem.}{Lems.}
\crefname{corollary}{Cor.}{Cors.}
\crefname{proposition}{Prop.}{Props.}
\crefname{prop}{Prop.}{Props.}
\crefname{assumption}{Asm.}{Asms.}
\crefname{asm}{Asm.}{Asms.}

\crefname{algorithm}{Alg.}{Algs.}
\Crefname{algorithm}{Algorithm}{Algorithms}
\crefname{figure}{Fig.}{Figs.}
\crefname{table}{Tab.}{Tabs.}

\newif\ifcomments

\usepackage{color}
\definecolor{lightblue}{RGB}{0,160,200}
\definecolor{gray}{rgb}{0.4,0.4,0.4}
\usepackage{listings}
\newenvironment{proofsketch}{%
  \proof}{\endproof}

\newif\ifnew

\newcommand{\new}[1]{
\ifnew
{\color{teal}#1}~\color{black}
\else
#1
\fi
}

\newcommand{\myparagraph}[1]{\paragraph{#1.}\hspace{-0.8em}}  %

\renewcommand{\cite}{\citep} %
\usepackage{bbm, mathtools}

\newcommand \reals {\mathbb{R}}

\newcommand \T {^{\top}}	%
\newcommand{\ones}{\mathbf{1}}
\newcommand \eps \epsilon
\newcommand{\pow}[1]{^{(#1)}}

\newcommand \expect {\mathbb{E}}
\newcommand{\Var}[1]{\operatorname{Var}\left[#1\right]}

\newcommand \prob {{\mathbb{P}}}

\DeclarePairedDelimiterX{\inp}[2]{\langle}{\rangle}{#1, #2} 
\newcommand{\norm}[1]{\left\Vert #1 \right\Vert} 
\newcommand{\normsq}[1]{\left\Vert #1 \right\Vert^2} 

\newcommand{\abs}[1]{\left\lvert #1 \right\rvert}

\newcommand \argmin {\operatorname*{arg\,min}} %
\newcommand \argmax {\operatorname*{arg\,max}} %
\newcommand \conv {\operatorname*{conv}} %

\newcommand \grad {\nabla}

\newcommand \D {\mathrm{d}}

\newcommand \Fcal {\mathcal F}

\newcommand \Pcal {\mathcal P}

\definecolor{cadmiumgreen}{rgb}{0.0, 0.42, 0.24}
\definecolor{harvardcrimson}{rgb}{0.79, 0.0, 0.09}

\newtheorem{theorem}{Theorem}
\newtheorem{lem}[theorem]{Lemma}

\newtheorem{property}[theorem]{Property}
\newtheorem{proposition}[theorem]{Proposition}
\newtheorem{prop}[theorem]{Proposition}
\newtheorem{corollary}[theorem]{Corollary}
\newtheorem{lemma}[theorem]{Lemma}

\newtheorem{assumption}[theorem]{Assumption}

\theoremstyle{definition}

\newtheorem{asm}{Assumption}[section]

\newcommand{\R}{\mathbb{R}}
\renewcommand{\L}{\mathbb{L}}

\renewcommand{\d}{\ensuremath{\, \text{d}}}

\newcommand{\ceil}[1]{\ensuremath{\left\lceil#1\right\rceil}}

\newcommand{\p}[1]{\ensuremath{\left(#1\right)}}
\newcommand{\br}[1]{\ensuremath{\left\{#1\right\}}}
\newcommand{\sbr}[1]{\ensuremath{\left[#1\right]}}

\newcommand{\sse}{\subseteq}
\renewcommand{\P}[2]{\ensuremath{{\mathbb P}_{#1}\left[#2\right]}}
\newcommand{\E}[2]{\ensuremath{{\mathbb E}_{#1}\left[#2\right]}}
\newcommand{\I}[2]{\ensuremath{{\mathbbm 1}_{#1}\left(#2\right)}}
\newcommand{\mc}[1]{\ensuremath{\mathcal{#1}}}
\newcommand{\mbb}[1]{\ensuremath{\mathbb{#1}}}
\newcommand{\msc}[1]{\ensuremath{\mathscr{#1}}}

\newcommand{\ip}[1]{\ensuremath{\left\langle #1 \right\rangle}}

\newcommand{\argsort}{\operatorname{argsort}}

\newcommand{\osvrg}{LSVRG\xspace}

\newcommand{\lstat}{h}
\newcommand{\dualvar}{\lambda}
\newcommand{\primobj}{\msc{R}_\sigma}
\newcommand{\mbobj}{\msc{R}_{\hat{\sigma}}}
\newcommand{\primobjreg}{\msc{R}_{\sigma, \mu}}

\newcommand{\W}{\msc{W}}
\newcommand{\surrobj}{\bar{\msc{R}}_{\hat{\sigma}}}
\newcommand{\surrobjreg}{\bar{\msc{R}}_{\hat{\sigma}, \mu}}
\newcommand{\primobjsmooth}[1]{\msc{R}_{\sigma, #1}}
\newcommand{\saddleobj}{\Phi}
\newcommand{\saddleobjsmooth}[1]{\Phi_{#1}}
\newcommand{\strongcvx}{\mu}
\newcommand{\avg}{{u_n}}
\newcommand{\lossval}{l}
\newcommand{\lossvalaux}{z}
\newcommand{\bregmean}{\operatorname{Avg}}
\newcommand{\pav}{\operatorname{PAV}}

\newcommand{\lse}{\operatorname{LSE}}

\newcommand{\Binom}{\operatorname{Binom}}

\newcommand{\stepsize}{\eta}
\newcommand{\rv}{Z}
\newcommand{\dat}{D}
\newcommand{\reg}{\mu}
\newcommand{\regobj}{\msc{R}_{\sigma, \mu}}

\newcommand\blfootnote[1]{%
  \begingroup
  \renewcommand\thefootnote{}\footnote{#1}%
  \addtocounter{footnote}{-1}%
  \endgroup
} 
\doparttoc
\faketableofcontents

\title{Stochastic Optimization for Spectral Risk Measures}
\author{Ronak Mehta$^1$ \qquad Vincent Roulet$^1$ \qquad Krishna Pillutla$^{2\dagger}$ \qquad Lang Liu$^1$ \qquad Zaid Harchaoui$^1$ \vspace{0.3cm} \\
$^1$Department of Statistics, University of Washington\\
$^2$Paul G. Allen School of Computer Science \& Engineering, University of Washington}

\begin{document}

\maketitle
\blfootnote{$^\dagger$ Now at Google Research.}

\begin{abstract}
  Spectral risk objectives -- also called $L$-risks -- allow for learning systems to interpolate between optimizing average-case performance (as in empirical risk minimization) and worst-case performance on a task. We develop stochastic algorithms to optimize these quantities by characterizing their subdifferential and addressing challenges such as biasedness of subgradient estimates and non-smoothness of the objective. We show theoretically and experimentally that out-of-the-box approaches such as stochastic subgradient and dual averaging are hindered by bias and that our approach outperforms them. 
\end{abstract}

\section{Introduction} \label{sec:intro}
A cornerstone of machine learning is the {\it empirical risk minimization (ERM)} problem, written
\begin{align}
    \min_{w \in \R^d} \left[ \msc{R}(w) := \frac{1}{n} \sum_{i=1}^n  \ell_{i}(w) \right],
    \label{eqn:erm}
\end{align}
where $\ell_i(w)$ quantifies loss on training example $i$ using a model with weights $w \in \R^d$. The objective~\eqref{eqn:erm} represents an often unquestioned modeling choice: to summarize $\ell_1(w), ..., \ell_n(w)$, the empirical distribution of losses, using its average. At first glance, this is a natural summary, inheriting both the statistical amenability of the sample mean \cite{Shalev-Shwartz2014Understanding} and the wide arsenal of optimization algorithms designed specifically for finite sum objectives \cite{Roux2012AStochastic, Defazio2014SAGA, Johnson2013Accelerating, Reddi2016FastIncremental}.
However, as modern learning systems are deployed in critical domain applications such as energy planning \cite{Guigues2013Risk-averse}, materials engineering \cite{Yeh2006Analysis}, and financial regulation \cite{He2022RiskMeasures}, safe and reliable performance in ``worst-case'' scenarios is paramount.

This imperative can be modeled by alternate {\it risk measures} (statistical functionals of the loss distribution), particularly those that encapsulate the behavior of the distribution's upper tail. We investigate objectives of the form
\begin{align} 
    \min_{w \in \R^d} \left[ \msc{R}_\sigma(w) := \sum_{i=1}^n \sigma_i \ell_{(i)}(w) \right],
    \label{eqn:orm}
\end{align}
where $\ell_{(1)}(w) \leq ... \leq \ell_{(n)}(w)$ are the {\it order statistics} of the empirical loss distribution, and $0 \leq \sigma_1 \leq \cdots \leq \sigma_n \leq 1$ is a sequence of non-decreasing weights satisfying $\sum_{i=1}^n \sigma_i = 1$, called the {\it spectrum} of $\msc{R}_\sigma$. 

The expression \eqref{eqn:orm} is called an {\it $L$-risk} \cite{Shorack2017Probability, Maurer2020RobustUnsupervised} in statistics and a {\it spectral risk measure} in economics and finance \cite{He2022RiskMeasures, Holland2021Spectralrisk}. The $\sigma_i$'s allow the practitioner to interpolate between the average-case ($\sigma_i = {1}/{n} \ \forall i$) and worst-case ($\sigma_n = 1$) performance on the training set. Such objectives have garnered a flurry of recent interest in machine learning \cite{Fan2017Learningwith,williamson2019fairness, Khim2020uniformConvergence, Maurer2020RobustUnsupervised, Holland2021Spectralrisk, Liu2019OnHuman, Lee2020LearningBounds, kawaguchi2020ordered}. 

Despite their increasing adoption, however, optimization approaches have relied on using the full-batch or stochastic subgradient method out-of-the-box \cite{Fan2017Learningwith, kawaguchi2020ordered, Laguel2020First-Order, Levy2020Large-Scale}, both enduring considerable limitations.
The per-iteration complexity of full-batch methods is $O(n)$ function/gradient evaluations and $O(n\log n)$ elementary operations (as we discuss in  \Cref{prop:cvxity}). For stochastic%
\footnote{
    We use the term ``stochastic'' to include both \emph{streaming} algorithms in which fresh samples from the data-generating distribution are provided at each iterate, and \emph{incremental} algorithms, in which multiple passes are made over a fixed dataset.
} 
variants, unbiased estimates of any subgradient, while needing only $O(1)$ gradient evaluations, still need $O(n)$ function calls and $O(n \log n)$ elementary operations. This yields practically the same per-iteration complexity as the full-batch method, inspiring methods that abandon convergence to the minimal $L$-risk altogether and resort to biased stochastic subgradient updates that use $O(1)$ function value and gradient calls per-iteration~\cite{kawaguchi2020ordered, Levy2020Large-Scale}. 

In remains an open question whether there exist optimization algorithms that converge to the minimum spectral risk while needing only $O(1)$ gradient calls per iteration; in this work, we answer the question in the affirmative. In \Cref{sec:theory}, we give a consistency result that relates empirical $L$-risks to their population counterpart. In \Cref{sec:optim}, we characterize the subdifferential and continuity properties of $L$-risks as a function of the underlying losses, quantify the bias of current stochastic approaches and propose a linearly convergent $L$-risk minimization algorithm requiring only $O(1)$ function/gradient evaluations and $O(\log n)$ elementary operations per iteration. Finally, we demonstrate superior convergence of the method experimentally via numerical evaluations in \Cref{sec:experiments}, with concluding remarks in \Cref{sec:discussion}.
 
\myparagraph{Related work} \label{sec:related_work}
Risk measures have been studied extensively in quantitative finance~\cite{Artzner1999Coherent, Follmer2002Convexmeasures, Rockafeller2013Thefundamental,acerbi2002coherence,pflug2005measuring,kuhn2019wasserstein}, convex analysis \cite{Rockafeller2014RandomVariables, Ben2007AnOld}, and distributionally robust learning~\cite{sarykalin2008value, Guigues2013Risk-averse, Fan2017Learningwith,hu2018does,lee2018minimax,duchi2019variance, Laguel2020First-Order,chen2020distributionally,li2020tilted}. We refer to \citet{He2022RiskMeasures} for a review of the axiomatic theory of risk measures and \citet[Chap. 6]{Shapiro2014Lectures} for applications to optimization. 

A number of recent works study $L$-risks, with a focus on statistical properties. The works \citet{Khim2020uniformConvergence} and \citet{Maurer2020RobustUnsupervised} provide classical statistical learning theoretic bounds for $L$-risk objectives and the latter focuses on unsupervised tasks like clustering.  
\citet{Holland2021Spectralrisk} presents a derivative-free learning procedure for general $L$-risk problems in the fully stochastic/streaming setting.
A particular spectral risk measure called the {\it superquantile} or {\it conditional value-at-risk (CVaR)}, has recently received careful attention in the learning setting~\cite{curi2020adaptive, Levy2020Large-Scale, Laguel2020First-Order, Laguel2022Superquantiles}. Other non-spectral risk functionals include CPT measures (which can be thought of as nonconvex $L$-risks) and OCE measures \cite{Liu2019OnHuman, Lee2020LearningBounds}.

\citet{Fan2017Learningwith} and \citet{kawaguchi2020ordered} study batch and stochastic optimization algorithms respectively for the ``average top-$k$" loss, which is exactly equivalent to the superquantile. We instead focus on developing incremental algorithms, as in \citet{Mairal2014Incremental, Roux2012AStochastic, Defazio2014SAGA, Shalev-Shwartz2013Stochastic, Johnson2013Accelerating} for ERM. We aim to find algorithms that operate on non-smooth objectives, a fixed training set and require only a constant number of function value and gradient computations per iterate. 
\section{Spectral Risk Measures}\label{sec:theory}
In this section, we relate the empirical quantity~\eqref{eqn:orm} to its population counterpart, justifying its use as an estimator for $n$ sufficiently large. To achieve this, we will write spectral/$L$-risks as functionals of an empirical cumulative distribution function (CDF), and show that it consistently estimates the value of the same functional applied to a population CDF.

\myparagraph{Notation}
Let $\{\dat_1, \dots, \dat_n\}$ be an i.i.d.~sample from a distribution $\prob$ over a sample space $\msc{\dat}$.
Let $\ell: \R^d \times \msc{\dat} \rightarrow \R$ be a loss function consuming model weights $w \in \R^d$ and $\msc{\dat}$-valued training example $\dat$ (e.g., a feature-label pair). We denote the training loss as $\ell_i(w) := \ell(w, \dat_i)$ for short.
Let $\rv_i := \ell(w, \dat_i)$ for $i \in \{1, \dots, n\}$.
It follows that $\{\rv_1, \dots, \rv_n\}$ is a real-valued i.i.d.~sample whose CDF is denoted by $F$, and
\begin{align} \label{eq:orm_x}
    \msc{R}_\sigma(w) = \sum_{i=1}^n \sigma_i \ell_{(i)}(w) = \sum_{i=1}^n \sigma_i \rv_{(i)},
\end{align}
where $\rv_{(1)} \leq ... \leq \rv_{(n)}$ are order statistics of $\{\rv_i\}_{i=1}^n$.

We describe subsequent results as if $\{\rv_i\}_{i=1}^n$ are arbitrary real-valued random variables drawn i.i.d.~from CDF $F$, keeping in mind that in our case, these refer to losses on data instances $\dat_i$ under parameters $w$. 

\myparagraph{$L$-functional}
We rewrite the spectral risk \eqref{eq:orm_x} as an \emph{$L$-functional} of the CDF.
Let $F_n(z):= \frac{1}{n} \sum_{i=1}^n \I{}{\rv_i \leq z}$ denote the (random) empirical CDFs of the losses. We also define the empirical {\it quantile function} (or inverse CDF) as $F_n^{-1}(t) := \inf\{z: F_n(z) \geq t\}$ for $t \in (0, 1)$.
The population quantile function is defined similarly as $F^{-1}(t) := \inf\{z: F(z) \geq t\}$.
The quantile functions are always well-defined.
The empirical quantile function can be written in terms of the order statistics as $F_n^{-1}\p{t} = \rv_{\p{\ceil{nt}}}$.
The empirical CDF and quantile function of the losses are plotted in \Cref{fig:cdf_quantile} (top left). Notice in particular that when $t \in \left(\frac{i-1}{n}, \frac{i}{n}\right)$, we have that $F_n^{-1}(t) = \rv_{(i)}$, where end-points are chosen to make $F_n^{-1}$ left continuous.
\begin{figure}[t!]
    \centering
    \adjincludegraphics[width=0.375\textwidth,trim={0 20pt 0 20pt},clip=true]{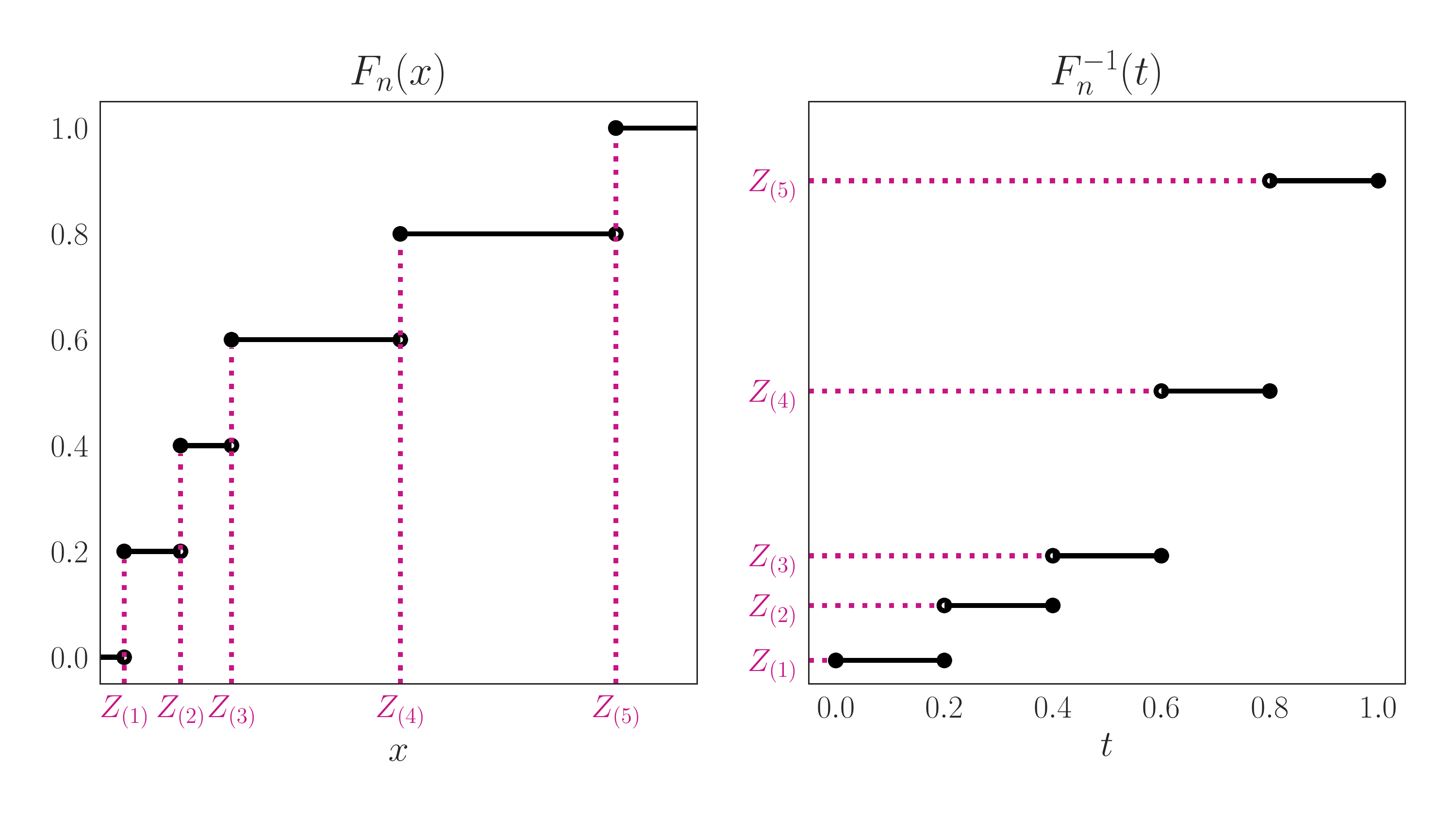}
    \adjincludegraphics[width=0.57\textwidth,trim={120pt 0 100pt 0},clip=true]{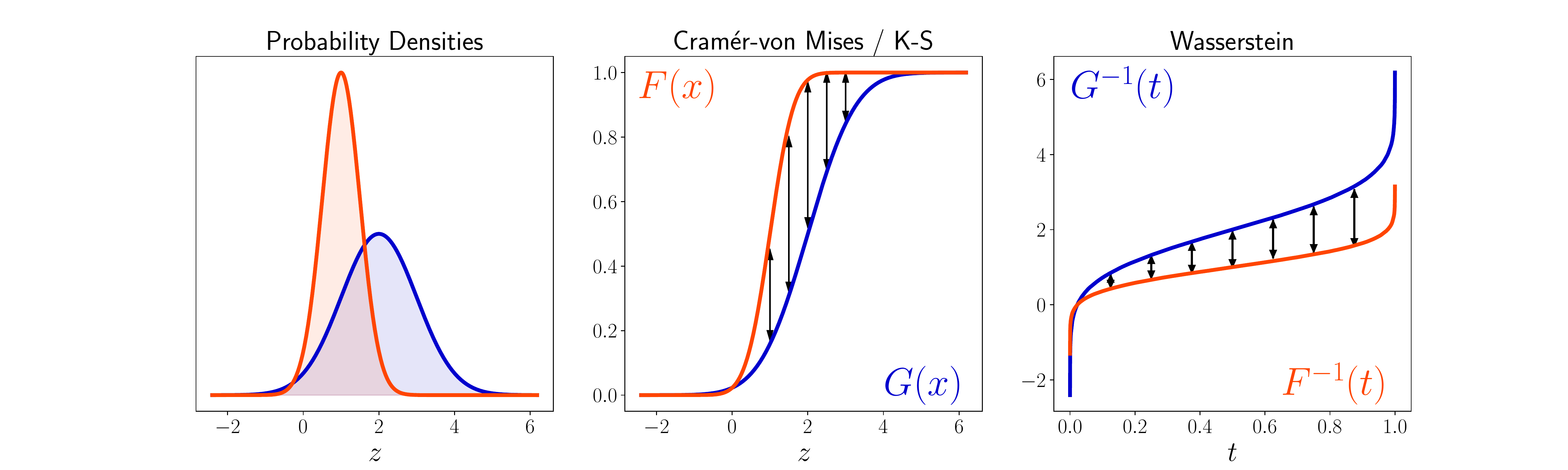}
    \adjincludegraphics[width=\textwidth, trim={50pt 0 50pt 0}, clip=true]{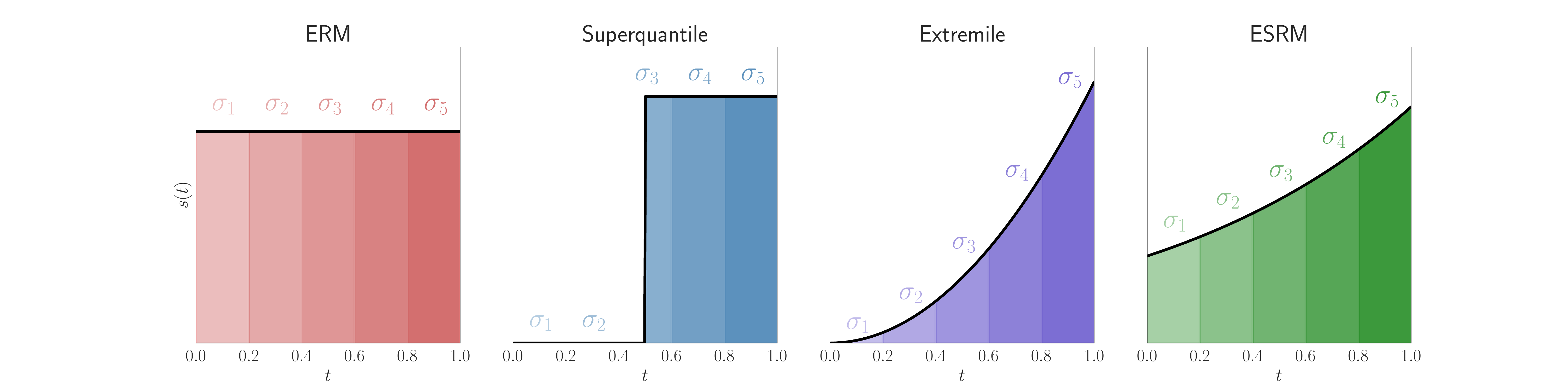}
    \caption{{\bf Top Left:} Empirical CDF $F_n$ and quantile function $F_n^{-1}$ of $\rv_1, \cdots, \rv_5$. {\bf Top Right: }Comparison of two distribution in CDFs ($F$ and $G$) as well as quantile functions ($F^{-1}$ and $G^{-1}$). {\bf Bottom:} Continuous spectra $s(t)$ and their discretization ($\sigma_1, \ldots, \sigma_5$) for various risk measures.}
    \label{fig:cdf_quantile}
\end{figure}

The spectrum $\sigma$ of a spectral risk is typically defined as a discretization of a probability density $s$ on $[0, 1]$, such that $\sigma_i = \int_{(i-1)/n}^{i/n} s(t) \d t$. Examples of spectra for various risk measures are shown in \Cref{fig:cdf_quantile} (bottom), in which the value of $\sigma_i$ is equal to the area of the shaded region immediately under it. The associated formulae are in \Cref{tab:spectra}. 
  The superquantile with parameter $q \in (0, 1)$ has enjoyed much attention in quantitative finance and more recently, machine learning \cite{Laguel2022Superquantiles},  the \emph{extremile} with parameter $r\geq 1$ has been introduced by~\cite{Daouia2019Extremiles} as an alternative risk measure, and the \emph{exponential spectral risk measure (ESRM)} with parameter $\rho> 0$ is a risk-aversion model used in futures clearinghouse margin requirements \cite{Cotter2006Extreme}.
Given both the construction of $s$ and $F_n^{-1}$ we can  rewrite the spectral risk \eqref{eq:orm_x} as 
\begin{align*}
    \msc{R}_\sigma(w) &= \sum_{i=1}^n \sigma_i \rv_{(i)} = \sum_{i=1}^n \p{\int_{(i-1)/n}^{i/n} s(t) \d t} \rv_{(i)}\\
    &= \sum_{i=1}^n \p{\int_{(i-1)/n}^{i/n} s(t) \cdot \rv_{\p{\ceil{nt}}} \d t}\\
    &= \int_0^1 s(t) \cdot F_n^{-1}(t) \d t =: \L_s\sbr{F_n},
\end{align*}
where $\L_s\sbr{G} := \int_0^1 s(t) G^{-1}(t) \d t$ is called an {\it $L$-functional} with spectrum $s$ applied to CDF $G$. It stands to reason that this quantity converges to $\L_s\sbr{F}$ in an appropriate sense.
Our proof relies on the notion of Wasserstein distances which we briefly recall here.

\myparagraph{Wasserstein distances}
For two probability distributions $\mu$ and $\nu$ on $\R$, the $1$-Wasserstein distance $W_1(\mu, \nu)$ between $\mu$ and $\nu$ is defined by
\begin{align*}
    W_1(\mu, \nu) := \inf_{\gamma \in \Pi(\mu, \nu)} \int_\R \abs{x - y} \d \gamma(x, y),
\end{align*}
where $\Pi(\mu, \nu)$ is the set of couplings (or joint distributions) with marginals being $\mu$ and $\nu$.
It is a metric on the space of probability distributions on $\R$.
If $F$ and $G$ are the CDFs associated with $\mu$ and $\nu$, respectively,
it is known \citep[e.g.,][Thm.~2.10]{Bobkov2019OnedimensionalEM} that $W_1(\mu, \nu)$ quantifies the disagreement in either the CDF or quantile functions, i.e.,
\begin{align}
    W_1(\mu, \nu) &= \int_0^1 \abs{F^{-1}(t) - G^{-1}(t)} \d t \label{eqn:quantile_W}\\
    &= \int_{-\infty}^{+\infty} \abs{F(z) - G(z)} \d z. \label{eqn:cdf_W}
\end{align}
In contrast, other statistical divergences such as the Cramer von Mises criterion $\int_{-\infty}^\infty \abs{F_n(z) - F(z)} \d F(z)$ and the Kolmogorov-Smirnoff statistic $\sup_{z \in \R} \abs{F_n(z) - F(z)}$ only measure the disagreement in CDFs, as illustrated in \Cref{fig:cdf_quantile} (top right). The relation \eqref{eqn:quantile_W} is used to prove the upcoming \Cref{prop:sgd}, whereas \eqref{eqn:cdf_W} helps establish the consistency result below.
\begin{restatable}{prop}{consistency}
    Assume that $\expect\abs{\rv}^p < \infty$ for some $p > 2$ and that $\norm{s}_\infty := \sup_{t \in (0, 1)} \abs{s(t)} < \infty$. Then,
    \begin{align*}
        \expect\, {\big|{\L_s\sbr{F_n} - \L_s\sbr{F}}\big|^2} &\leq \frac{4\norm{s}^2_\infty \p{\frac{p}{p-2}}^2 \E{}{\abs{\rv}^p}^{\frac{2}{p}}}{n}.
    \end{align*}
    \label{prop:consistency}
\end{restatable}
\begin{proofsketch}
By boundedness of $s$ and \eqref{eqn:cdf_W},
\begin{align*}  
    \expect\, \big|{\L_s\sbr{F_n} - \L_s\sbr{F}\big|^2} = \expect\abs{\int_0^1 s(t) \cdot \p{F^{-1}_n(t) - F^{-1}(t)}\d t}^2 \leq \norm{s}_\infty^2 \cdot \expect\, \p{\int_{-\infty}^{+\infty} \abs{F_n(z) - F(z)} \d z}^2.
\end{align*}
Apply the triangle inequality on $L^2(\prob)$ to obtain
\begin{align*}
    \sqrt{\expect\, \p{\int_{-\infty}^{+\infty} \abs{F_n(z) - F(z)} \d z}^2} \leq \int_{-\infty}^{+\infty} \sqrt{\expect\,\abs{F_n(z) - F(z)}^2} \d z = n^{-1/2}\int_{-\infty}^{+\infty} \sqrt{F(z)(1 - F(z))} \d z,
\end{align*}
where the last step uses that for any $z \in \R$, we have $nF_n(z) \sim \Binom\p{n, F(z)}$ and compute its variance.
The remainder of the proof uses elementary concentration inequalities to bound $\int_{-\infty}^{+\infty} \sqrt{F(z)(1 - F(z))} \d z$ (see \Cref{sec:a:consistency}).
\end{proofsketch}

\Cref{prop:consistency} operates in general conditions that are of particular importance in optimization. To put this in context, a number of works provide non-asymptotic uniform learning bounds on spectral (and related) risks \cite{Maurer2020RobustUnsupervised, Khim2020uniformConvergence, Lee2020LearningBounds}. However, these approaches require boundedness of the random variable of interest, which eliminates any potential application to heavy-tailed losses. Asymptotic approaches proceed by assuming Lipschitz continuity of the spectrum $s$ \cite{SHAO1989Functional}, the trimming of $s$ (i.e. $s(t) = 0$ for all $t \in [0, \alpha) \cup (1-\alpha,1]$ with $0 < \alpha \leq 1$)  \cite{Shorack2017Probability, SHAO1989Functional}, or bounded derivatives of the population quantile function $F^{-1}$ \cite{XIANG1995Anoteon}. The $q$-superquantile does not even have a continuous spectrum, whereas the spectrum of the $r$-extremile is not Lipschitz for $1 \leq r < 2$. Because $s$ must be non-decreasing to achieve convexity (as we discuss in the upcoming \Cref{prop:cvxity}), trimming the upper tail of $s$ is not reflective of practice. Finally, because losses such as the square loss or logistic loss can grow to infinity, the derivative $F^{-1}(t)$ as $t \rightarrow \infty$ cannot be assumed to be bounded. 
\Cref{prop:consistency} only requires that the population losses satisfy a moment condition and holds without trimming or assumptions of boundedness or Lipschitz continuity on the spectrum. 

\begin{table}[t!]
    \centering
    \begin{tabular}{ccc}
    \toprule
	   \textbf{Risk} & \textbf{Spectrum} $s(t)$ & \textbf{$L$-functional} $\mathbb{L}_s(F)$ \\
        \hline
        \rule{0pt}{1.2\normalbaselineskip}
        Uniform & $1$ & $\mathbb{E}[\rv]$ \\
        $q$-Superquantile
        &  $\tfrac{\I{[q, 1]}{t}}{1-q}$ & $\mathbb{E}[\rv | \rv\geq F^{-1}(q)] $ \\
        $r$-Extremile
        & $rt^{r-1}$ & $\mathbb{E}[\max_{k = 1, ... r}\rv_k] $ \\
        $\rho$-ESRM
        & $\tfrac{\rho e^{-\rho}e^{\rho t}}{1 - e^{-\rho}} $ & N/A
        \\
        \bottomrule
    \end{tabular}
    \caption{Common spectral risk measures, with spectra $s(t)$, interpretation of the  $L$-statistics $\mathbb{L}_s(F)$ for $F$ the CDF of $\rv$. \label{tab:spectra}}
\end{table} 
\section{Stochastic Optimization Algorithms}\label{sec:optim}
We now consider the optimization of the regularized empirical $L$-risk objective
\begin{align}
   \msc{R}_\sigma(w) + \frac{\reg}{2} \|w\|_2^2 \quad \mbox{for} \  \msc{R}_\sigma(w)= \sum_{i=1}^n \sigma_i \ell_{(i)}(w).
    \label{eqn:orm2}
\end{align}
for $\reg >0$, $0 \leq \sigma_1 \leq \ldots \sigma_n\leq 1$ such that $\sum_{i=1}^n \sigma_i=1$ and $\ell_i$ convex. 

\myparagraph{Convexity and subdifferential}
As in ERM, the function $\msc{R}_\sigma$ is convex as long as each $\ell_i$ is convex, as we see next. Let $\partial f$ denote the subdifferential of a convex function $f$ and $a S_1 + b S_2 = \{a s_1 + b s_2 \, :\, s_1 \in S_1, \, s_2 \in S_2\}$ denote the Minkowski sum of sets $S_1, S_2$ with weights $a, b \in \reals$.
\begin{restatable}{prop}{cvxity} \label{prop:cvxity}
    If $\ell_1, \ldots, \ell_n$ are convex, the function $\msc{R}_\sigma$ is also convex, with subdifferential
    \[
    \partial \msc{R}_\sigma (w) = \conv \left(\bigcup_{\pi \in \argsort\p{\ell(w)}}\sum_{i=1}^n \sigma_i \partial \ell_{\pi(i)}(w) \right)\,,
    \]
    where
    $
    \argsort\p{\ell(w)} = \br{\pi: \ell_{\pi(1)}(w) \leq ... \leq \ell_{\pi(n)}(w)}.
    $
    Moreover, if each $\ell_i$ is $G$-Lipschitz continuous,
    $\msc{R}_\sigma$ is also $G$-Lipschitz continuous.
\end{restatable}
Convexity crucially relies on $\sigma_i$'s being non-decreasing. If each $\ell_i$ is differentiable, the function $\msc{R}_\sigma$ is differentiable almost everywhere, as $\argsort\p{\ell(w)}$ is a singleton at almost all $w \in \R^d$. 
The objective \new{can be non-differentiable at vectors $w\in \R^d$ leading up to ties in the losses such as  $\ell_i(w) = \ell_j(w)$ for $i \neq j$.}

\myparagraph{Computing subgradients}
\Cref{prop:cvxity} also gives us a simple recipe to retrieve some $g \in \partial \msc{R}_\sigma(w)$ with a differentiable programming framework like JAX or PyTorch: 
(i) compute the losses $\ell_i(w)$,
(ii) sort the losses to get $\ell_{\pi(1)}(w), ..., \ell_{\pi(n)}(w)$,
(iii) compute the weighted sum of the sorted losses $\sum_i \sigma_i \ell_{\pi(i)}(w)$, and
(iv) access the gradient $g = \sum_i \sigma_i \grad \ell_{\pi(i)}(w)$ at the sorting given by $\pi$ using automatic differentiation.
We can write this in PyTorch as:
\begin{lstlisting}[language=Python]
l = compute_losses(w)
l_ord = torch.sort(l)[0]
risk = torch.dot(sigmas, l_ord)
g = torch.autograd.grad(risk, w)[0]
\end{lstlisting}
The dependence of the sorting permutation $\pi$ on $w$ is not recorded in the computation graph. 
Multiple options for $\pi$ occur with probability zero if the losses are continuous random variables, though if they do, we select one arbitrarily.

\myparagraph{Stochastic subgradient method (SGD)}
\begin{algorithm}[t]
\caption{
Stochastic Subgradient Method (SGD)
}\label{alg:sgd}
\begin{algorithmic}[1]
\Require Number of iterates $T$, minibatch size $m$, learning rate sequence $(\stepsize\pow{t})_{t=1}^T$, spectrum $s$, oracles $(\ell_i)_{i=1}^n$ and $(\grad \ell_i)_{i=1}^n$, regularization $\reg >0$.
\State Initialize $w\pow{0} = 0 \in \R^d$.
\State Compute $\hat\sigma_1, ..., \hat\sigma_m$, where $\hat{\sigma_j} := \int_{(j-1)/m}^{j/m} s(t) \d t$.
\label{alg:sgd:discretize}
\For{$t = 0, ..., T - 1$}
    \State Sample without replacement $(i_1, ..., i_m){\subseteq} [n]$.\label{alg:sgd:sample}
    \State Select $\pi \in \argsort\p{\ell_{i_1}(w\pow{t}), ..., \ell_{i_m}(w\pow{t})}$.
    \label{alg:sgd:sort}
    \State Set $v_m\pow{t} = \sum_{j=1}^m \hat{\sigma}_j \nabla \ell_{i_{\pi(j)}}\p{w\pow{t}}$.
    \label{alg:sgd:gradient}
    \State Set $w\pow{t+1} = (1 - \stepsize\pow{t} \reg) w\pow{t} - \stepsize\pow{t} v_m\pow{t}$.
    \label{alg:sgd:update}
\EndFor
\State \Return $\bar w\pow{T} = \frac{1}{T}\sum_{t=0}^{T-1} w \pow t$.
\end{algorithmic}
\end{algorithm}
A baseline approach is the stochastic subgradient method displayed in \Cref{alg:sgd}; we refer to this as (minibatch) SGD for convenience. Given a minibatch size $m$, the method discretizes the spectrum $s$ into $m$ bins (line \ref{alg:sgd:discretize}) instead of $n$ (as in objective \eqref{eqn:orm2}). We then sample $m$ indices $\{i_1, ..., i_m\}$ randomly sampled from $\{1, ..., n\}$ (line \ref{alg:sgd:sample}). We retrieve a sorting permutation $\pi: [m] \rightarrow [m]$ satisfying $\ell_{i_{\pi(1)}} \leq \ldots \leq \ell_{i_{\pi(m)}}$ (line \ref{alg:sgd:sort}) and use it to compute the update direction $v_m\pow{t}$ (line \ref{alg:sgd:gradient}). 
While the per-iteration cost is $m$ gradient evaluations and $O(md)$ time complexity, \Cref{alg:sgd} can fail to minimize the true objective $\msc{R}_\sigma$ for non-uniform $s$ due to the bias of the minibatch estimate. For instance, at the extreme $m=1$, notice that $\hat{\sigma}_1 = \hat{\sigma}_m = 1$ and the subgradient estimate corresponds to $\nabla \ell_{i}(w)$ for some $i$. This is an unbiased gradient estimate of the ERM objective rather than $\msc{R}_\sigma$, reducing the algorithm to standard SGD. For non-uniform $s$, the bias can only be fully avoided at $m = n$, recovering the full batch subgradient method.

\myparagraph{SGD analysis}
To analyze (biased) minibatch SGD, we make the following standard assumptions.
\begin{asm}
    Let $\ell_1, ..., \ell_n$ be $G$-Lipschitz continuous, differentiable, convex functions defining an objective $\regobj(w)= \msc{R}_\sigma(w) + \reg \|w\|_2^2/2$ as in~\eqref{eqn:orm2} with $\sigma_i$ non-negative, non-decreasing, summing up to 1.
    \label{asm:sgd_standard}
\end{asm}
Let $u(t) := \I{(0, 1)}{t}$ be the uniform distribution on $(0, 1)$.
We have the following convergence guarantee.
\begin{restatable}{prop}{sgd}
Given \Cref{asm:sgd_standard}, the output $\bar{w}\pow{T}$ of \cref{alg:sgd} with $\stepsize\pow{t} = \frac{1}{\mu (t+1)}$ satisfies
    \begin{align*}
        \E{}{\regobj\p{\bar{w}\pow{T}}} - \regobj(w^*) \leq \underbrace{2C_s B\frac{n-m}{n}}_{\text{bias term}} + \underbrace{\frac{2 G^2 (1 + \log T)}{T}}_{\text{optimization term}}.
    \end{align*}
    where $w^* =\argmin_{w\in \reals^d} \regobj(w)$, 
    $C_s = \sup_{t \in (0, 1)} \abs{s(t) - u(t)}$, and
    $
    B = \sup_{w: \norm{w}_2 \leq G/\mu} \ell_{(n)}(w) - \ell_{(1)}(w) < \infty.
    $
    The expectation is taken over the sampling of each minibatch.
    \label{prop:sgd}
\end{restatable}
\begin{proofsketch}
    Given a minibatch $i_1, \ldots, i_m$, let $\ell_{i_{(1)}}(w) \leq \ldots \leq \ell_{i_{(m)}}(w)$ be the order statistics of the losses. Define $\mbobj(w) := \sum_{j=1}^m \hat{\sigma}_j \ell_{i_{(j)}}(w)$. Consider the surrogate objective $\surrobj(w) := \E{}{\mbobj(w) \mid w}$, where the expectation is taken over the randomness in the minibatch indices $i_1, \ldots, i_m$. We observe that the update directions $v_m\pow{t}$ of \Cref{alg:sgd} are \emph{unbiased} estimates for the gradient of $\surrobj(w\pow{t})$. For $\surrobjreg(w) = \surrobj(w) + \tfrac{\reg}{2}\norm{w}_2^2$, after enough iterations, we have
    \begin{align*}
        \surrobjreg(\bar{w}\pow{T}) \approx \min_{w \in \R^d} \surrobjreg(w)
    \end{align*}
    with error quantified by the optimization term.
    Letting $\W = \{w \in \R^d: \norm{w}_2 \leq G / \mu\}$, we also show that 
    \begin{align*}
        \surrobjreg(w) - \min_{w' \in \W} \surrobjreg(w') \approx \primobjreg(w) - \min_{w' \in \W} \primobjreg(w')
    \end{align*}
    for any $w \in \W$, quantified by the bias term. After showing that the minimizers of $\primobjreg$ and $\surrobjreg$ over $\R^d$ as well as $w\pow{T}$ are contained in $\W$, we sum the two errors to give the final result.
\end{proofsketch}
In \Cref{prop:sgd}, notice that the bias term can be reduced either by decreasing $C_s$ (by pushing $s$ closer to uniformity, hence ERM), or decreasing $(n - m)/n$ (by increasing the minibatch size). The optimization term is standard for SGD on convex, Lipschitz objectives with strongly convex regularizers. 

\begin{algorithm}[t]
\caption{
\osvrg
}\label{alg:osvrg-u:theory-main}
\begin{algorithmic}[1]
\Require Number of iterations $T$, loss functions $(\ell_i)_{i=1}^n$ and their gradient oracles, initial point ${w}\pow{0}$, learning rate $\eta$, sorting update frequency $N$, spectrum $(\sigma_i)_{i=1}^n$,  probability of checkpointing $q^*$, regularization $\reg$
\For{$t = 0, ..., T - 1$}
    \If {$t\mod N = 0$} 
    \Comment{Update weights}
        \State Select $\pi \in \argsort\p{\ell_1(w\pow{t}),\ldots, \ell_n(w\pow t}$. \label{line:osvrg_sorting_checkpoint}
        \State Update $\dualvar\pow{t} = (\sigma_{\pi^{-1}(i)})_{i=1}^n$.
        \label{line:osvrg_weights_checkpoint}
        \label{line:osvrg-dual-update}
    \Else 
        \State $\dualvar\pow{t} = \dualvar\pow{t-1}$.
    \EndIf
    \State Sample $q_t \sim \text{Unif}([0, 1])$.
    \If{$t \mod N = 0$ or $q_t \le q^*$} \Comment{Checkpoint} \label{line:osvrg_random_checkpoint}
        \State Set $\bar w\pow{t} = w\pow{t}$. \label{line:osvrg_checkpoint_var}
        \State $\bar g\pow{t} = \sum_{i=1}^n \dualvar_i\pow{t} \grad \ell_i(\bar w\pow{t}) $. \label{line:osvrg_checkpoint_grad}
    \Else
        \State $\bar w\pow{t} = \bar w\pow{t-1}$ and 
        $\bar g\pow{t} = \bar g\pow{t-1}$.
    \EndIf
    \State Sample $i_{t} \sim \text{Unif}([n])$.
    \State $v^{(t)} = n\dualvar_{i_t}\pow{t} \nabla \ell_{i_{t}}(w^{(t)}) - n\dualvar_{i_t}\pow{t} \nabla \ell_{i_{t}}(\bar{w}\pow{t}) + \bar{g}\pow{t}$. \label{line:osvrg_main_iter1}
    \State $w^{(t+1)} = (1-\eta \reg) w^{(t)} - \eta v^{(t)}$.
    \label{line:osvrg_main_iter2}
\EndFor
\State \Return $w\pow{T}$
\end{algorithmic}
\end{algorithm}

\myparagraph{\osvrg algorithm}
To circumvent the per-iteration cost of full batch algorithms, we consider adapting the SVRG method~\cite{Johnson2013Accelerating} for ERM to account for the ordering of the losses, leading to the \osvrg algorithm presented in \cref{alg:osvrg-u:theory-main}. Overall, the algorithm consists of considering the objective $\sum_{i=1}^n \sigma_i \ell_{(i)}(w) + \mu \|w\|_2^2/2$ as a weighted average $\frac{1}{n}\sum_{i=1}^n (n\sigma_{\pi^{-1}(i)} \ell_i(w) + \mu \|w\|_2^2/2)$ for $\pi \in \argsort(\ell(w))$ and to run epochs of a q-SVRG~\cite{hofmann2015variance} algorithm on an objective of the form $\frac{1}{n}\sum_{i=1}^n (n\sigma_{\hat \pi^{-1}(i)} \ell_i(w) + \mu \|w\|_2^2/2)$ for $\hat \pi$ the ordering of losses computed at some regular checkpoints. 

Concretely, with frequency $N$ starting with the first iterate, we compute (i) the $n$ losses at the current iterate to define a vector of weights $\lambda \pow t$ associated to the empirical ordered statistics at that point in lines~\ref{line:osvrg_sorting_checkpoint} and~\ref{line:osvrg_weights_checkpoint}, and (ii) store the current iterate as a checkpoint $\bar w \pow t$ together with the average gradients of the losses $\bar g \pow t$ at that checkpoint in lines~\ref{line:osvrg_checkpoint_var} and~\ref{line:osvrg_checkpoint_grad}. In addition, with probability $q^*$ at each iteration we update the checkpoint $\bar w \pow t$ and the associated average gradients $\bar g \pow t$ as per rule of line~\ref{line:osvrg_random_checkpoint}, without updating the weights $\lambda\pow t$. The main iteration of the algorithm in lines~\ref{line:osvrg_main_iter1} and~\ref{line:osvrg_main_iter2}  is a variance-reduced gradient step akin to SVRG on an objective of the form $\frac{1}{n}\sum_{i=1}^n (n \bar\lambda_i \ell_i(w)+\reg \|w\|_2^2/2)$ where $\bar \lambda = \lambda\pow t$ are the current weights. 

\myparagraph{\osvrg analysis}
To account for the non-differentiability of the sorting operation in the convergence analysis, we consider a variant of the \osvrg algorithm that operates on the smooth approximation $\lstat_{\nu \Omega}$ of the empirical $L$-statistic $\lstat(\lossval) = \sum_{i=1}^n \sigma_i\lossval_{(i)}$, defined using a strongly convex function $\Omega$ as~\citep{nesterov2005smooth, beck2012smoothing} 
\[
    \lstat_{\nu \Omega} (\lossval) := \max_{\dualvar \in \mathcal{P}(\sigma)} \left\{\lossval^\top \dualvar - \nu \Omega(\dualvar)\right\},
\]
where  $\mathcal{P}(\sigma) = \{ \dualvar = \Pi \sigma: \Pi\ones = \ones, \Pi^\top \ones = \ones, \Pi \in [0,1]^{n\times n}  \}$ is the permutahedron generated by $\sigma$. 
The original $L$-statistic is obtained as $\nu \to 0$ since $h(\lossval) = \max_{\dualvar \in \Pcal(\sigma)} \lossval\T\dualvar$; this follows from the $\sigma_i$'s being non-decreasing. 
The implementation of the smooth approximation of the empirical $L$-statistic and its gradient is given by solving an
isotonic regression problem
at a cost of $O(n\log n)$ elementary computations; see \cref{sec:a:smoothing}.

The resulting smooth surrogate of \eqref{eqn:orm2} is
\begin{align} \label{eq:orm:smooth}
  \primobjsmooth{\mu,\nu\Omega}(w) = \lstat_{\nu\Omega}\big(\ell(w)\big) + 
 \frac{\mu}{2}\norm{w}^2,
\end{align}
for $\ell(w) = (\ell_1(w), \ldots, \ell_n(w))$.
The smoothed version of \osvrg we analyze computes the weights in line~\ref{line:osvrg-dual-update} as $\lambda\pow{t} = \grad h_{\nu \Omega}(\ell(w\pow{t-1}))$. Note that this update recovers the original one in \Cref{alg:osvrg-u:theory-main} as $\nu \to 0$ when the losses $\ell_i(w\pow{t})$ are unique.
Under appropriate smoothness assumptions and choice of the smoothing parameter $\nu$, this variant of \osvrg converges linearly to the minimizer of the smoothed objective.

\begin{theorem}\label{thm:main_osvrg}
Consider the smooth objective \eqref{eq:orm:smooth}
where each $\ell_i$ is convex, $G$-Lipschitz continuous and $L$-smooth, and $\Omega(\lambda) = \|\lambda - \ones/n\|_2^2/2$.
Consider the sequence $(w\pow{t})$ generated by the smoothed variant of \osvrg with inputs $\nu  \geq 4n G^2/\reg$, 
$N = 4(n+ 8\kappa)$,
$\stepsize = 2 / ((n+8\kappa)\reg)$
where $\kappa = n\sigma_{\max} L / \reg + 1$
is a condition number.
We have that $w\pow{t}$ converges to $w^* = \argmin_{w \in \reals^d} \primobjsmooth{\reg,\nu\Omega}(w)$ as 
\[
    \expect \|{w\pow{kN}-w^*}\| \le \left(1/2 \right)^{k} \|{w\pow{0} - w^*}\|
\]
for $k \in \mathbb{N}$.
Consequently, \osvrg can produce a point $\hat w$ satisfying $\big(\expect\norm{\hat w - w^*})^2 \le \eps$ in
\[
    T \le C (n + \kappa) \log 
    \left( \|w \pow{0} - w^*\|_2^2/\eps \right) 
\]
gradient evaluations, where $C$ is an absolute constant.
\end{theorem}
\begin{proofsketch}
    Consider 
    \[
        \saddleobj(w, \dualvar) := 
        \sum_{i=1}^n \dualvar_i \ell_i(w) + \frac{\reg}{2} \norm{w}^2 - \nu \Omega(\dualvar)\,,
    \]
    so that $\primobjsmooth{\mu, \nu\Omega}(w) =  \max_{\dualvar \in \Pcal(\sigma)} \saddleobj(w, \dualvar)$.
    We interpret \Cref{alg:osvrg-u:theory-main} as trying to find the unique saddle point $(w^*, \dualvar^*)$ of $\saddleobj$ by alternating the updates $\dualvar\pow{k} = \argmax_{\dualvar \in \Pcal(\sigma)} \saddleobj(w\pow{k}, \lambda)$ and 
    $w\pow{k+1} \approx w\pow{k+1}_* := \argmin_w \saddleobj(w, \dualvar\pow{k})$ using $N$ steps of q-SVRG. 
    An error analysis of the latter yields 
    \[
    \expect_k\|w\pow{k+1} - w\pow{k+1}_*\|
    \le \frac{1}{5} \|w\pow{k} - w\pow{k+1}_*\| \,,
    \]
    where $\expect_k$ denotes an expectation conditioned on the sigma-algebra generated by $w\pow{k}$.
    Smoothness and strong convexity/concavity of $\saddleobj$ gives
    \[
        \|w\pow{k+1}_* - w^*\| 
        \le \frac{\sqrt{n}G}{\reg} \|\dualvar\pow{k} - \dualvar^*\| 
        \le \frac{nG^2}{\reg \nu} \|w\pow{k} - w^*\| \,.
    \] 
    Putting these together with the triangle inequality and $nG^2/(\reg \nu) \le 1/4$ completes the proof.
\end{proofsketch}
The approximation error induced by the smooth approximation can be controlled by the smoothing coefficient $\nu$. For any non-negative, strongly convex, decomposable function $\Omega$ we have $0 \leq \msc{R}_{\sigma, \mu}(w) - \msc{R}_{\sigma, \mu , \nu \Omega}(w)  \leq \nu \Omega(\sigma)$. The quantity $\Omega(\sigma)$ can then itself be bounded in terms of a divergence of $s$ to the uniform distribution. In particular, using a centered negative entropy as $\Omega$, we have $\Omega(\sigma) \leq \operatorname{KL}(s\Vert u)$, Kullback-Leibler divergence from $s$ to $u$. On the other hand, using a centered squared Euclidean norm as in \cref{thm:main_osvrg}, we get $\Omega(\sigma)\leq \chi^2(s \Vert u)/n$, the $\chi^2$-divergence. See \Cref{sec:a:smoothing} for details.
In summary, if a point $\hat w$ is an $\varepsilon/2$-accurate minimizer of the smoothed objective, i.e., $\msc{R}_{\sigma, \mu, \nu\Omega}(\hat w) - \min_{w \in \reals^d} \msc{R}_{\sigma, \mu, \nu\Omega}(w)\leq \varepsilon$, then it is a $\varepsilon/2 + \nu \chi^2(s\Vert u)/n$-approximate one on the original non-smooth objective when choosing $\Omega$ as in \cref{thm:main_osvrg}.

Combining the smoothing error with the requirement $\nu > O(nG^2/\reg)$ of \Cref{thm:main_osvrg}, we get an end-to-end bound on the original non-smooth objective when $\varepsilon > G^2 \chi^2(s \Vert u) / \reg$. 
However, 
as we show empirically in \Cref{sec:a:additional}, smoothing has a minimal impact on the empirical behavior of \osvrg. 
While the non-smoothness of the empirical $L$-statistic is an obstacle for the theoretical convergence of \osvrg, this non-smoothness may not impact the empirical behavior. Indeed, if the minimizer of the objective has distinct loss values, then the objective is locally smooth around the minimizer. 

\myparagraph{Time complexity}
In practice, we consider simply taking $N=n$ and $q^*=0$ to simplify the hyperparameter choices and reduce the overall time complexity. In that case, the time complexity of \osvrg is $O(d)$ per iteration with 2 gradient evaluations, which is identical to the number of gradient calls of the biased subgradient method with batch size $m=2$. \osvrg also requires $n$ gradient evaluations and sorting at the start of an epoch, contributing an additional $O(nd + n\log n)$ time. This per-epoch complexity is nearly identical to vanilla SVRG in the ERM case. \osvrg, like vanilla SVRG, also requires an additional storage of $O(d)$ to store $\bar{g}_k \in \partial\msc{R}_{\sigma, \reg}(\bar{w}_k)$ as compared to the stochastic subgradient method. Run times are evaluated experimentally in \Cref{sec:a:additional}.
 
\section{Experimental Results}\label{sec:experiments}
\begin{figure}[t]
    \centering
    \includegraphics[width=0.9\linewidth]{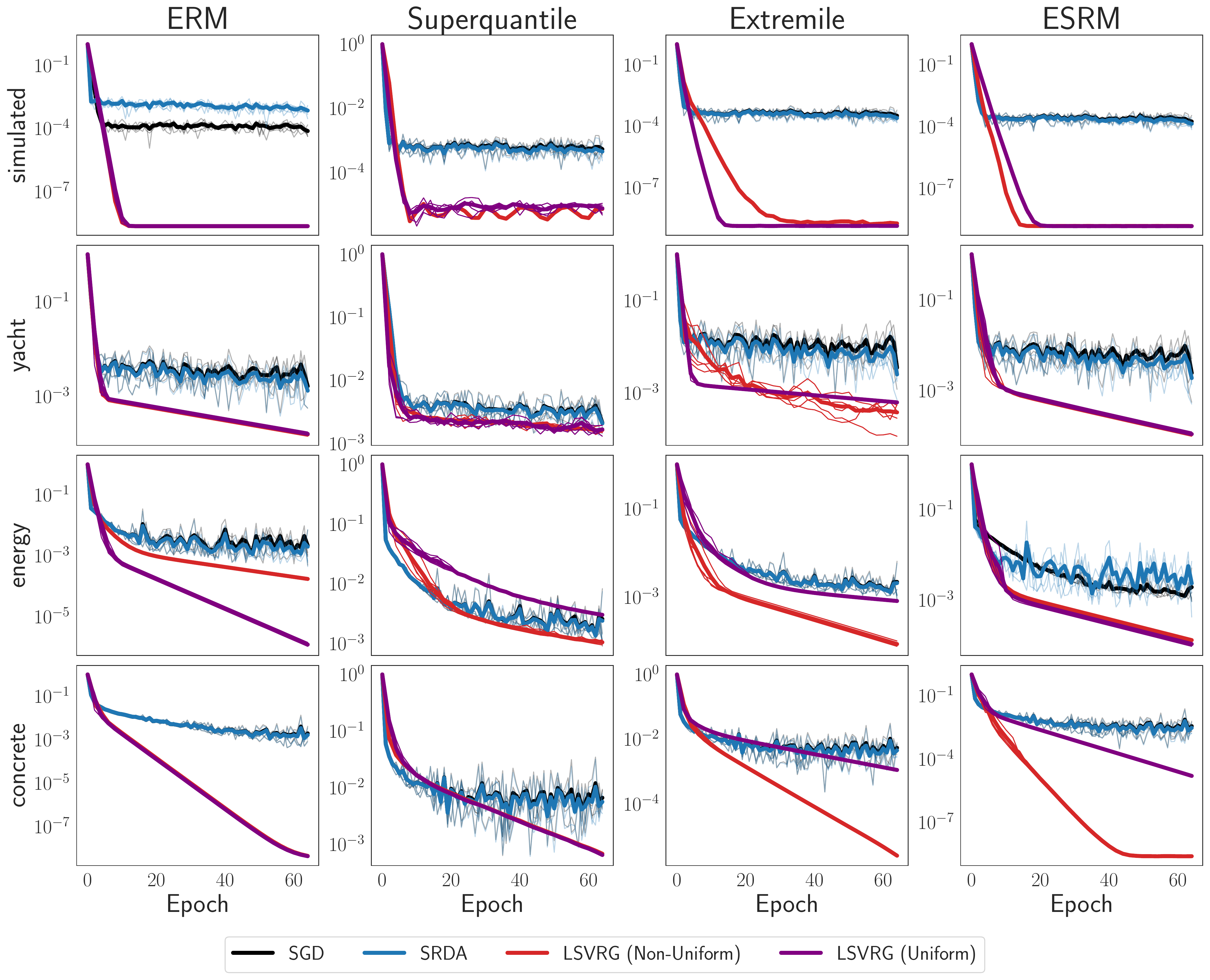}
    \caption{The suboptimality gap (\Cref{eqn:gap}) for various optimization algorithms on the mean, superquantile, extremile, and ESRM risk measures. The $x$-axis shows the number of effective passes through the data. Five seed trajectories are plotted translucently for every algorithm.
    }
    \label{fig:training_curves}
\end{figure}

\begin{figure*}[t!]
    \centering
    \includegraphics[width=.9\linewidth]{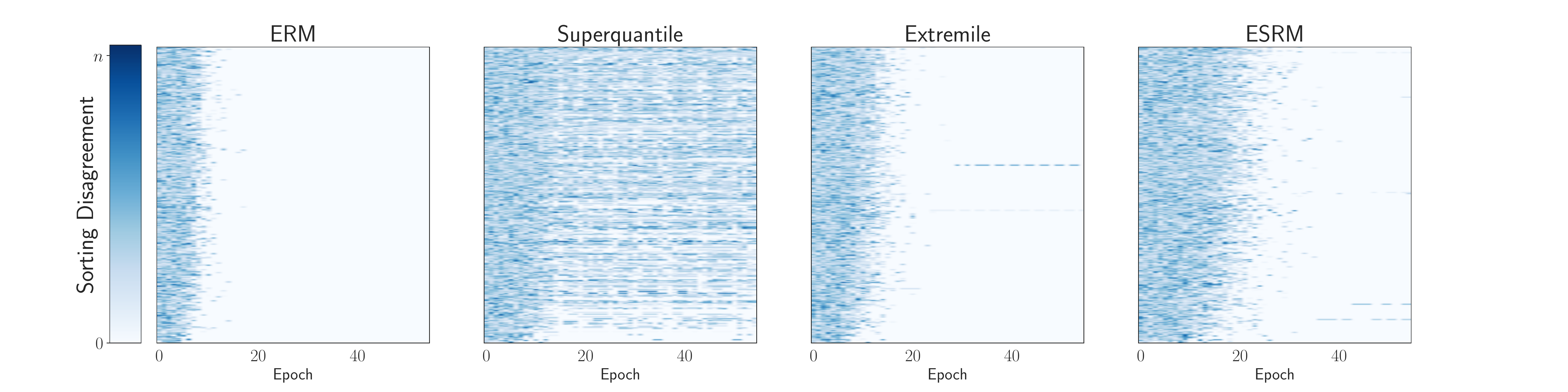}
    \caption{Sorting sensitivity for \texttt{simulated} dataset ($n = 800$) along epochs ($x$-axis) of \osvrg applied each spectral risk objective. Each heatmap shows the vector of disagreements between sorting permutations $\pi$ at each epoch of \Cref{alg:osvrg-u:theory-main}.}
    \label{fig:sensitivity}
\end{figure*}

We compare the performance of minibatch SGD and \osvrg on benchmark datasets and study their bias and variance properties in a number of supervised and unsupervised learning tasks. Experimental details can be found in \Cref{sec:a:experiments}, with additional experiments with varied hyperparameters can be found in \Cref{sec:a:additional}. The code as well as the scripts to reproduce the experiments are made publicly available online: \textsf{{https://github.com/ronakdm/lerm}}.

\subsection{Regression}
We consider 4 regression datasets: 
\begin{itemize}[nosep]
    \item \texttt{simulated}: a synthetic task of predicting observations generated from a noisy linear model.
    \item \texttt{yacht}: prediction of the residuary resistance of a sailing yacht based on its physical attributes \cite{Tsanas2012AccurateQE}.
    \item \texttt{energy}: prediction of the cooling load of a building based on its physical attributes \cite{Segota2020Artificial}.
    \item \texttt{concrete}: prediction of the compressive strength of a concrete type based on its physical and chemical attributes \cite{Yeh2006Analysis}.
\end{itemize}

\begin{figure}[t]
    \centering
    \includegraphics[width=0.88\linewidth]{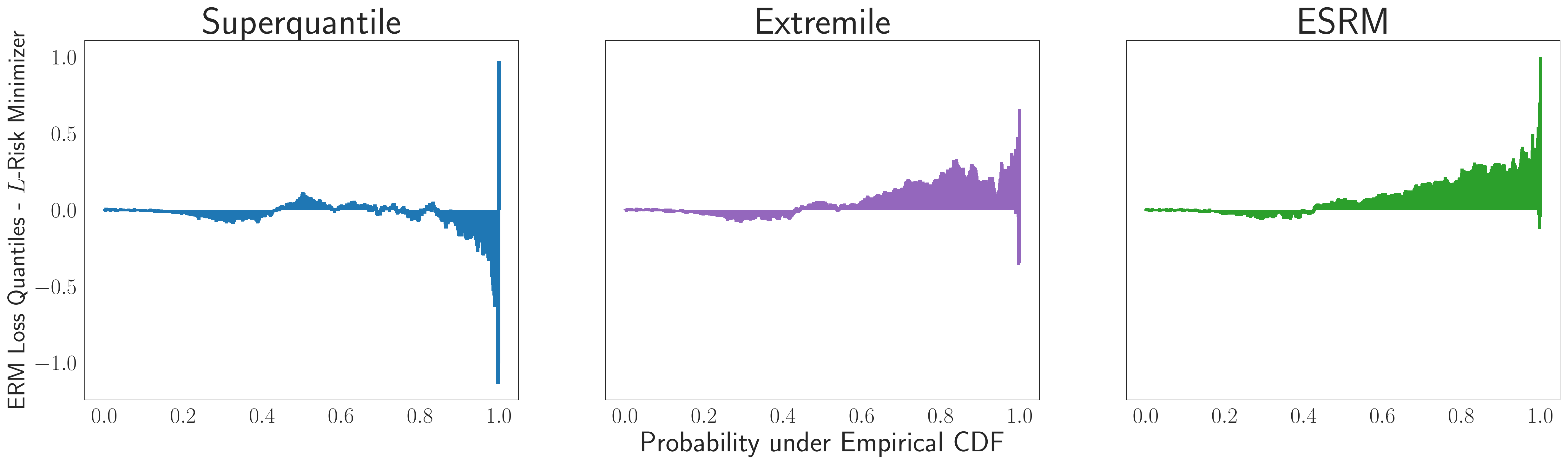}
    \includegraphics[width=0.88\linewidth]{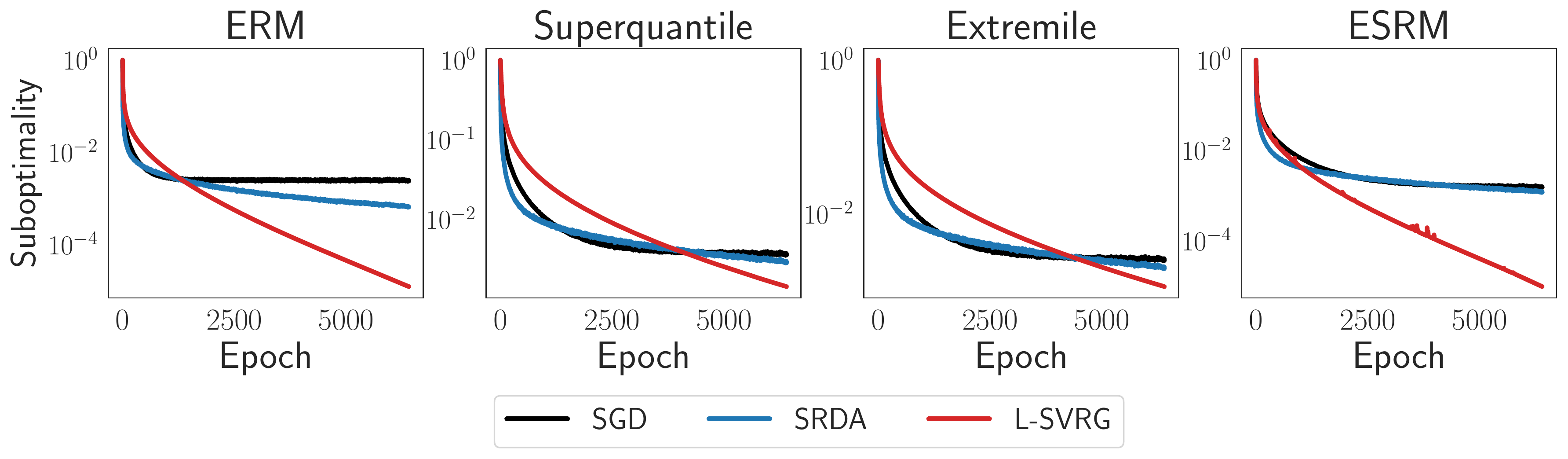}
    \caption{{\bf Top:} Differences in the $p$-th quantile of the loss distribution for the ERM solution and the $L$-Risk solution on the {\tt iWildCam} test set. {\bf Bottom:} Training curves for SGD, SRDA, and L-SVRG for different $L$-Risks on {\tt iWildCam} training set.}
    \label{fig:iwildcam_training_curves}
\end{figure}

We use the squared loss under a linear model and aim to minimize the regularized objective \eqref{eqn:orm2} 
where the spectra $s$ are obtained from the empirical mean, superquantile ($q = 0.5$), extremile ($r = 2$), and ESRM ($\rho=1$) of the losses. Both training curves and test losses for other values of $(q, r, \rho)$ are shown in \Cref{sec:a:additional}, which follow similar trends.

In addition to minibatch SGD, we consider another biased 
method {\it stochastic regularized dual averaging (SRDA)} \cite{Xiao2009Dual}, both with a batch size of $64$. 
We compare them with \osvrg, by plotting in \Cref{fig:training_curves} the suboptimality, defined as
\begin{align}
    \text{suboptimality gap}_t := \frac{\msc{R}_\sigma(w\pow{t}) - \msc{R}_\sigma(w^*)}{\msc{R}_\sigma(w\pow{0}) - \msc{R}_\sigma(w^*)}
    \label{eqn:gap}
\end{align}

We find that \osvrg (without smoothing) exhibits empirical linear convergence for the ERM, extremile, and ESRM. It often vastly outperforms SGD and SRDA, which exhibit sublinear convergence. On the superquantile,
\osvrg exhibits the same sublinear convergence as SGD, suggesting that the superquantile, with its discontinuous spectrum, can be hard to optimize. Overall, \osvrg is the best or close to the best algorithm across all tasks. 

\osvrg relies on the hypothesis that the sorted order of losses stabilize as iterates $w\pow{t}$ get close to the optimum. 
We see from \Cref{fig:sensitivity} that
there is a clear phase change after which disagreements between the true and estimated ordering are visually unnoticeable. 
The exception to this is the superquantile, where the sorting does not stabilize within $64$ epochs. This corroborates the apparent hardness of optimizing the superquantile in \Cref{fig:training_curves}.

\subsection{Classification}

We evaluate the $L$-Risk minimizers on a larger scale image classification benchmark. The {\tt iWildCam} challenge dataset contains natural images from wilderness sites with distribution shifts arising from diverse camera angles, backgrounds, and relative animal frequencies. We take a subsample of $n=20,000$ data points from classes with at least $100$ examples after removing the ``background image" class. We then featurize the images using the penultimate layer of a ResNet50 neural network with parameters pre-trained on ImageNet (see \Cref{sec:a:experiments} for further details). The convex optimization problem considered is multinomial logistic regression arising from fine-tuning the last layer of the model. The training curves in the bottom row of \Cref{fig:iwildcam_training_curves} indicate that SGD and SRDA fail to converge due to bias and variance. Letting $\hat{w}_{\text{ERM}}$ be the approximate solution of ERM, whereas $\hat{w}_{\text{LRM}}$ is the approximate solution of an $L$-Risk minimization problem other than ERM, the top row plots the following against $p$:
\begin{equation*}
    \frac{\ell_{(\ceil{np})}\p{\hat{w}_{\text{ERM}}} - \ell_{(\ceil{np})}\p{\hat{w}_{\text{LRM}}}}{\frac{1}{n} \sum_{i=1}^n \ell_i\p{\hat{w}_{\text{ERM}}}},
\end{equation*}
that is, the difference in the $p$-th quantile of the test loss of $\hat{w}_{\text{ERM}}$ and the $p$-th quantile of the test loss of $\hat{w}_{\text{LRM}}$, normalized by the mean test loss of ERM. Because logistic loss measures the negative logarithm of the probability that the model assigns to the correct label, tail events for this loss amount to a model exhibiting high confidence for a set of incorrect labels. The median test loss ($p = 0.5$) is similar between the $L$-risk minimizers and standard ERM. However, for $p > 0.5$, the ERM solution can make predictions with much higher losses. Comparing various $L$-risks, we find that the superquantile controls tail error at very high quantiles ($p > 0.95$), but generally underperforms for the rest of the loss distribution. The extremile and ESRM on the other hand, have generally better performance than ERM throughout the distribution. We also plot the quantile differences for the regression tasks in \Cref{sec:a:additional}.

Crucially, we find that the large $n$ regime exacerbates the bias issues when the epoch length is set to $n$. We instead use a smaller epoch length of $100$, and plot suboptimality by number of gradient evaluations in \Cref{fig:iwildcam_training_curves} to ensure a fair comparison. Each epoch is defined as the number of gradient evaluations in SGD or SRDA, which is $100m = 6,400$.

\begin{figure}[t!]
    \centering
    \includegraphics[width=0.35\linewidth]{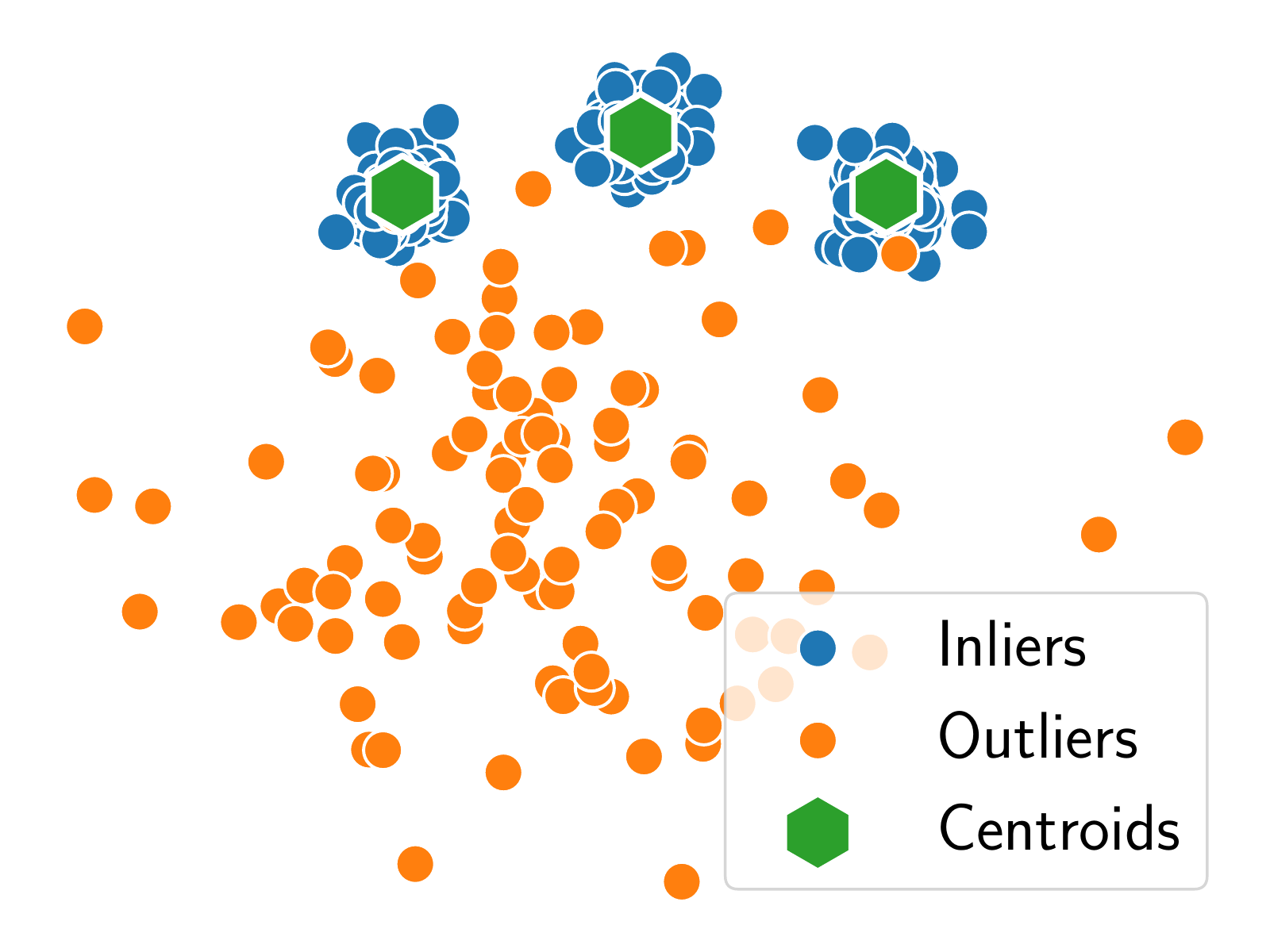}~
    \includegraphics[width=0.35\linewidth]{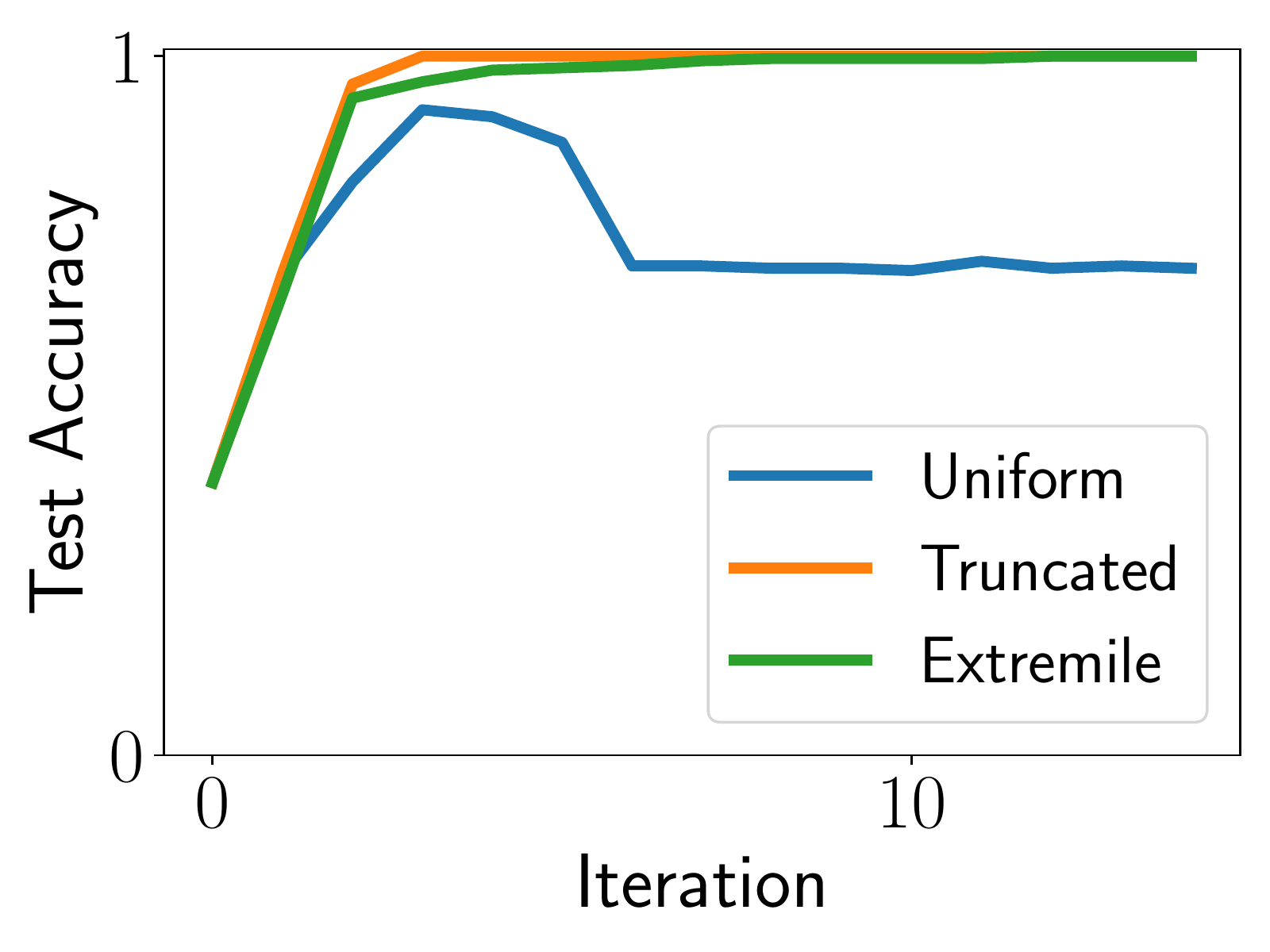}
    \includegraphics[width=0.35\linewidth]{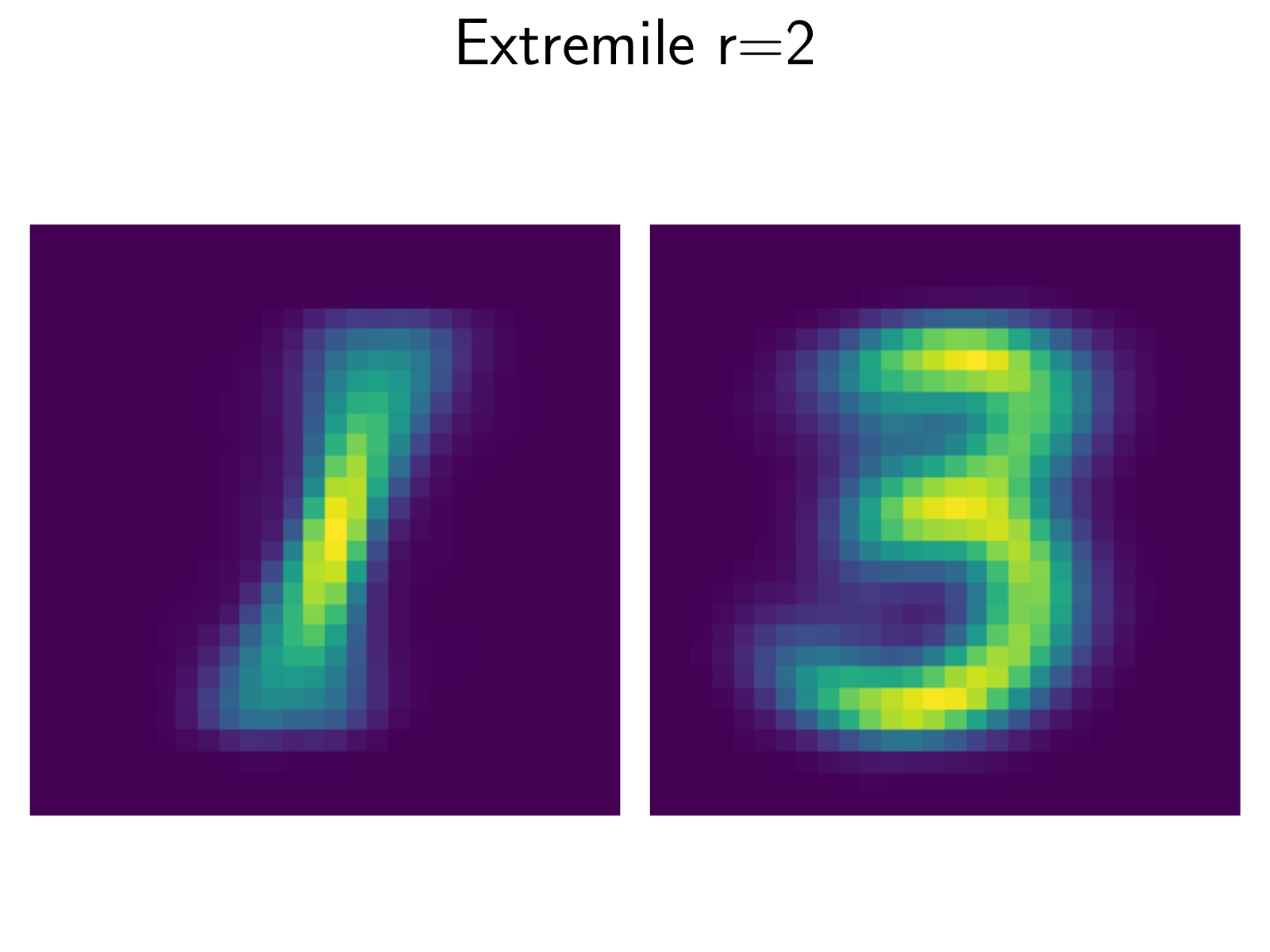}~
    \includegraphics[width=0.35\linewidth]{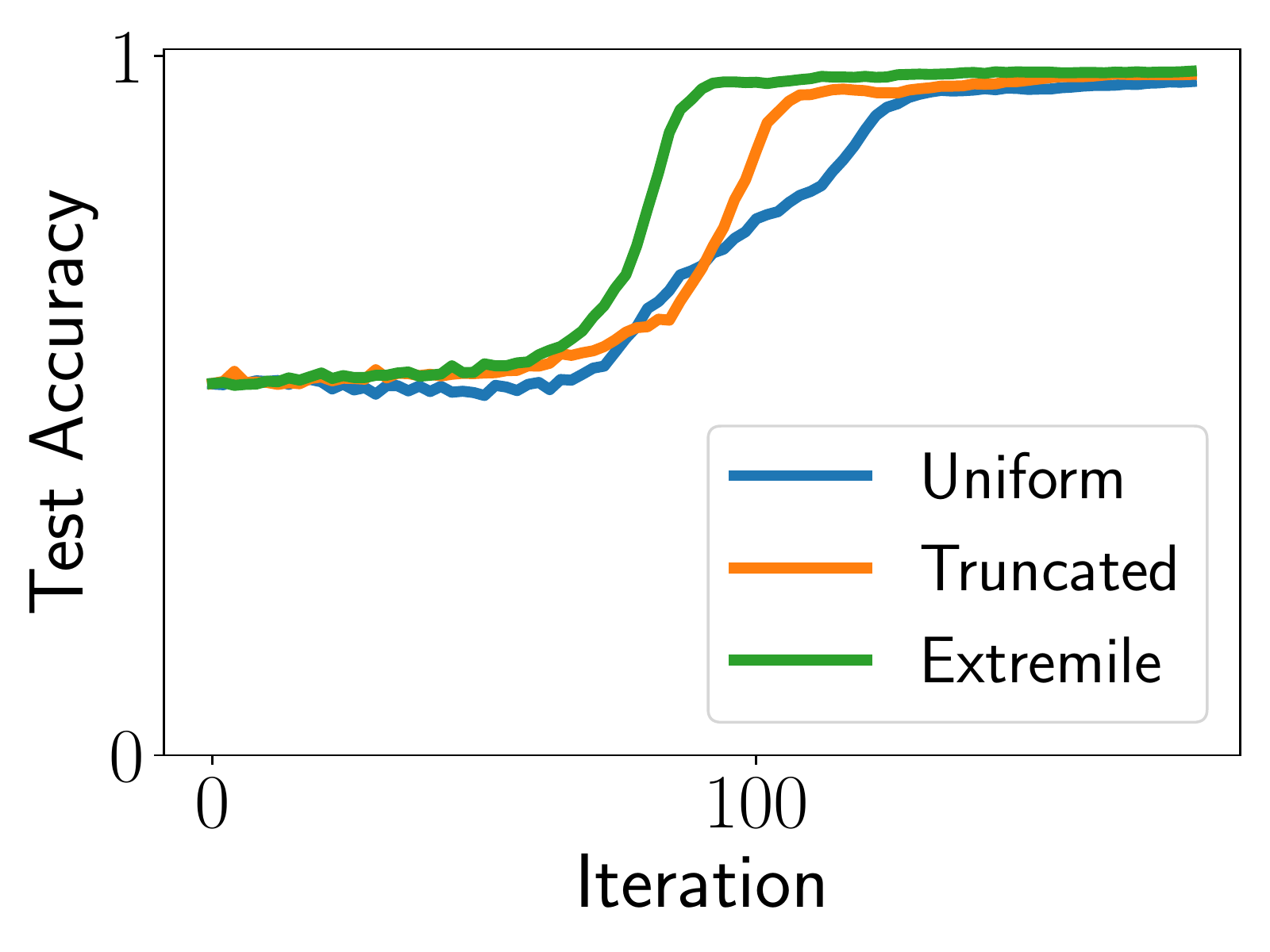}
    \caption{Robust clustering with $L$-statistics in the presence of outliers. \textbf{Top}: synthetic data, \textbf{Bottom}: MNIST digits.} 
    \label{fig:kmeans_exp}
\end{figure}

\subsection{Clustering}
We also explore an unsupervised clustering approach from~\citet{Maurer2020RobustUnsupervised} on synthetic data and real data. We seek to cluster $n$ points $x_1, \ldots, x_n$ into $k$ clusters with centers $C=(c_1, \ldots, c_k)$ by minimizing a weighted average of the distances of each point to its closest center, i.e., problems of the form
\[
\min_{C \in \reals^{d\times k}} \sum_{i=1}^n \sigma_i \ell_{(i)}(C) \ \mbox{for}\ \ell_i(C) = \hspace*{-6pt} \min_{\substack{z_i \in \{0, 1\}^k\\ z_i^\top \ones = \ones}} \hspace{-1pt}\sum_{j=1}^k z_{ij}\|x_i-c_j\|_2^2.
\]
Taking $\sigma_i =1/n$, we retrieve the usual objective minimized by $k$-means.
\citet{Maurer2020RobustUnsupervised} propose to take $\sigma_i$ non-uniform to mitigate the effect of outliers in the data. Specifically, we consider $\sigma_i=\int_{(i-1)/n}^{i/n} s(t) \D t$ for a truncated spectrum $s_q(t)= \mathbf{1}_{[0,q]}(t)/q$ or a risk-seeking version of the extremile, $s_r(t) =r(1-t)^r$.
In addition, \citet{Maurer2020RobustUnsupervised} optimize the clustering objective by alternating $k$-means iterations and sorting the resulting losses. Instead, we apply minibatch SGD from \Cref{alg:sgd} with a constant stepsize found by grid search and a batch size of $64$.   

\textbf{Synthetic data.}
We generate a dataset of three Gaussian clouds of $100$ points each and an additional set of 100 outliers (top left in \Cref{fig:kmeans_exp}). 
We compare the accuracy of clustering $300$ new inlier points with different spectra.
We observe in \Cref{fig:kmeans_exp} (top right) that 
the minibatch estimates of the subgradient of the truncated or extremile spectra are sufficient to reach a perfect $100\%$ accuracy, while vanilla $k$-means with its uniform spectrum leads to poor performance due to outliers.

\textbf{MNIST data.}
We consider distinguishing between the digits 1 and 3 from the MNIST dataset~\citep{lecun1998mnist} by clustering the images of a training set composed of 1000 samples of 1 and 3 each and additional 125 outliers for each other digit. We test the clustering procedure on the images of 1 and 3 digits from the MNIST test set.
We see from \Cref{fig:kmeans_exp} (bottom right) that minibatch SGD with a batch size of $256$ achieves $97.3\%$ for the truncated spectrum and $97.8\%$ for the extremile spectrum versus $96.3\%$ for the uniform spectrum. 
Even in terms of convergence speed for different spectra, we observe that extremile $\succ$ truncated $\succ$ uniform.
Finally, \Cref{fig:kmeans_exp} (bottom left) shows us that the centers computed with the extremile spectrum are clear representatives of these digits while taking a uniform spectrum leads to more blurry representatives, as shown in \Cref{sec:a:experiments}.

\section{Discussion}\label{sec:discussion}
In this paper, we propose stochastic optimization algorithms for minimizing spectral risk measures, allowing for practitioners to interpolate between optimizing the average-case and worst-case performance on a learning task. We establish consistency of the sample spectral risk. We present the \osvrg algorithm and analyze its convergence properties alongside biased minibatch SGD. The \osvrg algorithm demonstrates rapid empirical convergence on benchmark datasets. The experiments show that the algorithm converges linearly to the minimizer of the spectral risk, outperforming baselines. Future work includes establishing the regular subdifferential in the non-convex setting and studying the robustness properties of spectral risk minimizers.

\subsubsection*{Acknowledgements}
This work was supported by NSF DMS-2023166, NSF CCF-2019844, NSF DMS-2052239, NSF DMS-2134012, NSF DMS-2133244, NIH, CIFAR-LMB, and faculty research awards. Part of this work was done while Zaid Harchaoui was visiting the Simons Institute for the Theory of Computing, and while Krishna Pillutla was at the University of Washington. 
\clearpage
\bibliographystyle{abbrvnat}
\bibliography{bib}

\begin{thebibliography}{60}
\providecommand{\natexlab}[1]{#1}
\providecommand{\url}[1]{\texttt{#1}}
\expandafter\ifx\csname urlstyle\endcsname\relax
  \providecommand{\doi}[1]{doi: #1}\else
  \providecommand{\doi}{doi: \begingroup \urlstyle{rm}\Url}\fi

\bibitem[Acerbi and Tasche(2002)]{acerbi2002coherence}
C.~Acerbi and D.~Tasche.
\newblock On the coherence of expected shortfall.
\newblock \emph{Journal of Banking \& Finance}, 26\penalty0 (7):\penalty0
  1487--1503, 2002.

\bibitem[Artzner et~al.(1999)Artzner, Delbaen, Jean-Marc, and
  Heath]{Artzner1999Coherent}
P.~Artzner, F.~Delbaen, E.~Jean-Marc, and D.~Heath.
\newblock {Coherent Measures of Risk}.
\newblock \emph{Mathematical Finance}, 9:\penalty0 203 -- 228, 07 1999.

\bibitem[Bach(2023)]{Bach2023Learning}
F.~Bach.
\newblock \emph{Learning Theory from First Principles}.
\newblock The MIT Press, 2023.

\bibitem[Baressi~Segota et~al.(2020)Baressi~Segota, Andelic, Kudlacek, and
  Cep]{Segota2020Artificial}
S.~Baressi~Segota, N.~Andelic, J.~Kudlacek, and R.~Cep.
\newblock Artificial neural network for predicting values of residuary
  resistance per unit weight of displacement.
\newblock \emph{Journal of Maritime \& Transportation Science}, 57, 2020.

\bibitem[Beck and Teboulle(2012)]{beck2012smoothing}
A.~Beck and M.~Teboulle.
\newblock Smoothing and first order methods: A unified framework.
\newblock \emph{SIAM Journal on Optimization}, 22\penalty0 (2):\penalty0
  557--580, 2012.

\bibitem[Beery et~al.(2020)Beery, Cole, and Gjoka]{beery2020iwildcam}
S.~Beery, E.~Cole, and A.~Gjoka.
\newblock The iwildcam 2020 competition dataset.
\newblock \emph{arXiv preprint arXiv:2004.10340}, 2020.

\bibitem[Ben-Tal and Teboulle(2007)]{Ben2007AnOld}
A.~Ben-Tal and M.~Teboulle.
\newblock An old-new concept of convex risk measures: The optimized certainty
  equivalent.
\newblock \emph{Mathematical Finance}, 17:\penalty0 449--476, 2007.

\bibitem[Best et~al.(2000)Best, Chakravarti, and Ubhaya]{best2000minimizing}
M.~J. Best, N.~Chakravarti, and V.~A. Ubhaya.
\newblock {Minimizing Separable Convex Functions Subject to Simple Chain
  Constraints}.
\newblock \emph{SIAM Journal on Optimization}, 10\penalty0 (3):\penalty0
  658--672, 2000.

\bibitem[Blondel et~al.(2020)Blondel, Teboul, Berthet, and
  Djolonga]{blondel2020fast}
M.~Blondel, O.~Teboul, Q.~Berthet, and J.~Djolonga.
\newblock Fast differentiable sorting and ranking.
\newblock In \emph{International Conference on Machine Learning}, pages
  950--959, 2020.

\bibitem[Bobkov and Ledoux(2019)]{Bobkov2019OnedimensionalEM}
S.~G. Bobkov and M.~Ledoux.
\newblock {One-Dimensional Empirical Measures, Order Statistics, and
  Kantorovich Transport Distances}.
\newblock \emph{Memoirs of the American Mathematical Society}, 2019.

\bibitem[Chen and Paschalidis(2020)]{chen2020distributionally}
R.~Chen and I.~C. Paschalidis.
\newblock {Distributionally Robust Learning}.
\newblock \emph{Foundations and Trends{\textregistered} in Optimization},
  4\penalty0 (1-2):\penalty0 1--243, 2020.

\bibitem[Cotter and Dowd(2006)]{Cotter2006Extreme}
J.~Cotter and K.~Dowd.
\newblock {Extreme Spectral Risk Measures: An Application to Futures
  Clearinghouse Margin Requirements}.
\newblock \emph{Journal of Banking \& Finance}, 30\penalty0 (12):\penalty0
  3469--3485, 2006.

\bibitem[Curi et~al.(2020)Curi, Levy, Jegelka, and Krause]{curi2020adaptive}
S.~Curi, K.~Y. Levy, S.~Jegelka, and A.~Krause.
\newblock {Adaptive Sampling for Stochastic Risk-Averse Learning}.
\newblock In \emph{Neural Information Processing Systems}, volume~33, 2020.

\bibitem[Daouia et~al.(2019)Daouia, Gijbels, and
  Stupfler]{Daouia2019Extremiles}
A.~Daouia, I.~Gijbels, and G.~Stupfler.
\newblock {Extremiles: A New Perspective on Asymmetric Least Squares}.
\newblock \emph{Journal of the American Statistical Association}, 114\penalty0
  (527):\penalty0 1366--1381, 2019.

\bibitem[Defazio et~al.(2014)Defazio, Bach, and
  Lacoste-Julien]{Defazio2014SAGA}
A.~Defazio, F.~Bach, and S.~Lacoste-Julien.
\newblock {SAGA: A Fast Incremental Gradient Method With Support for
  Non-Strongly Convex Composite Objectives}.
\newblock In \emph{Neural Information Processing Systems}, volume~27, 2014.

\bibitem[Deng et~al.(2009)Deng, Dong, Socher, Li, Li, and
  Fei-Fei]{deng2009imagenet}
J.~Deng, W.~Dong, R.~Socher, L.-J. Li, K.~Li, and L.~Fei-Fei.
\newblock Imagenet: A large-scale hierarchical image database.
\newblock In \emph{2009 IEEE conference on computer vision and pattern
  recognition}, pages 248--255. Ieee, 2009.

\bibitem[Duchi and Namkoong(2019)]{duchi2019variance}
J.~C. Duchi and H.~Namkoong.
\newblock {Variance-based Regularization with Convex Objectives}.
\newblock \emph{Journal of Machine Learning Research}, 20\penalty0
  (68):\penalty0 1--55, 2019.

\bibitem[Fan et~al.(2017)Fan, Lyu, Ying, and Hu]{Fan2017Learningwith}
Y.~Fan, S.~Lyu, Y.~Ying, and B.~Hu.
\newblock {Learning with Average Top-$k$ Loss}.
\newblock In \emph{Neural Information Processing Systems}, volume~30, 2017.

\bibitem[F{\"{o}}llmer and Schied(2002)]{Follmer2002Convexmeasures}
H.~F{\"{o}}llmer and A.~Schied.
\newblock Convex measures of risk and trading constraints.
\newblock \emph{Finance Stochastics}, 6\penalty0 (4):\penalty0 429--447, 2002.

\bibitem[Guigues and Sagastiz{\'{a}}bal(2013)]{Guigues2013Risk-averse}
V.~Guigues and C.~A. Sagastiz{\'{a}}bal.
\newblock Risk-averse feasible policies for large-scale multistage stochastic
  linear programs.
\newblock \emph{Mathematical Programming}, 138\penalty0 (1-2):\penalty0
  167--198, 2013.

\bibitem[He et~al.(2016)He, Zhang, Ren, and Sun]{He2016DeepResidual}
K.~He, X.~Zhang, S.~Ren, and J.~Sun.
\newblock Deep residual learning for image recognition.
\newblock In \emph{2016 IEEE Conference on Computer Vision and Pattern
  Recognition (CVPR)}, pages 770--778, 2016.
\newblock \doi{10.1109/CVPR.2016.90}.

\bibitem[He et~al.(2022)He, Kou, and Peng]{He2022RiskMeasures}
X.~D. He, S.~Kou, and X.~Peng.
\newblock {Risk Measures: Robustness, Elicitability, and Backtesting}.
\newblock \emph{Annual Review of Statistics and Its Application}, 9\penalty0
  (1), 2022.

\bibitem[Henzi et~al.(2022)Henzi, M{\"o}sching, and
  D{\"u}mbgen]{henzi2022accelerating}
A.~Henzi, A.~M{\"o}sching, and L.~D{\"u}mbgen.
\newblock Accelerating the pool-adjacent-violators algorithm for isotonic
  distributional regression.
\newblock \emph{Methodology and computing in applied probability}, pages 1--13,
  2022.

\bibitem[Hiriart-Urruty and Lemar{\'e}chal(1993)]{HiriartUrruty1993Convex}
J.-B. Hiriart-Urruty and C.~Lemar{\'e}chal.
\newblock \emph{Convex Analysis and Minimization Algorithms}.
\newblock Springer, 1993.

\bibitem[Hofmann et~al.(2015)Hofmann, Lucchi, Lacoste-Julien, and
  McWilliams]{hofmann2015variance}
T.~Hofmann, A.~Lucchi, S.~Lacoste-Julien, and B.~McWilliams.
\newblock {Variance Reduced Stochastic Gradient Descent with Neighbors}.
\newblock \emph{Neural Information Processing Systems}, 28, 2015.

\bibitem[Holland and Mehdi~Haress(2022)]{Holland2021Spectralrisk}
M.~J. Holland and E.~Mehdi~Haress.
\newblock Spectral risk-based learning using unbounded losses.
\newblock In \emph{International Conference on Artificial Intelligence and
  Statistics}, volume 151, pages 1871--1886, 2022.

\bibitem[Hu et~al.(2018)Hu, Niu, Sato, and Sugiyama]{hu2018does}
W.~Hu, G.~Niu, I.~Sato, and M.~Sugiyama.
\newblock Does distributionally robust supervised learning give robust
  classifiers?
\newblock In \emph{International Conference on Machine Learning}, pages
  2029--2037, 2018.

\bibitem[Johnson and Zhang(2013)]{Johnson2013Accelerating}
R.~Johnson and T.~Zhang.
\newblock Accelerating stochastic gradient descent using predictive variance
  reduction.
\newblock In \emph{Neural Information Processing Systems}, volume~26, 2013.

\bibitem[Kawaguchi and Lu(2020)]{kawaguchi2020ordered}
K.~Kawaguchi and H.~Lu.
\newblock Ordered {SGD}: A new stochastic optimization framework for empirical
  risk minimization.
\newblock In \emph{International Conference on Artificial Intelligence and
  Statistics}, volume 108, pages 669--679, 2020.

\bibitem[Khim et~al.(2020)Khim, Leqi, Prasad, and
  Ravikumar]{Khim2020uniformConvergence}
J.~Khim, L.~Leqi, A.~Prasad, and P.~Ravikumar.
\newblock {Uniform Convergence of Rank-weighted Learning}.
\newblock In \emph{International Conference on Machine Learning}, volume 119,
  pages 5254--5263, 2020.

\bibitem[Kuhn et~al.(2019)Kuhn, Esfahani, Nguyen, and
  Shafieezadeh-Abadeh]{kuhn2019wasserstein}
D.~Kuhn, P.~M. Esfahani, V.~A. Nguyen, and S.~Shafieezadeh-Abadeh.
\newblock {Wasserstein Distributionally Robust Optimization: Theory and
  Applications in Machine Learning}.
\newblock In \emph{Operations Research \& Management Science in the Age of
  Analytics}, pages 130--166. INFORMS, 2019.

\bibitem[Laguel et~al.(2020)Laguel, Malick, and
  Harchaoui]{Laguel2020First-Order}
Y.~Laguel, J.~Malick, and Z.~Harchaoui.
\newblock {First-Order Optimization for Superquantile-Based Supervised
  Learning}.
\newblock In \emph{IEEE International Workshop on Machine Learning for Signal
  Processing}, pages 1--6, 09 2020.

\bibitem[Laguel et~al.(2021)Laguel, Pillutla, Malick, and
  Harchaoui]{Laguel2022Superquantiles}
Y.~Laguel, K.~Pillutla, J.~Malick, and Z.~Harchaoui.
\newblock {Superquantiles at Work: Machine Learning Applications and Efficient
  Subgradient Computation}.
\newblock \emph{Set-Valued and Variational Analysis}, 2021.

\bibitem[Le~Roux et~al.(2012)Le~Roux, Schmidt, and Bach]{Roux2012AStochastic}
N.~Le~Roux, M.~Schmidt, and F.~Bach.
\newblock {A Stochastic Gradient Method with an Exponential Convergence Rate
  for Finite Training Sets}.
\newblock In \emph{Neural Information Processing Systems}, volume~25, 2012.

\bibitem[LeCun et~al.(1998)LeCun, Cortes, and Christopher]{lecun1998mnist}
Y.~LeCun, C.~Cortes, and B.~Christopher.
\newblock {MNIST} handwritten digit database.
\newblock http://yann.lecun.com/exdb/mnist/, 1998.

\bibitem[Lee and Raginsky(2018)]{lee2018minimax}
J.~Lee and M.~Raginsky.
\newblock {Minimax Statistical Learning with Wasserstein distances}.
\newblock In \emph{Neural Information Processing Systems}, volume~31, pages
  2687--2696, 2018.

\bibitem[Lee et~al.(2020)Lee, Park, and Shin]{Lee2020LearningBounds}
J.~Lee, S.~Park, and J.~Shin.
\newblock {Learning Bounds for Risk-sensitive Learning}.
\newblock In \emph{Neural Information Processing Systems}, volume~33, pages
  13867--13879, 2020.

\bibitem[Leqi et~al.(2019)Leqi, Prasad, and Ravikumar]{Liu2019OnHuman}
L.~Leqi, A.~Prasad, and P.~K. Ravikumar.
\newblock {On Human-Aligned Risk Minimization}.
\newblock In \emph{Neural Information Processing Systems}, volume~32, 2019.

\bibitem[Levy et~al.(2020)Levy, Carmon, Duchi, and
  Sidford]{Levy2020Large-Scale}
D.~Levy, Y.~Carmon, J.~Duchi, and A.~Sidford.
\newblock {Large-Scale Methods for Distributionally Robust Optimization}.
\newblock In \emph{Neural Information Processing Systems}, volume~33, 2020.

\bibitem[Li et~al.(2021)Li, Beirami, Sanjabi, and Smith]{li2020tilted}
T.~Li, A.~Beirami, M.~Sanjabi, and V.~Smith.
\newblock {Tilted Empirical Risk Minimization}.
\newblock In \emph{International Conference on Learning Representations}, 2021.

\bibitem[Lim and Wright(2016)]{lim2016efficient}
C.~H. Lim and S.~J. Wright.
\newblock {Efficient Bregman Projections onto the Permutahedron and Related
  Polytopes}.
\newblock In \emph{International Conference on Artificial Intelligence and
  Statistics}, pages 1205--1213, 2016.

\bibitem[Mairal(2014)]{Mairal2014Incremental}
J.~Mairal.
\newblock {Incremental Majorization-Minimization Optimization with Application
  to Large-Scale Machine Learning}.
\newblock \emph{SIAM Journal on Optimization}, 25, 02 2014.

\bibitem[Maurer et~al.(2021)Maurer, Parletta, Paudice, and
  Pontil]{Maurer2020RobustUnsupervised}
A.~Maurer, D.~A. Parletta, A.~Paudice, and M.~Pontil.
\newblock {Robust Unsupervised Learning via {L}-statistic Minimization}.
\newblock In \emph{International Conference on Machine Learning}, pages
  7524--7533, 2021.

\bibitem[Nesterov(2005)]{nesterov2005smooth}
Y.~Nesterov.
\newblock Smooth minimization of non-smooth functions.
\newblock \emph{Mathematical programming}, 103\penalty0 (1):\penalty0 127--152,
  2005.

\bibitem[Pflug and Ruszczy{\'n}ski(2005)]{pflug2005measuring}
G.~C. Pflug and A.~Ruszczy{\'n}ski.
\newblock {Measuring Risk for Income Streams}.
\newblock \emph{Computational Optimization and Applications}, 32\penalty0
  (1):\penalty0 161--178, 2005.

\bibitem[Reddi et~al.(2016)Reddi, Sra, Póczos, and
  Smola]{Reddi2016FastIncremental}
S.~J. Reddi, S.~Sra, B.~Póczos, and A.~Smola.
\newblock Fast incremental method for smooth nonconvex optimization.
\newblock In \emph{IEEE 55th Conference on Decision and Control (CDC)}, pages
  1971--1977, 2016.

\bibitem[Rockafellar and Royset(2014)]{Rockafeller2014RandomVariables}
R.~T. Rockafellar and J.~O. Royset.
\newblock Random variables, monotone relations, and convex analysis.
\newblock \emph{Mathematical Programming}, 148\penalty0 (1–2):\penalty0
  297–331, 2014.

\bibitem[Rockafellar and Uryasev(2013)]{Rockafeller2013Thefundamental}
R.~T. Rockafellar and S.~Uryasev.
\newblock The fundamental risk quadrangle in risk management, optimization and
  statistical estimation.
\newblock \emph{Surveys in Operations Research and Management Science},
  18:\penalty0 33--53, 2013.

\bibitem[Sarykalin et~al.(2008)Sarykalin, Serraino, and
  Uryasev]{sarykalin2008value}
S.~Sarykalin, G.~Serraino, and S.~Uryasev.
\newblock {Value-at-Risk vs. Conditional Value-at-Risk in Risk Management and
  Optimization}.
\newblock In \emph{State-of-the-art decision-making tools in the
  information-intensive age}, pages 270--294. INFORMS, 2008.

\bibitem[Shalev-Shwartz and Ben-David(2014)]{Shalev-Shwartz2014Understanding}
S.~Shalev-Shwartz and S.~Ben-David.
\newblock \emph{Understanding Machine Learning: From Theory to Algorithms}.
\newblock Cambridge University Press, 2014.

\bibitem[Shalev-Shwartz and Zhang(2013)]{Shalev-Shwartz2013Stochastic}
S.~Shalev-Shwartz and T.~Zhang.
\newblock {Stochastic Dual Coordinate Ascent Methods for Regularized Loss}.
\newblock \emph{Journal of Machine Learning Research}, 14\penalty0
  (1):\penalty0 567–599, 2013.

\bibitem[Shao(1989)]{SHAO1989Functional}
J.~Shao.
\newblock Functional calculus and asymptotic theory for statistical analysis.
\newblock \emph{Statistics \& Probability Letters}, 8\penalty0 (5):\penalty0
  397--405, 1989.

\bibitem[Shapiro et~al.(2014)Shapiro, Dentcheva, and
  Ruszczynski]{Shapiro2014Lectures}
A.~Shapiro, D.~Dentcheva, and A.~Ruszczynski.
\newblock \emph{Lectures on Stochastic Programming - Modeling and Theory,
  Second Edition}, volume~16.
\newblock {SIAM}, 2014.

\bibitem[Shorack(2017)]{Shorack2017Probability}
G.~Shorack.
\newblock \emph{Probability for Statisticians}.
\newblock Springer Texts in Statistics, 2017.

\bibitem[Tsanas and Xifara(2012)]{Tsanas2012AccurateQE}
A.~Tsanas and A.~Xifara.
\newblock Accurate quantitative estimation of energy performance of residential
  buildings using statistical machine learning tools.
\newblock \emph{Energy and Buildings}, 49:\penalty0 560--567, 2012.

\bibitem[Williamson and Menon(2019)]{williamson2019fairness}
R.~Williamson and A.~Menon.
\newblock {Fairness Risk Measures}.
\newblock In \emph{International Conference on Machine Learning}, pages
  6786--6797, 2019.

\bibitem[Xiang(1995)]{XIANG1995Anoteon}
X.~Xiang.
\newblock {A note on the bias of $L$-estimators and a bias reduction
  procedure}.
\newblock \emph{Statistics \& Probability Letters}, 23\penalty0 (2):\penalty0
  123--127, 1995.

\bibitem[Xiao et~al.(2017)Xiao, Rasul, and Vollgraf]{xiao2017fashion}
H.~Xiao, K.~Rasul, and R.~Vollgraf.
\newblock {Fashion-MNIST: a Novel Image Dataset for Benchmarking Machine
  Learning Algorithms}.
\newblock \emph{arXiv Preprint}, 2017.

\bibitem[Xiao(2009)]{Xiao2009Dual}
L.~Xiao.
\newblock {Dual Averaging Method for Regularized Stochastic Learning and Online
  Optimization}.
\newblock In \emph{Neural Information Processing Systems}, volume~22, 2009.

\bibitem[Yeh(2006)]{Yeh2006Analysis}
I.~Yeh.
\newblock {Analysis of Strength of Concrete Using Design of Experiments and
  Neural Networks}.
\newblock \emph{Journal of Materials in Civil Engineering}, 18, 2006.

\end{thebibliography}

\clearpage
\appendix
\begingroup
\let\clearpage\relax 
\onecolumn 
\endgroup

\addcontentsline{toc}{section}{Appendix} 
\part{Appendix} 
In the appendices, we give the proofs of consistency (\Cref{prop:consistency}) in \Cref{sec:a:consistency} and the variational properties of the objective (\Cref{prop:cvxity}) in \Cref{sec:a:cvxity}. \Cref{sec:sgd:proof} contains the analysis of bias SGD (\Cref{prop:sgd}). \Cref{sec:osvrg:proof} contains the analysis of \osvrg (\Cref{thm:main_osvrg}), with necessary background in \Cref{sec:a:smoothing}. We then give describe the experimental setup in detail (\Cref{sec:a:experiments}) and give some additional numerical results (\Cref{sec:a:additional}). 
 \parttoc
\clearpage

\section{Consistency of the Empirical Spectral Risk}
\label{sec:a:consistency}
We first recall the setting of \cref{prop:consistency}. 
\new{
Let $\p{\Omega, \msc{F}, \prob}$ denote a common probability space, upon which we consider an i.i.d.~sample  $\{\rv_1, \dots, \rv_n\}$ with each $\rv_i: \Omega \rightarrow \R$ being $(\msc{F}, \msc{B}(\R))$-measurable, where $\msc{B}(\R)$ denotes the Borel sets on the real line. Each shares a common cumulative distribution function (CDF) $F$ and quantile function $F^{-1}$ given by
\begin{align*}
    F(z) := \P{}{\rv_1^{-1}\p{(-\infty, z]}} \text{ and } F^{-1}(t) := \inf\br{z: F(z) \geq t}.
\end{align*}
Similarly, define the empirical CDF and quantile functions by
\begin{align*}
    F_n(z; \omega) := \frac{1}{n} \sum_{i=1}^n \I{(-\infty, z]}{\rv_i(\omega)} \text{ and } F_n^{-1}(t; \omega) := \inf\br{z: F_n(z; \omega) \geq t}.
\end{align*}
Construct the random variables $F_n(z): \omega \mapsto F_n(z; \omega)$ and $F^{-1}_n(t): \omega \mapsto F_n^{-1}(t; \omega)$. Here, $z \in \R$ and $t \in (0, 1)$, and the infimum is always attained \citep[Page 83]{Bobkov2019OnedimensionalEM}. We can ensure measurability of $F_n^{-1}(t)$ by taking the infimum only over $z \in \mathbb{Q}$. All expected values will be taken with respect to $\p{\Omega, \msc{F}, \prob}$ and will be denoted by $\mbb{E}$. For $s$ a probability density function (PDF) on $(0, 1)$, the $L$-functional $\L_s$ with spectrum $s$ is defined as 
\begin{align}
    \L_s[F] := \int_0^1 s(t) \cdot F^{-1}(t) \d t.
    \label{eqn:srm}
\end{align}
We first establish that \eqref{eqn:srm} is well-defined, using a well-known result of quantile functions.
\begin{prop}{\citep[Proposition A.1]{Bobkov2019OnedimensionalEM}}
    Let $\rv$ be a random variable and let $F$ be its cumulative distribution function. If $U$ is a random variable distributed uniformly in $(0, 1)$, then the random variable $F^{-1}(U)$ has $F$ as its distribution function. In particular,
    \begin{align*}
        \mbb{E}\abs{\rv}^p = \int_0^1 \abs{F^{-1}(t)}^p \d t
    \end{align*}
    when the left hand side is finite.
    \label{prop:bobkohvA.1}
\end{prop}
\begin{lem}
    Let $s$ be bounded, and $\mbb{E}\abs{\rv_1} < \infty$. Then $\abs{\L_s[F]} < \infty$.
\end{lem}
\begin{proof}
    Let $\norm{s}_\infty := \sup_{t \in (0, 1)} \abs{s(t)} < \infty$.
    Write
    \begin{align*}
        \abs{\L_s[F]} &= \abs{\int_0^1 s(t)  \cdot F^{-1}(t) \d t} \leq \norm{s}_\infty \cdot \int_0^1 \abs{F^{-1}(t)} \d t \overset{\text{\Cref{prop:bobkohvA.1}}}{\leq} \norm{s}_\infty \mbb{E}\abs{\rv_1} < \infty.
    \end{align*}
\end{proof}
We restate \cref{prop:consistency} below. 
\consistency*

The proof is summarized by the following steps. 
\begin{enumerate}
    \item By boundedness of the spectrum, we have that $\mbb{E}\abs{\L_s\sbr{F_n} - \L_s\sbr{F}}^2 \leq \norm{s}_\infty^2 \cdot \E{}{\p{\int_0^1 \abs{F_n^{-1}(t) - F^{-1}(t)} \d t}^2}$.
    \item Using the triangle inequality on $L^2(\prob)$ and relationships between quantile functions and CDFs, we relate $\sqrt{\E{}{\p{\int_0^1 \abs{F_n^{-1}(t) - F^{-1}(t)} \d t}^2}}$ to the quantity $\frac{1}{\sqrt{n}} \int_{-\infty}^{+\infty} \sqrt{F(z) (1 - F(z))} \d z$.
    \item We then use elementary concentration inequalities to bound $\int_{-\infty}^{+\infty} \sqrt{F(z) (1 - F(z))} \d z$ by $\frac{p}{p-2}\E{}{\abs{\rv}^p}^{1/p}$.
\end{enumerate}

The following theorem details how the $L^1$ distance between the quantile functions of two probability distributions is equal to the $L^1$ distance between the corresponding CDFs. 
\begin{theorem}[Theorem 2.10 of \citet{Bobkov2019OnedimensionalEM}]
    Let $\mu$, $\nu$ be two probability distributions on $\reals$ with associated CDF $F$ and $G$, respectively, with quantile functions $F^{-1}(t) := \inf\{z \in \R: F(z) \geq t\}$ and $G^{-1} := \inf\{z \in \R: G(z) \geq t\}$. Given that $\mu$ and $\nu$ have finite first moment, i.e., $\int \abs{z} \d \mu(z) < \infty$ and $\int \abs{z} \d \nu(z) < \infty$, we have that
    \begin{align*}
        W_1(\mu, \nu) = \int_{-\infty}^\infty \abs{F(z) - G(z)} \d z &= \int_0^1 \abs{F^{-1}(t) - G^{-1}(t)} \d t,
    \end{align*}
    where both the left and right hand sides are finite.
    \label{thm:bobkhov2.10}
\end{theorem}

Next, we ensure that the $L^1$ distance between $F_n^{-1}$ and $F_n$ is a square-integrable random variable.
\begin{lem}
    Assume that $\mbb{E}\abs{\rv_1}^2 < \infty$. Then, the random variable $V_n(\omega) := \int_0^1 \abs{F_n^{-1}(t; \omega) - F^{-1}(t)} \d t$ is well-defined, and $\mbb{E}[V_n^2] < \infty$.
    \label{lem:Dnl2}
\end{lem}
\begin{proof}
    For any particular realization $\omega \in \Omega$, write
    \begin{align*}
        \abs{V_n(\omega)}^2 &= \abs{\int_0^1 \abs{F_n^{-1}(t; \omega) - F^{-1}(t)} \d t}^2\\
        &\leq \int_0^1 \abs{F_n^{-1}(t; \omega) - F^{-1}(t)}^2 \d t &\text{Jensen's inequality}\\
        &\leq 2\int_0^1 \abs{F_n^{-1}(t; \omega)}^2 \d t + 2\int_0^1 \abs{F^{-1}(t)}^2 \d t\\
        &= \frac{2}{n} \sum_{i=1}^n \abs{\rv_i(\omega)}^2 + 2\mbb{E}\abs{\rv_1}^2 & \text{\Cref{prop:bobkohvA.1}}.
    \end{align*}
    Then, $\mbb{E}[V_n^2] \leq 4\mbb{E}\abs{\rv_1}^2$, completing the proof.
\end{proof}
The next lemma applies the above theorem to bound the expected distance between empirical and population quantile functions in terms of the population CDF, expanding upon remarks made on page 20 of \citet{Bobkov2019OnedimensionalEM}.
\begin{lem}
    Assume that $\mbb{E}\abs{\rv_1}^2 < \infty$. Then,
    \begin{align*}
        \sqrt{\E{}{\p{\int_0^1 \abs{F_n^{-1}(t) - F^{-1}(t)} \d t}^2}} \leq \frac{1}{\sqrt{n}} \int_{-\infty}^{+\infty} \sqrt{F(z) (1 - F(z))} \d z,
    \end{align*}
    where the right hand side is permitted to be infinite.
    \label{lem:J1}
\end{lem}
\begin{proof}
    By \Cref{lem:Dnl2} we have that 
    \begin{align}
        \E{}{\p{\int_0^1 \abs{F_n^{-1}(t) - F^{-1}(t)} \d t}^2} = \E{}{V_n^2} < \infty,
        \label{eqn:l2}
    \end{align}
    so that the left hand side is well-defined and finite. By \Cref{thm:bobkhov2.10}, we also have that $\int_0^1 \abs{F_n^{-1}(t; \omega) - F^{-1}(t)} \d t = \int_{-\infty}^\infty \abs{F_n(z; \omega) - F(z)} \d z$, indicating with \eqref{eqn:l2} that the random variable
    \begin{align*}
        \omega \mapsto \int_{-\infty}^\infty \abs{F_n(z; \omega) - F(z)} \d z \in L^2(\prob).
    \end{align*}
    By the triangle inequality on $L^2(\prob)$, we have that
    \begin{align*}
        \sqrt{\E{}{\p{\int_0^1 \abs{F_n^{-1}(t) - F^{-1}(t)} \d t}^2}} &= \sqrt{\E{}{\p{\int_{-\infty}^\infty \abs{F_n(z) - F(z)} \d z}^2}}\\
        &= \norm{\int_{-\infty}^\infty \abs{F_n(z) - F(z)} \d z}_{L^2(\prob)}\\
        &\leq \int_{-\infty}^\infty \norm{\abs{F_n(z) - F(z)}}_{L^2(\prob)} \d z\\
        &= \int_{-\infty}^\infty \sqrt{\E{}{\abs{F_n(z) - F(z)}^2}} \d z.
    \end{align*}
    Next, notice that for fixed $z \in \R$, $nF_n(z) \sim \Binom(n, F(z))$, so that
    \begin{align*}
        \E{}{\abs{F_n(z) - F(z)}^2} = \Var{F_n(z)} = \frac{F(z) \p{1 - F(z)}}{n},
    \end{align*}
    completing the proof.
\end{proof}
The final lemma bounds the right hand side of \Cref{lem:J1}.
\begin{lem}
Consider a random variable $\rv$ with c.d.f. $F$. If $\rv$ satisfies $\E{}{\abs{\rv}^p} < \infty$ for $p > 2$, then
\[
\int_{-\infty}^{+\infty} \sqrt{F(z)(1-F(z))}\d z \leq \p{\frac{2p}{p-2}} \E{}{\abs{\rv}^p}^{\frac{1}{p}}.
\]
\label{lem:J1bound}
\end{lem}
\begin{proof}
    By definition, $\int_{-\infty}^{\infty} \sqrt{F(z)(1-F(z))}\d z = \lim_{a \rightarrow +\infty} \int_{-a}^a \sqrt{F(z)(1-F(z))}\d z$. Denote $c = \E{}{\abs{\rv}^p}^{1/p}$. For any constant $a \geq c > 0$, we have 
    \begin{align*}
        \int_{-a}^a \sqrt{F(z)(1-F(z))}\d z & = 
         \int_{-a}^0 \sqrt{F(z)(1-F(z))}\d z + \int_0^a \sqrt{F(z)(1-F(z))}\d z \\
        & \leq  \int_{-a}^0 \sqrt{F(z)}\d z + \int_0^a \sqrt{(1-F(z))}\d z \\
        & = \int_{-a}^0 \sqrt{\mathbb{P}(\rv\leq z)} \d z + \int_0^a \sqrt{\mathbb{P}(\rv > z)}\d z \\
        & = \int_{0}^a \sqrt{\mathbb{P}(\rv\leq -z)} \d z + \int_0^a \sqrt{\mathbb{P}(\rv > z)}\d z \\
        & \leq 2 \int_0^a \sqrt{\mathbb{P}(|\rv|\geq z)}\d z\\
        & \leq 2 \int_0^a \sqrt{\min\br{1, \frac{c^p}{z^p}}}\d z & \text{Markov's inequality}\\
        &= 2\p{c + c^{p/2}\int_c^a z^{-p/2} \d z}.
    \end{align*}
    Computing the integral yields
    \begin{align*}
        \int_c^a z^{-p/2} \d z = \frac{a^{1-p/2} - c^{1-p/2} }{1 - p/2}.
    \end{align*}
    Because $1 - p/2 < 0$, we have that $\lim_{a \rightarrow \infty} \int_c^a z^{-p/2} \d z = \frac{c^{1 - p/2}}{p/2 - 1}$. 
    Combining the steps above, we obtain
    \begin{align*}
        \int_{-\infty}^{\infty} \sqrt{F(z)(1-F(z))}\d z &= \lim_{a \rightarrow \infty} \int_{-a}^a \sqrt{F(z)(1-F(z))}\d z\\
        &\leq \lim_{a \rightarrow \infty} 2\p{c + c^{p/2}\int_c^a z^{-p/2} \d z}\\
        &= 2c\p{1 + \frac{1}{p/2 - 1}}\\
        &= \frac{2pc}{p-2}.
    \end{align*}
    Resubstituting $c = \E{}{\abs{\rv}^p}^{1/p}$ completes the proof.
\end{proof}

We now have the tools to prove \Cref{prop:consistency}.
\begin{proof}[Proof of \cref{prop:consistency}]
    For a particular realization $\rv_1(\omega), ..., \rv_n(\omega)$, we have that 
    \begin{align*}
        \abs{\L_s\sbr{F_n(\cdot; \omega)} - \L_s\sbr{F}} &= \abs{\int_0^1 s(t) \cdot F_n^{-1}(t; \omega) \d t - \int_0^1 s(t) \cdot F^{-1}(t) \d t}\\
        &= \abs{\int_0^1 s(t) \cdot \p{F_n^{-1}(t; \omega) - F^{-1}(t)} \d t} \\
        &\leq \sup_{t \in (0, 1)}\abs{s(t)} \cdot \int_0^1 \abs{F_n^{-1}(t; \omega) - F^{-1}(t)} \d t \\
        &= \norm{s}_\infty \cdot \int_0^1 \abs{F_n^{-1}(t; \omega) - F^{-1}(t)} \d t.
    \end{align*}
    We then take the square and expectation.
    \begin{align*}
        \mbb{E}\abs{\L_s\sbr{F_n} - \L_s\sbr{F}}^2 &\leq \norm{s}_\infty^2 \cdot \E{}{\p{\int_0^1 \abs{F_n^{-1}(t) - F^{-1}(t)} \d t}^2}\\
        &\leq \frac{\norm{s}_\infty^2}{n} \cdot \p{\int_{-\infty}^{+\infty} \sqrt{F(z) (1 - F(z))} \d z}^2 &\text{\Cref{lem:J1}}\\
        &\leq \frac{\norm{s}_\infty^2}{n}  \p{\frac{2p}{p-2}}^2 \E{}{\abs{\rv}^p}^{\frac{2}{p}}. &\text{\Cref{lem:J1bound}}
    \end{align*}
\end{proof}
}

\section{Proof of Convexity and Subdifferential Properties}
\label{sec:a:cvxity}
Recall the expression of the empirical L-statistics
\begin{align}
    \msc{R}_\sigma(w) := \sum_{i=1}^n \sigma_i \ell_{(i)}(w).
\end{align}
where $0 \leq \sigma_1 \leq \dots \leq \sigma_n$, $\sum_{i=1}^n \sigma_i = 1$, each $\ell_i: \R^d \rightarrow \R$ is a function representing performance of model weights $w$ on training instance $i$, and for a vector $\lossval \in \reals^n$, we denote $\lossval_{(1)}\leq \ldots \leq \lossval_{(n)}$ its ordered coefficients. We recall \cref{prop:cvxity} and present its proof. 
 
\cvxity*

\begin{proof}
Since the coefficients $\sigma = (\sigma_1, \ldots, \sigma_n)$ are non-decreasing, the function $\msc{R}_\sigma$ can be written as the maximum over all possible permutations of the losses, i.e.,
\begin{align*}
    \msc{R}_\sigma(w) = \max_{\pi \in \Pi_n} \sum_{i=1}^n \sigma_i \ell_{\pi(i)}(w) = \max_{\pi \in \Pi_n} \sum_{i=1}^n \sigma_{\pi^{-1}(i)} \ell_{i}(w),
\end{align*}
where $\Pi_n$ is the set of permutations of $\{1, \ldots, n\}$. For any $\pi \in \Pi_n$, $w \mapsto \sum_{i=1}^n \sigma_{\pi^{-1}(i)} \ell_{i}(w)$ is a convex combination of convex functions, hence it is convex. Since the pointwise maximum of convex functions is convex, $\msc{R}_\sigma$ is convex. 

The pointwise maximum  $f = \max_{j=1, ..., N} f_j$ of $N$ convex functions $\{f_j\}_{j=1}^N$ has a subdifferential defined by
$\partial f(x) = \conv{\bigcup_{j \in \argmax\br{f_j(x)}} \partial f_j(x)}$ where $\conv(A)$ denotes the convex hull of a set $A$~\citep[Lemma 4.4.1]{HiriartUrruty1993Convex}. Letting $N = n!$, consider the finite set of convex functions $\{f_\pi: \pi \in \Pi_n\}$ with each $f_{\pi}: w\mapsto \sum_{i=1}^n \sigma_i \ell_{\pi(i)}(w)$. The subdifferential of $f_{\pi}$ is $\partial f_\pi(w) = \sum_{i=1}^n \sigma_i \partial \ell_{\pi(i)}(w)$, where the sum is to be understood as a Minkowski sum of sets~\cite[Lemma 4.4.1]{HiriartUrruty1993Convex}.
Hence, the subdifferential of $\msc{R}_\sigma(w)$ is defined by 
\begin{align*}
    \partial \msc{R}_\sigma(w) & = \conv\p{\bigcup_{\pi \in \argmax f_\pi(w)} \partial f_{\pi}(w)}  = \conv \p{\bigcup_{\pi \in \argsort\p{\ell(w)}} \sum_{i=1}^n \sigma_i \partial \ell_{\pi(i)}(w)},
\end{align*}
where we used that $\argsort\p{\ell(w)} = \argmax_\pi \sum_{i=1}^n \sigma_i\ell_{\pi(i)}(w)$ when $\sigma_1 \leq \dots \le \sigma_n$.

Finally if all $\ell_i$ are $G$-Lipschitz continuous, i.e., have $G$-bounded subgradients, then for any permutation $\pi$ of $\{1, \ldots, n\}$, any $g \in \sum_{i=1}^n \sigma_i \partial \ell_{\pi(i)}(w)$ is bounded by $G$ as a convex combination of $G$-bounded vectors. Hence any $g \in \partial \msc{R}_\sigma(w)$ is bounded by $G$ as a convex combination of $G$-bounded vectors. The function $\msc{R}_\sigma$ is then convex with subgradients bounded by $G$, hence it is $G$-Lipschitz continuous. 

\end{proof}

\section{Biased SGD Convergence Analysis}
\label{sec:sgd:proof}
Recall the regularized $L$-risk considered
\begin{align}
    \min_{w \in \R^d} \msc{R}_\sigma(w) + \frac{\reg}{2} \|w\|_2^2 \quad \mbox{for} \ \msc{R}_\sigma(w) = \sum_{i=1}^n \sigma_i \ell_{(i)}(w),
\end{align}
where $\reg>0$,  $0 \leq \sigma_1 \leq \dots \leq \sigma_n$, $\sum_{i=1}^n \sigma_i = 1$,  each $\ell_i: \R^d \rightarrow \R$ is a function representing performance of model weights $w$ on training instance $i$ and $\ell_{(1)}(w) \leq ... \leq \ell_{(n)}(w)$. In the following we consider $G$-Lipschtiz continuous convex losses.
The aim of this section is to prove the following proposition. 
\sgd*
\new{
The proof will proceed in three parts, which comprise the next three subsections. The final subsection proves the main result.
\begin{enumerate}
    \item The convexity and $G$-Lipschitz continuity of the losses is used to analyze the convergence of~\cref{alg:sgd} on a surrogate objective for which there is no bias.
    \item We then establish a uniform bias bound between the surrogate function and the original function over a set $\W \sse \R^d$.
    \item We then relate the suboptimality gap of the surrogate objective to the suboptimality of the original objective by using the bias bound and establishing that the iterates and minimizers are contained in $\W$.
\end{enumerate}

\subsection{SGD Analysis for Convex, Lipschitz Loss and Strongly Convex Regularizer}

We present first a generic convergence result for stochastic subgradient algorithms such as \cref{alg:sgd} applied to regularized non-smooth functions in \cref{lem:nonsmooth_sgd}, which is a minor adaptation of \citet[Theorem 5.5]{Bach2023Learning}.
\begin{lemma}\label{lem:nonsmooth_sgd}
	Consider a $G$-Lipschitz continuous, convex function $f:\reals^d \rightarrow \reals$ and a regularization $\reg\|\cdot\|_2^2$ for $\reg>0$ defining an objective of the form $f_\reg(w) = f(w) + \reg \|w\|_2^2/2$. Given a an initial point $w^{(0)} = 0 \in \reals^d$ consider iterates of the form, for $t \geq 0$,  
	$$
	w\pow{t+1} = w\pow{t} - \eta\pow{t} (v\pow{t} + \reg  w \pow t),
	$$
	for $v\pow{t}$ a random vector whose distribution depends only on $w\pow{t}$ satisfying $\E{}{v\pow{t}\mid w\pow{t}} \in \partial f(w\pow{t})$ and $\norm{v\pow{t}}_2 \leq G$, and $\eta\pow{t} > 0$.
	Consider outputting after $T$ iterations the estimate  $\bar w \pow T =\frac{1}{T} \sum_{t=0}^{T-1}   w\pow t$. Provided that  $\eta\pow{t} = \frac{1}{\reg (t+1)}$, this estimate satisfies
	\begin{align*}
		\mathbb{E}[f_\reg(\bar w\pow T)] - f_\reg(w^*) & \leq \frac{2G^2(1 + \log T)}{\mu T},
	\end{align*}
	for $w^*=\argmin_{w \in \R^d}f_\reg(w)$, where the expectation is taken over the sequence $w\pow{1}, \ldots, w\pow{T-1}$. 
\end{lemma}
\begin{proof}
    Write the expansion
    \begin{align}
        \norm{w\pow{t+1} - w^*}_2^2 =\norm{w\pow{t} - w^*}_2^2 - 2\stepsize\pow{t} \ip{v\pow{t} + \reg w\pow{t}, w\pow{t} - w^*} + (\stepsize\pow{t})^2 \norm{v\pow{t} + \reg w\pow{t}}_2^2.
        \label{eqn:expansion}
    \end{align}
    Because $\norm{w\pow{0}}_2 = \norm{0}_2 \leq G / \mu$, $\|v\pow t\|\leq G$ and $w\pow {t+1}$ can be expressed as a convex combination of $w\pow t$ and $-v\pow t/\mu$, that is,
    \begin{align*}
        w\pow{t+1} = (1 - \eta\pow{t} \reg) w\pow{t} + \stepsize\pow{t} \reg \p{-\frac{1}{\reg} v\pow{t}},
    \end{align*}
    we can conclude by induction that $\norm{w\pow{t}}_2 \leq G / \mu$ for all $t = 0, \ldots, T-1$ when $\eta\pow{t} \reg \leq 1$, which is satisfied for our choice of $\eta\pow{t}$. Thus, $(\stepsize\pow{t})^2 \norm{v\pow{t} + \reg w\pow{t}}_2^2 \leq (\stepsize\pow{t})^2 \cdot 4G^2$. Taking the conditional expectation of $v\pow{t}$ given $w\pow{t}$ of~\eqref{eqn:expansion} yields
    \begin{align}
        \E{}{\norm{w\pow{t+1} - w^*}_2^2 \mid w\pow{t}} \leq \norm{w\pow{t} - w^*}_2^2 - 2\stepsize\pow{t} \ip{\E{}{v\pow{t} \mid w\pow{t}} + \reg w\pow{t}, w\pow{t} - w^*} + (\stepsize\pow{t})^2 4G^2.
        \label{eqn:expansion2}
    \end{align}
    Because $\E{}{v\pow{t} \mid w\pow{t}} \in \partial f(w\pow{t})$, we have that $\E{}{v\pow{t} \mid w\pow{t}} + \reg w\pow{t} \in \partial f_\reg(w\pow{t})$, and by the $\reg$-strong convexity of $f_\mu$, we have
    \begin{align*}
        - \ip{\E{}{v\pow{t} \mid w\pow{t}} + \reg w\pow{t}, w\pow{t} - w^*} \leq  -\p{f_\reg(w\pow{t}) - f_\reg(w^*)} - \frac{\reg}{2}\norm{w\pow{t} - w^*}_2^2,
    \end{align*}
    which, substituted into~\eqref{eqn:expansion2} gives
    \begin{align*}
        \E{}{\norm{w\pow{t+1} - w^*}_2^2 \mid w\pow{t}} &\leq \p{1 - \eta\pow{t} \mu}\norm{w\pow{t} - w^*}_2^2 - 2\stepsize\pow{t} \p{f_\reg(w\pow{t}) - f_\reg(w^*)} + (\stepsize\pow{t})^2 4G^2\\
        \implies f_\reg(w\pow{t}) - f_\reg(w^*) &\leq \frac{1}{2}\p{\p{\frac{1}{\eta\pow{t}} - \reg}\norm{w\pow{t} - w^*}_2^2 - \frac{1}{\stepsize\pow{t}}  \E{}{\norm{w\pow{t+1} - w^*}_2^2 \mid w\pow{t}}} + 2\eta\pow{t} G^2\\
        &= \frac{1}{2}\p{\reg t\norm{w\pow{t} - w^*}_2^2 - \reg (t+1)  \E{}{\norm{w\pow{t+1} - w^*}_2^2 \mid w\pow{t}}} + \frac{2G^2}{\mu (t+1)}.
    \end{align*}
    Take the expectation over the entire sequence $w\pow{0}, \ldots, w\pow{t}$, sum over $t = 0, \ldots, T - 1$, and divide by $T$ to get
    \begin{align*}
        \E{}{\frac{1}{T} \sum_{t=0}^{T-1} f_\reg(w\pow{t})} - f_\reg(w^*) \leq \frac{1}{T}\sum_{t=0}^{T-1}\frac{2G^2}{\mu (t+1)} \leq \frac{2G^2(1 + \log T)}{\mu T},
    \end{align*}
    which combined with $f_\reg(\bar{w}\pow{T}) \leq \frac{1}{T} \sum_{t=0}^{T-1} f_\reg(w\pow{t})$ completes the proof.
\end{proof}

\subsection{Bias Control}

In this section, we control the bias term appearing in the convergence analysis. The following lemmas consider a set of real numbers, representing losses at a single $w \in \R^d$. Let $x_1, \ldots, x_n \in \R$ be call the \emph{full batch}, and let $X_1, \ldots X_m$ be a random sample selected uniformly \emph{without} replacement from $\{x_1, \ldots, x_n\}$, called the \emph{minibatch}. Let
\begin{align*}
    F_n(x) := \frac{1}{n} \sum_{i=1}^n \I{}{x_i \leq x} \text{ and } F_{n, m}(x) := \frac{1}{m} \sum_{j=1}^m \I{}{X_j \leq x} 
\end{align*}
be the empirical CDFs, and let 
\begin{align*}
    F^{-1}_n(t) := \inf\br{x: F_n(x) \geq t} \text{ and } F_{n, m}(t) :=  \inf\br{x: F_{n, m}(x) \geq t}.
\end{align*}
be the empirical quantile functions of the full batch and minibatch respectively. Similarly, let
\begin{align*}
    \mu_n :=\sum_{i=1}^n \delta_{x_i} \text{ and } \mu_{n, m} = \sum_{j=1}^m \delta_{X_j}
\end{align*}
be the empirical measures of the full batch and minibatch, respectively, with $\delta_x$ indicating a Dirac point mass at $x$. Let $u(t) := \I{(0, 1)}{t}$ be the uniform spectrum. 
\begin{lem}
We have that for any $n\in \mathbb{N}$, $m\leq n$, 
    \begin{align*}
        \E{}{\L_{u}[F_{n, m}]} = \L_u[F_n],
    \end{align*}
    where the expectation is taken over the sampling of $X_1, \ldots, X_m$ without replacement.
    \label{lem:unbiased_spectrum_main}
\end{lem}
\begin{proof}
    \begin{align*}
        \E{}{\L_{u}[F_{n, m}]} = \E{}{\frac{1}{m}\sum_{j=1}^m X_{(j)}} = \E{}{\frac{1}{m}\sum_{j=1}^m X_{j}} = \frac{1}{\binom{n}{m}} \sum_{i_1 < \ldots < i_m} \frac{1}{m} \sum_{j=1}^m x_{i_j} = \frac{1}{n} \sum_{i=1}^n x_i = \L_{u}[F_{n}].
    \end{align*}
\end{proof}

Recall the definition of the $q$-Wasserstein distance (raised to the $q$-th power) between real-valued probability measures $\mu$ and $\nu$ with finite $q$-th moment, given by
\begin{align*}
    W_q^q(\mu, \nu) := \inf_{\gamma \in \Pi(\mu, \nu)} \int_\R \abs{x - y}^q \d \gamma(x, y),
\end{align*}
where $\Pi(\mu, \nu)$ is the set of couplings (or joint distributions) with marginals being $\mu$ and $\nu$.
\begin{lem}
    For any $q \geq 1$, $n\in \mathbb{N}, m\leq n$ the $q$-Wasserstein distance between the empirical measures satisy
    \begin{align*}
        W^q_q(\mu_{n, m}, \mu_n) \leq \frac{n-m}{n} \p{x_{(n)} - x_{(1)}}^q,
    \end{align*}
    where $x_{(n)} = \max_i x_i$ and $x_{(1)} = \min_i x_i$.
    \label{lem:wasserstein_bound}
\end{lem}
\begin{proof}
    We assume for simplicity that $m$ divides $n$. By definition, we have that
    \begin{align*}
        W^q_q(\mu_{n, m}, \mu_n) = \inf_{\pi \in \Pi(\mu_n, \mu_{n, m})} \sum_{i=1}^n \sum_{j=1}^m \abs{x_i - X_j}^{q} \pi(x_i, X_j),
    \end{align*}
    where $\Pi(\mu_n, \mu_{n, m})$ is the set of couplings, or joint distributions with marginals $\mu_n$ and $\mu_{n, m}$. We construct one such coupling that achieves the bound in the claim. Without loss of generality, assume that $X_1 = x_1, \ldots, X_m = x_m$. Then, partition $x_1, \ldots, x_n$ into $m$ non-overlapping subsets $S_1, \ldots, S_m$ each of size $\frac{n}{m}$, such that $x_j \in S_j$ for $j = 1, \ldots, m$. Define $\pi^*$ by
    \begin{align*}
        \pi^*(y, y') = \begin{cases}
        \frac{1}{n} &\text{ if } y = x_i, \ y' = X_j, \ i \in S_j\\
        0 &\text{ otherwise.}
        \end{cases}
    \end{align*}
    This is indeed a coupling, and
    \begin{align*}
        W^q_q(F_{n, m}, F_n) &\leq \sum_{i=1}^n \sum_{j=1}^m \abs{x_i - X_j}^{q} \pi^*(x_i, X_j) = \sum_{j=1}^m \sum_{i \in S_j, i \neq j}^n \abs{x_i - X_j}^{q} \pi^*(x_i, X_j)\\
        &\leq \sum_{j=1}^m \p{\frac{n}{m} - 1} \abs{x_{(n)} - x_{(1)}}^q \frac{1}{n} = \frac{n-m}{n} \p{x_{(n)} - x_{(1)}}^q.
    \end{align*}
\end{proof}
Next, we can give the bias bound for an $L$-functional applied to $\mu_{n, m}$ and $\mu_n$.
\begin{lem}
    \label{lem:mini_bias}
    Let $\L_s$ be an $L$-functional with spectrum $s$, and let $u$ be the uniform spectrum. Then,
    \begin{align*}
        \abs{\E{}{\L_{s}[F_{n, m}]} - \L_{s}[F_{n}]} \leq \frac{n-m}{n} \norm{s - u}_\infty \p{x_{(n)} - x_{(1)}},
    \end{align*}
    where the expectation is taken over the sampling of $X_1, \ldots, X_m$ without replacement, and $\norm{s - u}_\infty = \sup_{t \in (0, 1)} \abs{s(t) - u(t)}$.
\end{lem}
\begin{proof}
    Write
    \begin{align*}
        \abs{\E{}{\L_{s}[F_{n, m}]} - \L_{s}[F_{n}]} &= \abs{\E{}{\L_{s}[F_{n, m}]} - \L_{s}[F_{n}] - (\E{}{\L_{u}[F_{n, m}]} - \L_{u}[F_{n}])} &\text{\Cref{lem:unbiased_spectrum_main}}\\
        &= \abs{\E{}{\int_0^1 (s(t) - u(t)) \cdot \p{F^{-1}_{n, m}(t) - F^{-1}_n(t) \d t}}}\\
        &\leq \E{}{\abs{\int_0^1 (s(t) - u(t)) \cdot \p{F^{-1}_{n, m}(t) - F^{-1}_n(t) \d t}}} &\text{Jensen's inequality}\\
        &\leq \norm{s - u}_\infty \E{}{\norm{F^{-1}_{n, m} - F^{-1}_n}_1} &\text{H\"older's inequality}\\
        &= \norm{s - u}_\infty \E{}{W_1\p{\mu_{n, m}, \mu_n}}. &\text{\Cref{thm:bobkhov2.10}}
    \end{align*}
    Next, by \Cref{lem:wasserstein_bound}, we have that $W_1\p{\mu_{n, m}, \mu_n} \leq (n-m)/n \p{x_{(n)} - x_{(1)}}$, which combined with the above yields
    \begin{align*}
        \abs{\E{}{\L_{s}[F_{n, m}]} - \L_{s}[F_{n}]} &\leq \norm{s - u}_\infty \p{\frac{n-m}{n}} \p{x_{(n)} - x_{(1)}},
    \end{align*}
    which completes the proof.
\end{proof}

Next, consider the situation in which $x_1 = \ell_1(w), \ldots, x_n = \ell_n(w)$, i.e., the loss functions for each data point evaluated at $w$, with $X_j = \ell_{i_j}(w)$ for the randomly samplied minibatch $(i_1, \ldots, i_m)$. By defining 
\begin{align*}
    F_n(x; w) := \frac{1}{n} \sum_{i=1}^n \I{}{\ell_i(w) \leq x} \text{ and } F_{n, m}(x; w) := \frac{1}{m} \sum_{j=1}^m \I{}{\ell_{i_j}(w) \leq x},
\end{align*}
we have that
\begin{align*}
    \primobj(w) = \L_s[F_n(\cdot; w)] \text{ and } \mbobj(w) := \L_s[F_{n, m}(\cdot; w)].
\end{align*}
This gives the following corollary.
\begin{corollary}
    \label{coro:bias}
    Given a minibatch $(i_1, \ldots, i_m)$ sampled uniformly randomly without replacement from $[n]$, define $\mbobj(w) := \sum_{j=1}^m \hat{\sigma}_{j} \ell_{i_{(j)}}(w)$, where $\ell_{i_{(1)}}(w) \leq \ldots \leq \ell_{i_{(m)}}(w)$ are the minibatch order statistics.
    It holds that
    \begin{align*}
        \sup_{w: \norm{w}_2 \leq G/\mu} \abs{\E{}{\mbobj(w) \mid w} - \primobj(w)} \leq \norm{s - u}_\infty \p{\frac{n-m}{n}} B,
    \end{align*}
    where $B = \sup_{w : \norm{w} \leq G/\mu} \p{\ell_{(n)}(w) - \ell_{(1)}(w)} < \infty$.
\end{corollary}
\begin{proof}
    For any $w \in \R^d$, we have that
    \begin{align*}
        \abs{\E{}{\mbobj(w)\mid w} - \primobj(w)} &= \abs{\E{}{\L_{s}[F_{n, m}(\cdot; w)]} - \L_{s}[F_{n}(\cdot; w)]} \leq \norm{s - u}_\infty \p{\frac{n-m}{n}} \p{\ell_{(n)}(w) - \ell_{(1)}(w)} & \text{\Cref{lem:mini_bias}}.
    \end{align*}
    Take the supremum for $\{w: \norm{w}_2 \leq G/\mu\}$ on both sides. Because each $\ell_i$ is continuous, so is $\ell_{(n)}(w) = \ell_{(1)}(w)$. Because the supremum is taken over a compact set (the ball of radius $G/\mu$), it is finite.
\end{proof}
}

\subsection{Proof of Main Result}
\new{
\sgd*
\begin{proof}
    Define the surrogate function
    \begin{align*}
        \surrobj(w) := \E{}{\mbobj(w) \mid w} = \E{}{\sum_{j=1}^m \hat{\sigma}_{j} \ell_{i_{(j)}}(w) \mid w},
    \end{align*}
    where the expectation is taken over sampling the minibatch $(i_1, \ldots, i_m)$. The expectation is over a discrete distribution on a finite set, so it is a well-defined function of $w$. We now establish the properties of $\surrobj$ required to apply the generic analysis of the stochastic subgradient method given in \Cref{lem:nonsmooth_sgd} with $f = \surrobj$ and $v\pow{t} = \sum_{j=1}^m \hat{\sigma}_j \grad \ell_{i_{(j)}}(w)$. This choice of $\surrobj$ is clearly convex, and because $v\pow{t}$ is a subgradient of $\mbobj$, $\E{}{v\pow{t} \mid w\pow{t}}$ is a subgradient of $\surrobj$. Given the $G$-Lipschitzness of each $\ell_i$, we have that $\mbobj$ is also $G$-Lipschitz by \Cref{prop:cvxity}, and thus so is $\surrobj$. Then, letting $\surrobjreg(w) := \surrobj(w) + \frac{\reg}{2}\norm{w}_2^2$, applying \Cref{lem:nonsmooth_sgd} gives
    \begin{align}
        \E{}{\surrobjreg\p{\bar{w}\pow{T}}} - \surrobjreg\p{\bar{w}^*} \leq \frac{2G^2(1 + \log T)}{\reg T},
        \label{eqn:1}
    \end{align}
    where $\bar{w}^* = \argmin_{w \in \R^d}\surrobjreg\p{w}$. We must now pass this result regarding $\surrobjreg$ to a similar one regarding $\primobjreg$. Define $\W := \{w \in \R^d : \norm{w}_2 \leq G/\mu\}$. We first establish that $w^*, \bar{w}^* \in \W$, so that
    \begin{align}
        \min_{w \in \R^d} \surrobjreg\p{w} = \min_{w \in \W}\surrobjreg\p{w}
        \text{ and }
        \min_{w \in \R^d} \primobjreg\p{w} = \min_{w \in \W}\primobjreg\p{w}.
        \label{eqn:min_ball1}
    \end{align}
    The subdifferentials of $\surrobjreg$ and $\primobjreg$ are given by
    \begin{align*}
        \partial \surrobjreg(w) &= \partial \surrobj(w) + \reg w = \{g + \reg w: g \in \partial \surrobj(w)\},\\
        \partial \primobjreg(w) &= \partial \primobj(w) + \reg w = \{g + \reg w: g \in \partial \primobj(w)\}.
    \end{align*}
    Then, by optimality,
    \begin{align*}
        g^* + \reg w^* = 0 \text{ and } \bar{g} + \reg \bar{w}^* = 0
    \end{align*}
    for some $g^* \in \partial \primobj(w^*)$ and $\bar{g} \in \partial \surrobj(\bar{w}^*)$, yielding that $w^*, \bar{w}^* \in \W$ because $\norm{g^*}_2, \norm{\bar{g}}_2 \leq G$.
    Next, note that $\norm{w\pow{t}}_2 \leq G/\mu$ for any $t = 0, \ldots, T$. To see this, observe that $\norm{w\pow{0}}_2 = \norm{0}_2 \leq G/\mu$, and if $\norm{w\pow{t}}_2 \leq G/\mu$, then
    \begin{align*}
        \norm{w\pow{t+1}}_2 &= \norm{(1 - \eta\pow{t} \reg) w\pow{t} + \eta\pow{t} \reg \p{-\frac{1}{\reg} v\pow{t}}}_2 \leq (1 - \eta\pow{t} \reg) \norm{w\pow{t}}_2 + \eta\pow{t} \reg \norm{\frac{1}{\reg} v\pow{t}}_2 \leq \frac{G}{\mu}
    \end{align*}  
    if $\eta\pow{t} \reg \leq 1$, which is satisfied for $\eta\pow{t} = 1/(\reg t)$. By convexity of $\W$, this means that $\bar{w}\pow{T} \in \W$. Given that $\bar{w}^*, w^*,$ and $\bar{w}\pow{T}$ are contained in $\W$, if we can show that $\primobj$ and $\surrobj$ are close on this set, then optimizing $\surrobj$ should also result in a near-optimal value of $\primobj$. Assume there existed $\delta > 0$ such that
    $
        \sup_{w \in \W} \abs{\primobj(w) - \surrobj(w)} = \delta < \infty.
    $
    Then, $\primobj\p{\bar{w}\pow{T}} \leq \surrobj\p{\bar{w}\pow{T}} + \delta$, and, for any $w \in \W$,
    \begin{align}
        \primobj\p{w} \geq \surrobj\p{w} - \delta \implies 
        -\min_{w \in \W} \p{\primobj\p{w} + \frac{\reg}{2}\norm{w}_2^2} \leq -\min_{w \in \W} \p{\surrobj\p{w} + \frac{\reg}{2}\norm{w}_2^2} + \delta,
        \label{eqn:bias}
    \end{align}
    giving
    \begin{align*}
        \E{}{\primobjreg\p{\bar{w}\pow{T}}} - \min_{w \in \R^d} \primobjreg\p{w}
        &= \E{}{\primobjreg\p{\bar{w}\pow{T}}} - \min_{w \in \W}\primobjreg\p{w} & \eqref{eqn:min_ball1}\\
        &\leq 2\delta + \E{}{\surrobjreg\p{\bar{w}\pow{T}}} - \min_{w \in \W}\surrobjreg\p{w} &\eqref{eqn:bias}\\
        &= 2\delta + \E{}{\surrobjreg\p{\bar{w}\pow{T}}} - \min_{w \in \R^d} \surrobjreg\p{w} &\eqref{eqn:min_ball1}\\
        &\leq 2\delta + \frac{2G^2(1 + \log T)}{\reg T}. & \text{\eqref{eqn:1}}
    \end{align*}
    Establishing the existence of such a $\delta$ and showing $2\delta \leq 2 C_s B \p{\frac{n-m}{n}}$ completes the proof. This is accomplished by \Cref{coro:bias}, which gives
    \begin{align*}
        2\delta =  \sup_{w \in \W} \abs{\surrobj(w) - \primobj(w)} =
        \sup_{w: \norm{w}_2 \leq G/\mu} \abs{\E{}{\mbobj(w) \mid w} - \primobj(w)} 
        \leq \norm{s - u}_\infty \p{\frac{n-m}{n}} B,
    \end{align*}
    the desired result.
\end{proof}
}

\section{Smoothing the Empirical Spectral Risk Measure}
\label{sec:a:smoothing}
Recall that we consider objectives of the form (ignoring the regularization part)
\begin{equation}\label{eq:lstat_obj}
 \primobj(w) = \sum_{i=1}^{n} \sigma_i \ell_{(i)}(w) ,
\end{equation}
where $\sigma_i$ are the weights associated with the discretization of a spectral risk, i.e., $\sigma_i =	\int_{\frac{i-1}{n}}^{\frac{i}{n}} s(t) \D t$ for $s: [0, 1] \rightarrow [0, +\infty)$ non-decreasing and such that $\int_0^1s(t) \D t=1$.

We can rewrite problems~\eqref{eq:lstat_obj} as minimizing a composition
\begin{align}
	\primobj(w) = \lstat(\ell(w))\
	\mbox{with} \  \lstat(\lossval) =  \sum_{i=1}^{n}\sigma_{i} \lossval_{(i)} \ \mbox{and}\ \ \ell(w) = (\ell_1(w), \ldots, \ell_n(w)). \nonumber
\end{align}
Since the coefficients $\sigma_i$ are not decreasing, the outer function $\lstat$ can be expressed as 
\[
\lstat(\lossval) = \max_{\dualvar \in \mathcal{P}(\sigma)} \dualvar^\top \lossval
\]
where  $\mathcal{P}(\sigma) = \{ \dualvar = \Pi \sigma: \Pi\ones = \ones, \Pi^\top \ones = \ones, \Pi \in [0,1]^{n\times n}  \}$ is the permutahedron associated with the weights $\sigma$. 

The function $\lstat$ is non-differentiable at points $\lossval$ with ties. However, smooth approximations of $\lstat$ can be defined by means of a strongly convex regularizer as presented by~\citet{nesterov2005smooth}.
For $\Omega$ strongly convex w.r.t. some norm $\|\cdot\|$ and $\nu \geq 0$, we consider a smooth approximation of  $\lstat$ defined by
\[
\lstat_{\nu \Omega} (\lossval) := \max_{\dualvar \in \mathcal{P}(\sigma)} \left\{\lossval^\top \dualvar - \nu \Omega(\dualvar)\right\}, \quad \nabla \lstat_{\nu \Omega}(\lossval) = \argmax_{\dualvar \in \mathcal{P}(\sigma)} \left\{\lossval^\top \dualvar - \nu \Omega(\dualvar)\right\}.
\]
By standard convex duality arguments, the smooth approximation of $\lstat$ can also be written as the inf-convolution of $\lstat$ with the convex conjugate of $\nu \Omega$, i.e, 
\begin{equation}\label{eq:smoothing_primal}
    \lstat_{\nu \Omega} (\lossval) = \min_{\lossvalaux \in \mathbb{R}^n} \left\{\lstat(\lossvalaux) +  \nu \Omega^\star((\lossval-\lossvalaux)/\nu) \right\}, \ 	\nabla \lstat_{\nu \Omega} (\lossval)  = \nabla \Omega^\star((\lossval-\lossvalaux^*)/\nu) 	\ \mbox{for}\  \lossvalaux^*= \argmin_{\lossvalaux \in \mathbb{R}^n} \left\{\lstat(\lossvalaux) +  \nu \Omega^\star((\lossval-\lossvalaux)/\nu) \right\}.	
\end{equation}
We consider the surrogate objective  defined by
\begin{equation}\label{eq:smooth_lstat_obj}
 \primobjsmooth{\nu \Omega}(w) := \lstat_{\nu \Omega}(\ell(w)), \mbox{with}\ \lstat_{\nu \Omega} (\lossval) := \max_{\dualvar \in \mathcal{P}(\sigma)} \left\{\lossval^\top \dualvar - \nu \Omega(\dualvar)\right\} \ \mbox{and} \ \ell(w) = (\ell_1(w), \ldots, \ell_n(w)).
\end{equation}
Note that for any $\nu \geq 0$, if the losses $\ell_i$ are convex, then the surrogate objective $ \primobjsmooth{\nu \Omega}$ is also convex.

In the following, we recall the smoothness properties of the smooth approximation in \Cref{lem:smooth_prop}, we present the approximation incurred by the smoothing in terms of the spectrum in \Cref{lem:smooth_approx}. We give the implementation of the gradient evaluation of the smooth approximation for appropriate choices of regularizers in Sec.~\ref{ssec:pav}.

\subsection{Smoothing Properties and Approximation Bounds}
We recall below the smoothness properties of $\lstat_{\nu \Omega} (\lossval)$, see, e.g.,~\citep{nesterov2005smooth, beck2012smoothing} for detailed proofs. 

\begin{lemma}[Smoothing properties]\label{lem:smooth_prop}
	For any $\Omega$ and $\nu >0$, the smoothed $\lstat_{\nu \Omega}$ is $\|\sigma\|_p$-Lipschitz continuous w.r.t. $\|\cdot\|_p$ for any $p \in \{1, \ldots\} \cup \{+\infty\}$.
	For any $\Omega$ that is $1$-strongly convex w.r.t. $\|\cdot\|$,  the smoothed $\lstat_{\nu \Omega}$  is $1/\nu$ smooth w.r.t. to the dual norm $\|\cdot\|_*$, i.e., for any $\lossval, \lossval' \in \reals^n$, 
	$
	\|\nabla \lstat_{\nu \Omega}(\lossval) - \nabla \lstat_{\nu \Omega}(\lossval')\| \leq \|\lossval-\lossval'\|_*/\nu.
	$
\end{lemma}

Usual examples are  $\Omega=\|\cdot\|_2^2$ or $\Omega = H: \dualvar \mapsto  \sum_{i=1}^n \dualvar_i \ln \dualvar_i$, for which we have access to $\lstat_{\nu \Omega}$ by isotonic regression, e.g.,~\citep{lim2016efficient, blondel2020fast}.
In the following, we consider such functions centered around their minimizers in $\mathcal{P}(\sigma)$ to get tighter approximation bounds.  Namely, we define $\avg = \ones/n \in \mathcal{P}(\sigma)$  and we consider 
\begin{gather}
    \label{eq:chosen_smoothing1}
	\Omega_1(\dualvar) = D_H(\dualvar;\avg) := H(\dualvar) - H(\avg) - \nabla H(\avg) ^\top (\dualvar-\avg) = \sum_{i=1}^n \dualvar_i \log(n\dualvar_i), \quad\mbox{and}  \\
	\Omega_2(\dualvar) = \frac{1}{2} \|\dualvar-\avg\|_2^2. \label{eq:chosen_smoothing}
\end{gather}
We have that $\Omega_1$  is $1$-strongly convex w.r.t. $\|\cdot\|_1$ and $\Omega_2$ is $1$-strongly convex w.r.t. $\|\cdot\|_2$. 

We can then consider optimizing the surrogate objective for $\Omega \in \{\Omega_1, \Omega_2\}$,  defined by
$
 \primobjsmooth{\nu \Omega}(w) := \lstat_{\nu \Omega}(\ell(w)).
$
Note that for any $\nu \geq 0$, if the losses $\ell_i$ are convex, then the surrogate objective $ \primobjsmooth{\nu \Omega}$ is also convex.
\Cref{lem:smooth_approx} details the approximation done by considering the smoothed version of the objective.

\begin{lemma}[Approximation bounds]\label{lem:smooth_approx}
	For any strongly convex function $\Omega$ invariant by permutation and such that $\inf_{\dualvar \in \mathcal{P}(\sigma) } \Omega(\dualvar) \geq 0$, we have that for any $\nu \geq 0$, $w \in \reals^d$,
	\begin{align*}
	 0\leq 	\lstat(w) - \lstat_{\nu \Omega} (w)  \leq \nu \Omega(\sigma)
	\end{align*}
	If, in addition, $\Omega$ is decomposable as $\Omega(\dualvar) = \sum_{i=1}^n \omega(\dualvar_i)$ with $\omega$ convex and $\sigma$ is the discretization of a function $s$ such that $\sigma_i =  \int_{\frac{i-1}{n}}^{\frac{i}{n}} s(t) \D t$, then 
	\[
	\Omega(\sigma) \leq n\int_{0}^{1} \omega\left( \frac{s(t)}{n}\right) \D t.
	\]
\end{lemma}
\begin{proof}
    One one hand, we have that 
    \[
    \lstat_{\nu \Omega}(\lossval) = \sup_{\dualvar \in \mathcal{P}(\sigma)} \{\dualvar^\top \lossval - \nu \Omega(\dualvar) \} \leq \sup_{\dualvar \in \mathcal{P}(\sigma)}\dualvar^\top \lossval = \lstat(\lossval),
    \]
    since $\inf_{\dualvar \in \mathcal{P}(\sigma) } \Omega(\dualvar) \geq 0$.
    On the other hand,  we have that 
    \[
    \lstat_{\nu \Omega}(\lossval) = \sup_{\dualvar \in \mathcal{P}(\sigma)} \{\dualvar^\top \lossval - \nu \Omega(\dualvar) \} \geq \sup_{\dualvar \in \mathcal{P}(\sigma)} \{\dualvar^\top \lossval\} - \nu \sup_{\dualvar\in \mathcal{P}(\sigma)} \Omega(\lambda) = \lstat(\lossval) - \nu \sup_{\dualvar\in \mathcal{P}(\sigma)} \Omega(\lambda).
    \]

	Hence, we have  $ 0 \leq \lstat(\lossval) - \lstat_{\nu \Omega}(\lossval)\leq \nu\max_{\dualvar\in \mathcal{P}(\sigma)} \Omega(\dualvar)$.
	The maximum of a convex function $\Omega$ over a polytope $\Pcal(\sigma)$ is attained at a corner. The corners of the permutahedron $\Pcal(\sigma)$ are permutations of $\sigma$. 
	Since $\Omega$ is permutation invariant, we have that $\Omega(\sigma) = \Omega(\pi(\sigma))$ for any permutation $\pi$. Thus, $\max_{\dualvar \in \Pcal(\sigma)} \Omega(\dualvar) = \Omega(\sigma)$, completing the proof of the first part. 
	For the second claim, we use Jensen's inequality to get
	\begin{align*}
		\Omega(\sigma) = \sum_{i=1}^{n} \omega\left(  \int_{\frac{i-1}{n}}^{\frac{i}{n}} s(t) \D t \right) 
		& = \sum_{i=1}^{n} \omega\left(  \int_{\frac{i-1}{n}}^{\frac{i}{n}} \left(\frac{s(t)}{n}\right) n\, \D t \right)  \\
		& \leq \sum_{i=1}^{n}  \int_{\frac{i-1}{n}}^{\frac{i}{n}} \omega\left( \frac{s(t)}{n}\right) n \, \D t = n\int_{0}^{1} \omega\left( \frac{s(t)}{n}\right) \D t.
	\end{align*}
\end{proof}

\begin{corollary}\label{cor:smooth_approx}
	For $\sigma_i =  \int_{\frac{i-1}{n}}^{\frac{i}{n}} s(t) \D t$ with $s$ a spectrum such that $\int_0^1s(t)\D t=1$, and for any $\nu \geq 0$, $w \in \reals^d$, we have
	\begin{align*}
		0 \leq \lstat(w) -\lstat_{\nu \Omega_1} (w) &  \leq \nu D_H(\sigma; \avg) \leq \nu \int_{0}^{1}s(t) \ln s(t) \D t := \nu \operatorname{KL}(s\Vert u),  \\
		0 \leq \lstat(w)- \lstat_{\nu \Omega_2} (w)   & \leq \frac{\nu}{2}\|\sigma - \avg\|_2^2 \leq\frac{\nu}{2n} \int_{0}^{1} (s(t) -1)^2\D t := \frac{\nu}{2n} \chi^2(s\Vert u), 
	\end{align*}
where $\operatorname{KL}(s\Vert u)$ and $\chi^2(s\Vert u)$ denote respectively the Kullback-Leibler divergence and the Chi-square divergence between the spectrum $s$ and the uniform distribution $u$.  
\end{corollary}
For example, we can derive the bounds for some specific choices of spectra. 
\begin{enumerate}[nosep]
	\item (Superquantile) For $s(t) = \frac{1}{1-q} \mathbf{1}_{[q, 1]}(t)$, with $q \in [0, 1]$, we have
$\chi^2(s\Vert u)=   \frac{q}{1-q}$ and $
	 \operatorname{KL}(s\Vert u) = - \ln(1-q).$
	\item (Extremile)  For $s(t) =  rt^{r-1}$, with $r\geq 1$, we have 
$\chi^2(s\Vert u)= \frac{(r-1)^2}{(2r-1)}$ and $  \operatorname{KL}(s\Vert u)  = \ln r + \frac{1}{r} -1.$
\end{enumerate}
The approximations bounds computed for $\lstat_{\nu \Omega}$ and  $\lstat$ naturally apply for $\primobjsmooth{\nu\Omega}$ and $\primobj$, that is, for any $w\in \reals^d$, we have $0 \leq \primobj(w) -\primobjsmooth{\nu\Omega}(w) \leq \nu \Omega(\sigma)$. This gives the following lemma mentioned in the main text. 
\begin{lemma}\label{lem:approx_min}
    Consider the regularized objective $\primobjsmooth{\mu}(w) = \primobj(w) + \mu \|w\|_2^2/2$ for $\primobj$ defined as in~\eqref{eq:lstat_obj} by non-decreasing non-negative coefficients $\sigma_i$ summing up to 1 and $n$ functions $(\ell_i)_{i=1}^n$, and consider the smoothed approximation $\primobjsmooth{\mu, \nu \Omega}(w) = \primobjsmooth{\nu \Omega}(w) + \mu \|w\|_2^2/2$  for $\primobjsmooth{\nu \Omega}$ defined as in~\eqref{eq:smooth_lstat_obj} by $\nu>0$ and a strongly convex function $\Omega$ invariant by permutation and such that $\inf_{\dualvar \in \mathcal{P}(\sigma) } \Omega(\dualvar) \geq 0$.
    
    If $\hat w \in \reals^d$ is a $\varepsilon$-accurate minimum of the smoothed regularized objective, i.e.,  $\primobjsmooth{\mu, \nu\Omega}(\hat w) - \min_{w\in \reals^d} \primobjsmooth{\mu, \nu\Omega}(w) \leq \varepsilon$ then it is an $\varepsilon + \nu \Omega(\sigma)$ accurate minimum of the original  regularized objective $\primobjsmooth{\mu}$, where upper-bounds of $\Omega(\sigma)$ are provided in 
    \cref{lem:smooth_approx} and \cref{cor:smooth_approx}.
\end{lemma}
\begin{proof}
    Denote $w^* = \argmin_{w\in \reals^d} \primobjsmooth{\mu}(w)$. If $\hat w \in \reals^d$ satisfies $\primobjsmooth{\mu, \nu\Omega}(\hat w) - \min_{w\in \reals^d} \primobjsmooth{\mu, \nu\Omega}(w) \leq \varepsilon$ then
    \begin{align*}
        \primobjsmooth{\mu}(\hat w) - \primobjsmooth{\mu}(w^*) & \leq \primobjsmooth{\mu, \nu \Omega}(\hat w) - \primobjsmooth{\mu, \nu \omega}(w^*) + \nu \Omega(\sigma) \leq \varepsilon + \nu \Omega(\sigma)
    \end{align*}
    where we used that $0 \leq \primobj(w) - \primobjsmooth{\nu\Omega}(w)  \leq \nu \Omega(\sigma)$ since \cref{lem:smooth_approx} holds. 
\end{proof}

\begin{algorithm}[t]\caption{Pool Adjacent	 Violators (PAV) Algorithm  \label{algo:pav1} for $\omega_1$}
 	\begin{algorithmic}[1]
 		\State {\bf Inputs:} Number of coefficients $n$, coefficients $(\lossval_i)_{i=1}^n$ and $(s_i)_{i=1}^n$ with $s_i = \ln \sigma_i$
 		\State {\bf} Initialize  $P_1 = \{1\}$, $\mathcal{P} = (P_1)$, $v_1 = \lossval_1 -s_1 -\ln n$, $L_1 = \lossval_1$, $M_1=s_1$, $d=1$.
 		\For{$i=2, \ldots n$}
 		\State Set  $P_{d+1} = \{i\}$, $\mathcal{P} \leftarrow (P_1, \ldots, P_{d+1})$,  $v_{d+1} =  \lossval_i -s_i -\ln n$,  $L_{d+1} = \lossval_i$, $M_{d+1} = s_i$ and $d=d+1$,
 		\While{$d \geq 2$ \mbox{and} $ v_{d-1} \geq v_{d} $  }
 		\State Set $v_{d-1} \leftarrow \lse(L_{d-1}, L_d) - \lse(M_{d-1}, M_d) -\ln n $
 		\State Set $L_{d-1} \leftarrow \lse(L_{d-1}, L_d) $, $M_{d-1} \leftarrow \lse(M_{d-1}, M_d) $
 		\State Set $\mathcal{P} \leftarrow (P_1, \ldots, P_{d-1} \cup P_{d})$
 		\State Set $d\leftarrow d-1$
 		\EndWhile
 		\EndFor
 		\State {\bf Output:} $\lossvalaux \in \reals^n$ such that $\lossvalaux_i = v_s$ for $i \in P_s$, $s \in \{1, \ldots, d\}$.
 	\end{algorithmic}
 \end{algorithm}
 
\subsection{Implementation}\label{ssec:pav}
The implementation of the smoothing is based on considering the primal formulation of the smoothing given in~\eqref{eq:smoothing_primal} as an isotonic regression problem and by calling a Pool Adjacent Violators (PAV) algorithm to solve it. It has been described in detail by, e.g., ~\citet{best2000minimizing, lim2016efficient, blondel2020fast, henzi2022accelerating}. For completeness, we detail here the rationale behind the implementation. We then specify here the overall implementation of the gradient oracles of the smooth approximations for the chosen regularizers $\Omega_1$ and $\Omega_2$ defined in \eqref{eq:chosen_smoothing1} and \eqref{eq:chosen_smoothing} respectively.

\paragraph{Formulation as an isotonic regression problem.}
Consider the primal problem~\eqref{eq:smoothing_primal} defining the smoothing approximation with a decomposable function $\Omega$ such that $\Omega(\dualvar) = \sum_{i=1}^n \omega(\dualvar_i)$, that is
\[
\lstat_{\nu \Omega} (\lossval) = \min_{\lossvalaux \in \mathbb{R}^n} \left\{\lstat(\lossvalaux) +  \nu \Omega^\star((\lossval-\lossvalaux)/\nu) \right\} = \min_{\lossvalaux \in \mathbb{R}^n} \left\{\sum_{i=1}^n \left[\sigma_i \lossvalaux_{(i)} +  \nu \omega^\star((\lossval_i-\lossvalaux_i)/\nu) \right]\right\}
\]
As shown by~\citet[Lemma 4]{blondel2020fast}, for any scalars $\lossval_i, \lossval_j, \lossvalaux_i, \lossvalaux_j$ such that $\lossval_i\leq \lossval_j$ and $\lossvalaux_i \geq \lossvalaux_j$, we have, using the convexity of $\omega^*$ that
$
\omega^\star(\lossval_i -\lossvalaux_i) + \omega^\star(\lossval_j - \lossvalaux_j) \geq \omega^\star(\lossval_i - \lossvalaux_j) + \omega^\star(\lossval_j -\lossvalaux_i).
$
Hence for $\lossvalaux\in \R^n$ to minimize $\nu\Omega^\star((\lossval-\lossvalaux)/\nu)= \sum_{i=1}^n\nu \omega^\star((\lossval_i-\lossvalaux_i)/\nu)$, the coordinates of $\lossvalaux$ must be ordered in the same order as $\lossval$. Since $\lstat(\lossvalaux)=\sum_{i=1}^n \sigma_i \lossvalaux_{(i)}$ is independent of the ordering of the coordinates of $\lossvalaux$, we get that, given a permutation $\tau$ of $\{1, \ldots, n\}$ such that $\lossval_{\tau_1}\leq \ldots\leq \lossval_{\tau_n}$, problem~\eqref{eq:smoothing_primal} is equivalent to 
\[
\lstat_{\nu\omega}(\lossval)= \min_{\substack{\lossvalaux \in \reals^n \\ \lossvalaux_{\tau_1} \leq \ldots \leq \lossvalaux_{\tau_n}}}  \sum_{i=1}^n \left(\sigma_i \lossvalaux_{\tau_i}+\nu  \omega^\star((\lossval_i-\lossvalaux_i)/\nu)\right).
\]
An oracle on the gradient of the smooth approximation is then given by $\argmax_{\dualvar \in \mathcal{P}(\sigma)} \left\{\lossval^\top \dualvar - \nu \Omega(\dualvar)\right\}= \nabla \lstat_{\nu \Omega}(\lossval)$ with
\begin{align}\label{eq:smooth_oracle_overview}
     \nabla \lstat_{\nu \Omega}(\lossval) = \nabla \Omega^\star((\lossval-\lossvalaux^*)/\nu) \quad 
	\mbox{for} \ \lossvalaux^*  = \nu (\pav_\omega(\lossval_\tau/\nu))_{\tau^{-1}}, \quad\pav_\omega(\lossval) = \argmin_{\substack{y\in \reals^n \\ \lossvalaux_1\leq \ldots \leq \lossvalaux_n}} \sum_{i=1}^n \left(\lossvalaux_i \sigma_i  + \omega^\star(\lossval_i-\lossvalaux_i) \right),
\end{align}
where $\pav$ is the output of the Pool Adjacent Violators algorithm~\citep{henzi2022accelerating, lim2016efficient, best2000minimizing} applied to the given isotonic regression problem.  

\begin{algorithm}[t]\caption{Pool Adjacent	Violators (PAV) Algorithm  \label{algo:pav2} for $\omega_2$}
	\begin{algorithmic}[1]
		\State {\bf Inputs:} Number of coefficients $n$, coefficients $(\lossval_i)_{i=1}^n$ and $(\sigma_i)_{i=1}^n$ 
		\State {\bf} Initialize  $P_1 = \{1\}$, $\mathcal{P} = (P_1)$, $v_1 = \lossval_1+1/n -\sigma_1$, $C_1 = 1$, $d=1$.
		\For{$i=2, \ldots n$}
		\State Set  $P_{d+1} = \{i\}$, $\mathcal{P} \leftarrow (P_1, \ldots, P_{d+1})$,  $v_{d+1} =  \lossval_i+1/n -\sigma_i$,  $C_{d+1} = 1$ and $d=d+1$,
		\While{$d \geq 2$ \mbox{and} $ v_{d-1} \geq v_{d} $  }
		\State Set $v_{d-1} \leftarrow \frac{C_{d-1} v_{d-1} + C_d v_d}{C_{d-1} + C_d}$
		\State Set $C_{d-1} \leftarrow C_d + C_{d-1}$
		\State Set $\mathcal{P} \leftarrow (P_1, \ldots, P_{d-1} \cup P_{d})$
		\State Set $d\leftarrow d-1$
		\EndWhile
		\EndFor
		\State {\bf Output:} $\lossvalaux \in \reals^n$ such that $\lossvalaux_i = v_s$ for $i \in P_s$, $s \in \{1, \ldots, d\}$.
	\end{algorithmic}
\end{algorithm}

\paragraph{Pool Adjacent Violators algorithm.}
We briefly recall the rationale of the Pool Adjacent Violator algorithm whose implementation for the choices of $\omega_1$ and $\omega_2$ are given in \cref{algo:pav1} and \cref{algo:pav2} respectively, where we denote $\lse(s_{S}) = \ln \sum_{i\in S} \exp(s_i)$.

The Pool Adjacent Violators Algorithm is used to solve problems of the form
\begin{equation}\label{eq:isotonic}
\min_{\substack{z \in \reals^n\\ z_1 \leq \ldots\leq z_n}}\  \sum_{i=1}^n f_i(\lossvalaux_i)
\end{equation}
for some set of functions $\mathcal{F} = (f_i)_{i=1}^n$, which in our case~\eqref{eq:smooth_oracle_overview} are given by $f_i(\lossvalaux_i) =  \lossvalaux_i \sigma_i  + \omega^\star(\lossval_i-\lossvalaux_i)$. 

If at the solution $\lossvalaux^*$, the constraint $\lossvalaux_i^* \leq \lossvalaux_{i+1}^*$ is active, then by definition, $\lossvalaux_i^*= \lossvalaux_{i+1}^*$. More generally if the constraint $\lossvalaux_i^* \leq \lossvalaux_j^*$ is active for $i<j$, then all constraints of the form $\lossvalaux_k^* \leq \lossvalaux_{k+1}^*$  for $k \in \{i, \ldots, j-1\}$ are active, i.e., $\lossvalaux_k^*=\lossvalaux_i^*$ for all $k \in \{i, \ldots, j\}$. Overall the solution of~\eqref{eq:isotonic} is characterized by a set of $p\leq n$ coordinates $v^*_1, \ldots, v^*_p$ and a partition $\mathcal{P}^* = (P_1^*, \ldots, P_p^*)$ of $\{1, \ldots, n\}$ into contiguous blocks $P_s^*= \{b_{s-1}+1, \ldots, b_s\}$ for $0 = b_0 < b_1 <\ldots< b_p = n$ such that $\lossvalaux^*_i = v^*_{s}$ if $i \in \{b_{s-1}+1, \ldots, b_s\}$. For any feasible candidate solution $\lossvalaux$ we can define the corresponding partition $\mathcal{P}(\lossvalaux)$ of $\{1, \ldots, n\}$ into contiguous blocks of coordinates. Conversely, given a partition $\mathcal{P}=\{P_1, \ldots, P_p\}$ of $\{1, \ldots, n\}$ into contiguous blocks, we can define a vector $\lossvalaux(\mathcal{P})$ with constant blocks such that $\lossvalaux_i = \bar \lossvalaux_{P_s} = \bregmean(\mathcal{F}, P_s)$ for $i \in P_s$ where for a set of functions $\mathcal{F} = (f_i)_{i=1}^n$ and a subset $S \subset \{1, \ldots, n\}$,  we define the function $\bregmean$ that computes the average solution of the objective of the PAV algorithm on $S$, i.e.
\begin{equation}\label{eq:merging}
	\bregmean(\mathcal{F}, S) = \argmin_{\lossvalaux \in \R} \sum_{i\in S} f_i(\lossvalaux).
\end{equation}
The principle of the PAV algorithm is to compute the optimal contiguous partition of $\{1, \ldots, n\}$ corresponding to the solution of~\eqref{eq:isotonic} by adding one coordinate of the problem at a time and merging this coordinate with previously computed blocks if the constraints are not satisfied. We refer to, e.g., \citep{best2000minimizing, henzi2022accelerating} for a proof of the validity of this strategy.
Most importantly, the efficiency of the PAV algorithm relies on having access to a function, which, for $S, T \subseteq \{1, \ldots, n\}$, $S\cap T = \emptyset$, is able to compute $\bregmean(\mathcal{F}, S\cup T)$ given appropriate stored values ($L_s$, $M_s$ in \cref{algo:pav1} and $v_s, C_s$ in \cref{algo:pav2}). The algorithms presented in \cref{algo:pav1} and \cref{algo:pav2} are then based on the computation of $\bregmean(\mathcal{F}, S)$ for the functions $f_i$ considered. Namely, denoting $\mathcal{F}_{\omega, \lossval} = (f_{\omega, \lossval,  i})_{i=1}^n $ for $f_{\omega, \lossval, i}(\lossvalaux_i) = \lossvalaux_i \sigma_i  + \omega^\star(\lossval_i-\lossvalaux_i)$,  $s_i = \ln\sigma_i$ and $\lse(s_{S}) = \ln \sum_{i\in S} \exp(s_i)$, we have
\begin{align*}
\bregmean( \mathcal{F}_{\omega_1, \lossval}, S)  
=  \lse(\lossval_{S}) -\lse(s_{S})-\ln n,  \quad 
\bregmean(\mathcal{F}_{\omega_2, \lossval}, S) 
 = \frac{1}{|S|} \sum_{i\in S}(z_i + 1/n-\sigma_i),
\end{align*}
and merging two subsets of coordinates can be done in $O(1)$ time given appropriate stored values as we have
\begin{align*}
	\bregmean(\mathcal{F}_{\omega_1, \lossval}, S\cup T) 
	& = \lse(\lse (\lossval_S), \lse (\lossval_T)) -\lse(\lse(s_S),\lse( s_T)) -\ln n
	\\
	\bregmean(\mathcal{F}_{\omega_2, \lossval}, S\cup T)  
	& = \frac{|S| \bregmean(\mathcal{F}_{\omega_2, \lossval}, S) + |T|\bregmean(\mathcal{F}_{\omega_2, \lossval}, T)}{|S| +|T|}.
\end{align*}

\section{\osvrg Convergence Analysis}
\label{sec:osvrg:proof}
\subsection{Setup for the Convergence Analysis}
Consider the optimization problem 
\begin{align} \label{eq:smooth:main}
    \min_{w} 
    \left[ 
    \primobjsmooth{\reg,\nu}(w) := \lstat_\nu\big(\ell(w)) + \frac{\strongcvx}{2} \normsq{w}_2 \right] \,,
    \quad \text{where} \ 
    \lstat_\nu(z) = \max_{\dualvar \in \Pcal(\sigma)} \left\{ \dualvar\T z   - \frac{\nu}{2} \norm{\dualvar- u_n}^2_2  \right\}
\end{align}
is the $L_2$-smoothing as defined in \Cref{sec:a:smoothing}
where $\Omega(\dualvar) = \norm{\dualvar - u_n}^2/2$.
Here, 
$\sigma_1 \le \cdots \le \sigma_n$ are given nonnegative weights that sum to 1, $\Pcal(\sigma)$ is the permutahedron of $\sigma$, $\strongcvx$ is a regularization parameter on the $w$'s, $\nu$ is smoothing parameter and $u_n = \mathbf{1}_n / n$ denotes the uniform distribution over $n$ items. 

It is convenient to look at the saddle form
\begin{align} \label{eq:saddle:main}
    \saddleobjsmooth{\nu}(w, \dualvar) :=
    \dualvar\T \ell(w) + \frac{\strongcvx}{2} \norm{w}^2_2 - \frac{\nu}{2} \norm{\dualvar- u_n}^2_2 \,.
\end{align}

 Throughout, we make the following assumption: 

\begin{assumption} \label{asmp:saddle}
For each $i \in [n]$,  $w \mapsto \ell_i(w)$ is convex, $G$-Lipschitz, and $L$-smooth. 
\end{assumption}

We analyze \osvrg with smoothing, as given in \Cref{alg:osvrg-u:theory-1}. It only differs from \Cref{alg:osvrg-u:theory-main} presented in the main paper in line~\ref{line:osvrg-smooth:dual-update}. 

\begin{algorithm}[t]
\caption{
\osvrg with smoothing
}\label{alg:osvrg-u:theory-1}
\begin{algorithmic}[1]
\Require Number of iterations $T$, loss functions $(\ell_i)_{i=1}^n$ and their gradient oracles, initial point ${w}\pow{0}$, regularization parameter $\reg$, learning rate $\eta$, sorting update frequency $N$, probability of checkpointing $q^*$.
\For{iterate $t = 0, ..., T - 1$}
    \If {$t \mod N = 0$} \Comment{Update weights (generalization of updating the sorting)}
        \State Update $\dualvar\pow{t} = \argmax_{\dualvar \in \Pcal(\sigma)} \left\{ 
        \sum_{i=1}^n \dualvar_i \ell_i(w \pow t) - \frac{\nu}{2} \|\dualvar - u_n\|_2^2 \right\}$ computed using \cref{eq:smooth_oracle_overview} and \cref{algo:pav2}.
        \label{line:osvrg-smooth:dual-update}
    \Else 
        \State $\dualvar\pow{t} = \dualvar\pow{t-1}$
    \EndIf
    \State Sample $q_{t} \sim \text{Unif}([0, 1])$
    \If{$t \mod N = 0$ or  $q_{t} \le q^*$} \Comment{Update batch gradient}
        \State Set $\bar w\pow{t} = w\pow{t}$ and 
            $\bar g\pow{t} = \sum_{i=1}^n \dualvar_i\pow{t} \grad \ell_i(\bar w\pow{t}) $.
    \Else
        \State $\bar w\pow{t} = \bar w\pow{t-1}$ and 
        $\bar g\pow{t} = \bar g\pow{t-1}$.
    \EndIf
    \State Sample $i_{t} \sim \text{Unif}([n])$.
    \State $v^{(t)} = n\dualvar_{i_t}\pow{t} \nabla \ell_{i_{t}}(w^{(t)}) - n\dualvar_{i_t}\pow{t} \nabla \ell_{i_{t}}(\bar{w}\pow{t}) + \bar{g}\pow{t}$.
    \State $w^{(t+1)} = (1-\eta \mu) w^{(t)} - \eta v^{(t)}$.
\EndFor
\State \Return $w\pow{T}$
\end{algorithmic}
\end{algorithm}

\subsection{Convergence Analysis}
\Cref{alg:osvrg-u:theory-1} can be interpreted as 
an algorithm that alternates exactly maximizing over $\dualvar$ in $\saddleobjsmooth{\nu}(w, \cdot)$ with $w$ fixed and 
minimizing $\saddleobjsmooth{\nu}(\cdot, \dualvar)$ with $\dualvar$ fixed using a particular variant of SVRG known as q-SVRG~\cite{hofmann2015variance}; see \Cref{alg:qsvrg} for a review of q-SVRG. 

\begin{proposition}
    The iterates $(w_1\pow{t}, \dualvar_1\pow{t})$ 
    produced by \Cref{alg:osvrg-u:theory-1} 
    and $(w_2\pow{k}, \dualvar_2\pow{k})$ produced by \Cref{alg:osvrg-u:theory-2} with a given starting point $w\pow{0}$, learning rate $\eta$, weight update frequency (or inner loop length) $N$, and number of iterates $T= KN$ where $K$ is the number of epochs of \Cref{alg:osvrg-u:theory-2} satisfy 
    $w_2\pow{k} = w_1\pow{kN}$ and
     $\dualvar_2\pow{k} = \dualvar_1\pow{kN}$ for each epoch $k$.
\end{proposition}
\begin{proof}
    The two algorithms are equivalent iteration for iteration and the proof follows from pattern matching. 
\end{proof}

\paragraph{Convergence Analysis}
We have the following rate when the smoothing parameter $\nu > O(nG^2 / \strongcvx)$.

\begin{algorithm}[t]
\caption{
\osvrg with smoothing: Rewriting
}\label{alg:osvrg-u:theory-2}
\begin{algorithmic}[1]
\Require Number of epochs $K$, number of SVRG steps $N$, loss functions $(\ell_i)_{i=1}^n$ and their gradient oracles, initial point ${w}\pow{0}$, regularization parameter $\reg$, learning rate $\eta$, probability of checkpointing $q^*$, smoothing coefficient $\nu$.
\For{epoch $k = 0, ..., K - 1$}
    \State Compute $\dualvar\pow{k} = \argmax_{\dualvar \in \Pcal(\sigma)} \saddleobjsmooth{\nu} (w\pow{k}, \dualvar)$ using \cref{eq:smooth_oracle_overview} and \cref{algo:pav2}.
    \State Define $\tilde \ell_i\pow{k}(w) := n \dualvar_i\pow{k} \ell_i(w) + \reg\norm{w}^2_2/2$ for $i\in \{1, \ldots, n\}$.
    \State Compute $w\pow{k+1} = \text{q-SVRG}\Big(N, (\tilde\ell_i\pow{k})_{i=1}^n, w\pow{k}, \eta, q^* \Big)$ using \Cref{alg:qsvrg}.
\EndFor
\State \Return $w\pow{K}$
\end{algorithmic}
\end{algorithm}

\begin{theorem}
    Consider problem \eqref{eq:smooth:main} satisfying \Cref{asmp:saddle}. Suppose the smoothing parameter satisfies $\nu \ge 4 nG^2 / \strongcvx$.
    The sequence of iterates produced by \Cref{alg:osvrg-u:theory-2} with inputs 
    $N = (n(1+ 8\sigma_{\max} L/\mu) +8) \log (125/4)$, 
    $\eta = 2 / ( n (8\sigma_{\max}L+\strongcvx) +8 \strongcvx)$
    , $q^* =1/n$, satisfies 
    \[
        \expect \|w\pow{k}-w^*\|_2 \le \left(\frac{1}{2} \right)^k \|w\pow{0} - w^*\|_2 \,,
    \]
    where $w^* = \argmin_{w\in \reals^d} \primobjsmooth{\reg,\nu}(w)$. 
    
    Consequently, \Cref{alg:osvrg-u:theory-2} (and hence \Cref{alg:osvrg-u:theory-1}) can produce a point $\hat w$ satisfying $\big(\expect\norm{\hat w - w^*}_2)^2 \le \eps$ in
    \[
        T \le C(n(1+ 8\sigma_{\max} L/\mu) +8) \log \left( \|w\pow{0} - w^*\|^2/\eps \right) 
    \]
    gradient evaluations, where $C$ is an absolute constant. 
\end{theorem}
\begin{proof}
    For each epoch $k$, \Cref{alg:osvrg-u:theory-2} runs 
    q-SVRG on the function 
    \[
        \varphi\pow{k}(w) := \saddleobjsmooth{\nu}(w, \dualvar\pow{k}) = \frac{1}{n} \sum_{i=1}^n \tilde \ell_i\pow{k}(w) \quad \text{where} \quad
        \tilde \ell_i\pow{k}(w) = n \dualvar_i\pow{k} \ell_i(w) + \frac{\strongcvx}{2}\norm{w}^2 \,.
    \]
    The aim of this step is to approximate 
    $w\pow{k+1}_* = \argmin_{w\in \reals^d} \saddleobjsmooth{\nu}(w, \dualvar\pow{k})$ with $w\pow{k+1}$. We start by quantifying this error. 
    
    Since $\Pcal(\sigma)$ is the permutahedron on $\sigma$, we have that  
    \[
 \sigma_{\min} \leq  \min \left\{ \dualvar_i \, :\, \dualvar \in \Pcal(\sigma) \right\}  \leq \max \left\{ \dualvar_i \, :\, \dualvar \in \Pcal(\sigma) \right\} \le \sigma_{\max}.
    \]
    Hence, we have that each $\tilde \ell_i\pow{k} = n \dualvar_i\pow{k} \ell_i + \reg\norm{\cdot}^2_2/2$ is $n\sigma_{\max}L+\strongcvx$-smooth and $\strongcvx$-strongly convex, and its condition number is $\kappa = (n\sigma_{\max} L/\mu +1)$.
    Denote the sigma-algebra generated by $w\pow{k}$ as $\Fcal_k$, we have 
    from \Cref{thm:qsvrg} that
    \begin{align*}
        \expect\left[\norm{w\pow{k+1} - w\pow{k+1}_*}^2 \middle| \Fcal_k\right]
        \le \frac{5}{4} \exp\left( - \frac{N}{8 \kappa + n}\right) \norm{w\pow{k} - w\pow{k+1}_*}^2
        = \frac{1}{25} \norm{w\pow{k} - w\pow{k+1}_*}^2 \,.
    \end{align*}
    Therefore, Jensen's inequality gives us 
    \begin{align} \label{eq:osvrg-pf-1}
        \expect\left[\norm{w\pow{k+1} - w\pow{k+1}_*} \middle| \Fcal_k\right] \le \frac{1}{5} \norm{w\pow{k} - w\pow{k+1}_*} \,.
    \end{align}
    Denote $\dualvar^* = \argmax_{\dualvar \in \Pcal(\sigma)} \saddleobjsmooth{\nu}(w^*, \dualvar)$. Since $\saddleobjsmooth{\nu}(\cdot, \dualvar)$ is strongly convex and $\saddleobjsmooth{\nu}(w, \cdot)$ is strongly concave, 
    we have that strong duality holds, i.e., $\min_{w\in \reals^d} \max_{\dualvar \in \Pcal(\sigma)} \saddleobjsmooth{\nu}(w, \dualvar) = \max_{\dualvar \in \Pcal(\sigma)} \min_{w\in \reals^d} \saddleobjsmooth{\nu}(w, \dualvar)$~\citep[e.g.,][Thm. VII.4.3.1]{HiriartUrruty1993Convex}
    Therefore, $(w^*, \dualvar^*)$ is the unique saddle point of $\saddleobjsmooth{\nu}$, so $w^* = \argmin_{w\in \reals^d} \saddleobjsmooth{\nu}(w, \dualvar^*)$. 
    Together with \Cref{lem:lips-minimizer}, this gives us
    \begin{align} \label{eq:osvrg-pf-2}
        \norm{w\pow{k+1}_* - w^*} &\le \frac{\sqrt n G}{\strongcvx} \norm{\dualvar\pow{k} - \dualvar^*}, 
        \quad \text{and} \quad
        \norm{\dualvar\pow{k} - \dualvar^*} \le \frac{\sqrt n G}{\nu} \norm{w\pow{k} - w^*} \,.
    \end{align}
    From repeated invocations of the triangle inequality, we get, 
    \begin{align*}
        \expect\left[\norm{w\pow{k+1} - w^*} \middle| \Fcal_k\right]
        &\le 
        \expect\left[\norm{w\pow{k+1} - w\pow{k+1}_*} \middle| \Fcal_k\right]
        + \norm{w\pow{k+1}_* - w^*}  \\ 
        &\stackrel{\eqref{eq:osvrg-pf-1}}{\le}
            \frac{1}{5} \norm{w\pow{k} - w_*\pow{k+1}}_2 
            +
            \norm{w\pow{k+1}_* - w^*}_2 \\
        &\le
            \frac{1}{5} \norm{w\pow{k} - w^*}_2 
            +
            \frac{6}{5}
            \norm{w\pow{k+1}_* - w^*}_2 \\
        &\stackrel{\eqref{eq:osvrg-pf-2}}{\le}
            \frac{1}{5} \norm{w\pow{k} - w^*}_2 
            + \frac{ 6\sqrt n G}{5\strongcvx} \norm{\dualvar\pow{k} - \dualvar^*}_2  \\
        &\stackrel{\eqref{eq:osvrg-pf-2}}{\le}
            \left(\frac{1}{5} + \frac{6n G^2}{5\strongcvx\nu}\right) \norm{w\pow{k} - w^*}_2 \\
        &\le \frac{1}{2} \norm{w\pow{k} - w^*} \,,
    \end{align*}
    since we assumed $\nu$ satisfies $6n G^2 / (5\strongcvx \nu) \le 3/10$. 
    Taking an expecation w.r.t. $\Fcal_k$ and unrolling this completes the proof.
\end{proof}

\subsection{\osvrg Variants}
Algorithm~\ref{alg:extr_svrg} gives a variant of the \osvrg algorithm that computes checkpoints and the sorting at regular intervals. For simplicity, we visualize this algorithm as running in epochs. 
As in the usual SVRG algorithm for the ERM setting, we compute the full-batch subgradient at the checkpoint $\bar w_k$ at the start of each epoch (line~\ref{line:osvrg:batch-grad}). 
This is used to define the variance-reduced update in line~\ref{line:update_svrg}. 
Note also that we consider sampling at each iteration an example $i_t$ distributed as  $p_\sigma(i) = \P{}{i_t = i} = \sigma_i$; this is well-defined since $\sigma_1, ..., \sigma_n$ defines a probability measure over $\{1, \cdots, n\}$.

\begin{algorithm}[ht]
\caption{
Epoch-based \osvrg with nonuniform sampling
}\label{alg:extr_svrg}
\begin{algorithmic}[1]
\Require Number of iterates $T$ per epoch, number of epochs $K$,
regularization parameter $\reg$, learning rate $\eta$, non-decreasing probability mass function $\sigma = (\sigma_i)_{i=1}^n$, loss functions $(\ell_i)_{i=1}^n$ and their gradient oracles, initial point $\bar{w}_0$.
\For{epoch $k = 0, 1, 2, ..., K - 1$}
    \State Select $\pi_k \in \argsort\p{\ell\p{\bar{w}_k}}$.
    \State $\bar{g}_k = \sum_{i=1}^n \sigma_i \nabla \ell_{\pi_k(i)}(\bar{w}_k)$.
    \label{line:osvrg:batch-grad}
    \State $w^{(0)} = \bar{w}_k$.
    \For{iterate $t = 0, ..., T - 1$}
        \State Sample $i_{t} \sim p_\sigma$.
        \State $v^{(t)} = \nabla \ell_{\pi_k(i_{t})}(w^{(t)}) - \nabla \ell_{\pi_k(i_{t})}(\bar{w}_k) + \bar{g}_k$. \label{line:update_svrg}
        \State $w^{(t+1)} = (1 - \eta \mu) w^{(t)} - \eta v^{(t)}$.
    \EndFor
    \State Set $\bar w_{k+1} = w^{(T)}$.
\EndFor
\State \Return $\bar{w}_{K}$.
\end{algorithmic}
\end{algorithm}

\subsection{q-SVRG Review}
Consider the risk-neutral problem
\[
    f(w) = \frac{1}{n} \sum_{i=1}^n \ell_i(w) \,.
\]
The q-SVRG is a variant of SVRG that updates the batch gradient with probability $1/m$ at each step, rather than once every $m$ steps like the usual version of SVRG~\cite{hofmann2015variance}. 
See \Cref{alg:qsvrg} for details.
It has the following convergence guarantee. 

\begin{theorem}[\citealp{hofmann2015variance}, Lemma 3]
\label{thm:qsvrg}
Suppose each $\ell_i$ is $L$-smooth and $\reg$-strongly convex.
Then \Cref{alg:qsvrg} with a learning rate $\eta = 2 / (8L + n \reg)$ and $q^* = 1/n$ produces a sequence $(w\pow{t})$ that satisfies
\[
    \expect\norm{w\pow{t} - w^*}^2 \le 
    \frac{5}{4} \exp\left( - \frac{t}{8 \kappa + n} \right) \norm{w\pow{0} - w^*}^2 \,,
\]
where $w^* = \argmin_w f(w)$ and $\kappa = L/\reg$ is the condition number. 
\end{theorem}

\begin{algorithm}[t]
\caption{
q-SVRG
}\label{alg:qsvrg}
\begin{algorithmic}[1]
\Require Number of iterations $T$, loss functions $(\ell_i)_{i=1}^n$ and their gradient oracles, initial point ${w}\pow{0}$, learning rate $\eta$, probability of checkpointing $q^*$.
\State Set $\bar w\pow{-1} = w\pow{0}$ and $\bar g\pow{-1} = \frac{1}{n}\sum_{i=1}^n\grad \ell_i(w\pow{0})$
\For{iterate $t = 0, ..., T - 1$}
    \State Draw $q_t \sim \text{Unif}([0, 1])$
    \If{ $q_t \le q^*$} \Comment{Update the batch gradient}
        \State Set $\bar w\pow{t} = w\pow{t}$ and 
            $\bar g\pow{t} =\frac{1}{n}\sum_{i=1}^n\grad \ell_i(w\pow{t})$
    \Else
        \State $\bar w\pow{t} = \bar w\pow{t-1}$ and 
        $\bar g\pow{t} = \bar g\pow{t-1}$
    \EndIf
    \State Sample $i_{t} \sim \text{Unif}([n])$
    \State $v^{(t)} = \nabla \ell_{i_{t}}(w^{(t)}) - \nabla \ell_{i_{t}}(\bar{w}\pow{t}) + \bar{g}\pow{t}$
    \State $w^{(t+1)} = w^{(t)} - \eta v^{(t)}$
\EndFor
\State \Return $w\pow{T}$
\end{algorithmic}
\end{algorithm}

\subsection{Technical Results}

Note the following properties of the joint function $\saddleobjsmooth{\nu}$ defined in \eqref{eq:saddle:main}. 
\begin{property}
    The following smoothness properties hold: 
    \begin{enumerate}[label=(\alph*),nosep]
        \item For each $\dualvar \in \Pcal(\sigma)$, $\grad_w \saddleobjsmooth{\nu}(\cdot, \dualvar)$ is $(L+\strongcvx)$-Lipschitz
        \item For each $w \in \reals^d$, $\grad_\dualvar \saddleobjsmooth{\nu}(w, \cdot)$ is $\nu$-Lipschitz.
        \item For each $w \in \reals^d$, $\grad_w \saddleobjsmooth{\nu}(w, \cdot)$ 
        is $\sqrt{n}G$-Lipschitz. 
        \item For each $\dualvar \in \Pcal(\sigma)$, $\grad_\dualvar \saddleobjsmooth{\nu}(\cdot, \dualvar)$ is $\sqrt{n}G$-Lipschitz.
    \end{enumerate}
\end{property}
\begin{proof}
    The result follows from the expressions
    \begin{align*}
        \grad_w \saddleobjsmooth{\nu}(w, \dualvar) = \sum_{i=1}^n \dualvar_i \grad \ell_i(w) + \strongcvx w \quad \text{and} \quad
        \grad_\dualvar \saddleobjsmooth{\nu}(w, \dualvar) = \ell(w) - \nu(\dualvar - u_n) \,.
    \end{align*}
    \begin{enumerate}[label=(\alph*),nosep]
        \item For any $w, w' \in \R^d$,
        \begin{align*}
            \norm{\grad_w \saddleobjsmooth{\nu}(w, \dualvar) - \grad_w \saddleobjsmooth{\nu}(w', \dualvar)}_2 &\leq \sum_{i=1}^n \dualvar_i \norm{\grad \ell_i(w) - \grad \ell_i(w')}_2 + \mu \norm{w - w'}_2\\
            & \leq \sum_{i=1}^n \dualvar_i L \norm{w - w'}_2 + \mu \norm{w - w'}_2\\
            &\leq \p{L + \mu} \norm{w - w'}_2,
        \end{align*}
        as $\sum_{i=1}^n \dualvar_i = 1$ for $\dualvar \in \mc{P}(\sigma)$.
        \item For any $\dualvar, \dualvar' \in \mc{P}(\sigma)$,
        \begin{align*}
            \norm{\grad_\dualvar \saddleobjsmooth{\nu}(w, \dualvar) - \grad_\dualvar \saddleobjsmooth{\nu}(w, \dualvar')}_2 &= \norm{\nu \dualvar - \nu \dualvar'}_2 = \nu \norm{\dualvar - \dualvar'}_2.
        \end{align*}
        \item For any $\dualvar, \dualvar' \in \mc{P}(\sigma)$,
        \begin{align*}
            \norm{\grad_w \saddleobjsmooth{\nu}(w, \dualvar) - \grad_w \saddleobjsmooth{\nu}(w, \dualvar')}^2_2 &= \norm{\sum_{i=1}^n (\dualvar_i - \dualvar_i')\grad \ell_i(w)}_2^2 \\
            & \leq \sum_{i=1}^n \norm{\grad \ell_i(w)}_2^2 \sum_{i=1}^n \p{\dualvar_i - \dualvar_i'}^2 \\
            &\leq n G^2 \norm{\dualvar - \dualvar'}_2^2.
        \end{align*}
        \item For any $w, w' \in \R^d$,
        \begin{align*}
            \norm{\grad_\dualvar \saddleobjsmooth{\nu}(w, \dualvar) - \grad_\dualvar \saddleobjsmooth{\nu}(w, \dualvar')}^2_2 &= \norm{\ell(w) - \ell(w')}_2^2\\
            &= \sum_{i=1}^n \p{\ell_i(w) - \ell_i(w')}_2^2\\
            &\leq \sum_{i=1}^n G \norm{w - w'}_2^2\\
            &= n G \norm{w - w'}_2^2.
        \end{align*}
    \end{enumerate}
\end{proof}

\begin{lemma} \label{lem:lips-minimizer}
    Given closed, convex sets $X \subset \reals^d$, $Y \subset \reals^p$, 
    consider a continuously differentiable function 
    $f: X \times Y \to \reals$ such that $f(\cdot, y)$ is $\reg$-strongly convex for all $y \in Y$ and 
    $\grad_x f(x, \cdot)$ is $L_{x,y}$-Lipschitz for each $x \in X$.
    Then, the map $x^*(y) = \argmin_{x \in X} f(x, y)$ is well-defined and is $L_{x,y}/\reg$ Lipschitz.  
\end{lemma}
\begin{proof}
    The map $x^*(y)$ is well-defined because $f(\cdot, y)$ is strongly convex and $X$ is closed, convex. Consider two points $y_1, y_2 \in Y$ and let $x_i = x^*(y_i)$ be the corresponding $x$-values. 
    From the first order optimality conditions of $f(\cdot, y_1)$ and $f(\cdot, y_2)$ respectively, we have 
    \begin{align} \label{eq:foo-opt}
        \inp{\grad_x f(x_1, y_1)}{x_2 - x_1} \ge 0, 
        \quad\text{and}\quad 
        \inp{\grad_x f(x_2, y_2)}{x_1 - x_2} \ge 0 \,.
    \end{align}
    Using the co-coercivity property $(*)$ of the strong convexity of $f(\cdot, y)$, we have, 
    \begin{align*}
        \reg\norm{x_1 - x_2}^2 
        &\stackrel{(*)}{\le} \inp{\grad_x f(x_1, y_2) - \grad_x f(x_2, y_2)}{x_1 - x_2} \\
        &\stackrel{\eqref{eq:foo-opt}}{\le}
            \inp{\grad_x f(x_1, y_2)}{x_1 - x_2} \\
        &\stackrel{\eqref{eq:foo-opt}}{\le} 
            \inp{\grad_x f(x_1, y_2) - \grad_x f(x_1, y_1)}{x_1 - x_2} \\
        &\le \norm{\grad_x f(x_1, y_2) - \grad_x f(x_1, y_1)} \, \norm{x_1 - x_2} \\
        &\le L_{x,y} \norm{y_1 - y_2} \, \norm{x_1 - x_2} \,.
    \end{align*}
\end{proof}

\section{Experimental Details}
\label{sec:a:experiments}
\Cref{sec:a:task} describes the tasks, datasets, and preprocessing steps used in the experiments. \Cref{sec:a:objective} reviews the objective minimized (including regularization). \Cref{sec:a:methods} describes the baseline methods compared. \Cref{sec:a:hyperparam} lists the hyperparameters of each algorithm and describes how they are selected. \Cref{sec:a:code} describes the compute environment used to run the experiments.

\subsection{Task and Dataset Descriptions}
\label{sec:a:task}

We start by describing the tasks and datasets considered in the experiments as well as their preprocessing steps.
For each task, we consider an input $x \in \msc{X}$, a feature map $\phi: \msc{X} \to \reals^d$, and an output space $\msc{Y}$. 
For regression, we have $\msc{Y} = \reals$
and for classification, we have $\msc{Y} = \{1, \ldots, C\}$, where $K$ is the number of classes. 
We make predictions with a linear model $x \mapsto w\T \phi(x)$, where $w \in \reals^d$ is the parameter vector to be optimized over. We consider the square loss between these predictions and the target $y_i$: 
\[
    \ell_i(w) = \frac{1}{2} (y_i - w\T \phi(x_i))^2 \,.
\]
for regression, and the multinomial logistic loss
\[
    \ell_i(w) = \log p_{y_i}(x_i; w), \text{ where } p_{y_i}(x_i; w) := \frac{\exp\p{w_{\cdot y}^\top x_i}}{\sum_{y' = 1}^C \exp\p{w_{\cdot y'}^\top x_i}}, \ w \in \R^{d \times C}
\]
for classification. Each input feature $\phi_j(x)$ for $j=1,\cdots, d$  is standardized to zero mean and unit variance (as are the targets $y_i$ in regression).
We now describe the datasets considered. The size and dimensionality of the resulting datasets are summarized in \Cref{tab:dataset}. 
\begin{enumerate}[nosep, label=(\alph*), leftmargin=\widthof{ (a) }]
    \item \texttt{simulated}:
This regression task entails prediction of a synthetic, real-valued response based on $d$-dimensional real vectors. The dataset is generated by sampling the inputs $x_1, ..., x_n$ and true parameter vector $w^*$ from the $d$-dimensional standard normal distribution $\mc{N}(0, I_d)$ for $n = 1000$ and $d = 10$, and the noise $\epsilon_1, ..., \epsilon_n \in \mc{N}(0, 1)$. Then, $y_i = w^\top x_i + \epsilon_i$ for $i = 1, ..., n$. The feature map $\phi$ is taken to be the identity.
    \item \texttt{yacht}: 
This regression task entails prediction of the residuary resistance of a sailing yacht based on its physical attributes \cite{Tsanas2012AccurateQE}. Each input $x \in \msc{X}$ is a sailing yacht and the feature map $\phi(x) \in \reals^d$ lists $d=6$ geometric attributes such as the length-beam ratio.
    \item \texttt{energy}:
This regression task entails prediction of the cooling load of a building based on its physical attributes \cite{Segota2020Artificial}. Each input $x \in \msc{X}$ is a building and the feature map $\phi(x) \in \reals^d$ lists $d=8$ structural attributes such as the surface area, height, etc.
    \item \texttt{concrete}: 
This regression task entails prediction of the compressive strength of a concrete type based on its physical and chemical attributes \cite{Yeh2006Analysis}. Each input $x \in \msc{X}$ is a particular composition of concrete and the feature map $\phi(x) \in \reals^d$ lists $d=8$ physical/chemical attributes such as amount of cement vs water.
    \item \texttt{iWildCam}: 
This classification task entails prediction of an animal present in an image captured by various wilderness camera traps, with drastic variation in illumination, camera angle, background, vegetation, color, and relative animal frequencies \cite{beery2020iwildcam}. Each input $x \in \msc{X}$ is an image the feature map $\phi(x) \in \reals^d$ for $d=189$ is the output of the sequence of the following operations.
\begin{itemize}
    \item A ResNet50 neural network \cite{He2016DeepResidual} that is pretrained on ImageNet \cite{deng2009imagenet} is applied to the image $x_i$, resulting in vector $x'_i$.
    \item The $x'_1, \ldots, x'_n$ are standardized to have zero mean and unit variance in each dimension.
    \item Principle Components Analysis (PCA) is applied, resulting in $d = 189$ components that explain $99\%$ of the variance, resulting in vectors $x''_i \in \R^{189}$.
    \item The $x''_1, \ldots, x''_n$ are standardized once again, giving $\phi(x_1), \ldots, \phi(x_n)$.
\end{itemize}

\end{enumerate}

\begin{table}[t!]
    \centering
    \begin{tabular}{ccccc}
        Dataset & $d$ & $n_{\text{train}}$ & $n_{\text{test}}$ & Source\\
        \hline
        \texttt{simulated} & 10 & 800 & 200 & n/a\\
        \texttt{yacht} & 6 & 244 & 62 & UCI\\
        \texttt{energy} & 8 & 614 & 154 & UCI\\
        \texttt{concrete} & 8 & 824 & 206 & UCI\\
    \end{tabular}
    \caption{Benchmark dataset descriptions.}
    \label{tab:dataset}
\end{table}

\subsection{Objective}\label{sec:a:objective}

In the experiments, we consider minimizing regularized ordered risk minimization problems of the form
\begin{align*}
    \min_{w\in \R^d} \quad & \msc{R}_\sigma(w) + \frac{\mu}{2} \|w\|_2^2, \\
    \mbox{where} \quad & \msc{R}_\sigma(w) = \sum_{i=1}^n \sigma_i \ell_{(i)}(w),
\end{align*}
where the coefficients $\sigma$ are defined using the spectrum of the spectral risk measure in question. We consider the mean, superquantile, extremile, and exponential spectral risk measure (ESRM), as defined in \Cref{sec:theory}. The regularization parameter $\mu$ is chosen as $1/n$ in the experiments presented in the main text, whereas other choices of $\mu$ are shown in \Cref{sec:a:additional}. By adding a regularization $\| \cdot \|_2^2$ to the objective, the \osvrg algorithm is modified by considering a direction of the form $v^{(t)}_{\mathrm{reg}} = v^{(t)}_{\mathrm{non\_reg}} + \mu w^{(t)}$, where $v^{(t)}_{\mathrm{non\_reg}} = \bar{g}\pow{t}$ is the direction presented in line 10 of Algorithm~\ref{alg:extr_svrg}. All algorithms are initialized with $w^{(0)} = 0$.

\subsection{Baseline Methods}
\label{sec:a:methods}

The baseline methods described below rely on a {\it stochastic subgradient estimate}, or a random quantity $g^{(t)}$ that estimates $\nabla \msc{R}_{\sigma}(w^{(t)})$ if $\msc{R}_{\sigma}$ is differentiable at $w^{(t)}$ and a subgradient of $\partial \msc{R}_{\sigma}(w^{(t)})$ otherwise. As described in \Cref{sec:optim}, we use
\begin{align}
    g\pow{t} := \sum_{j = 1}^m \hat{\sigma}_{j} \nabla \ell_{i_{(j)}}(w^{(t)})
    \label{eqn:stochastic_gradient}
\end{align}
for a minibatch $\{i_1, ..., i_m\}$ of size $m$ with weights
$\hat{\sigma}_j = \int_{\frac{j-1}{m}}^{\frac{j}{m}} s(t) \d t$, and $\ell_{i_{(j)}}$ being the ordered losses $\ell_{i_{(1)}} \leq \ldots \ell_{i_{(m)}}$ in the minibatch. We refer to the direction as $g\pow{t} = v_m\pow{t}$ in \Cref{alg:sgd}. 

\myparagraph{SGD} We refer to the stochastic subgradient method as SGD. The update can be written as
\begin{align*}
    w^{(t+1)} &:= w^{(t)} - \eta (g\pow{t} + \mu w\pow{t}),
\end{align*}
where $v\pow{t}_m$ is a stochastic estimate of the minibatch extremile subgradient (\Cref{eqn:stochastic_gradient}).

\myparagraph{SRDA} The stochastic regularized dual averaging (SRDA)~\cite{Xiao2009Dual} update can be written as
\begin{align*}
    w^{(t+1)} &:= \argmin_{w \in \R^d} w^\top\bar{g}^{(t)}+ \frac{\mu}{2} \norm{w}_2^2 + \frac{1}{2\eta t} \|w\|_2^2,
\end{align*}
where $\bar{g}^{(t)} = \sum_{i=0}^t g\pow{i}$ is the average of all stochastic subgradients (again computed by \Cref{eqn:stochastic_gradient}). 
Note that for $\Omega = \|\cdot\|_2^2/2$ and $w^{(0)} = 0$,
Note that for $w^{(0)} = 0$,
\begin{align*}
    w^{(t+1)} &= 0 -\frac{1}{\mu + {1}/{t\eta}} \bar{g}^{(t)}\\
    &= w^{(0)} - \sum_{s=0}^t  \frac{1}{\mu t + 1/\eta} g^{(s)}.
\end{align*}
Thus, the SRDA solution at time $t + 1$ can be seen as applying SGD with a {\it constant} learning rate of $\eta = 1/({\mu t}/{n} + \beta)$ (as $t$ refers to the value of only the last iteration). It is also seen that when $\mu = 0$ (no statistical regularization), SRDA reduces exactly to SGD.

\subsection{Hyperparameter Selection}
\label{sec:a:hyperparam}

The fixed optimization hyperparameters include the minibatch size $m = 64$ (SGD, SRDA) and the epoch length $N = n$ (\osvrg). The statistical regularization parameter $\mu = 1 / n$ is shown in the main text, whereas training curves for $\mu = 0.1 / n$ and $\mu = 10 / n$ are shown in \Cref{sec:a:additional}. Specifically, $c \in \{1, 2, 3, 4, 5\}$ be a seed that determines the randomness for sampling the minibatch $\{i_1, ..., i_m\}$ at each iteration of SGD and SRDA, $i_t$ at each iteration of \osvrg. Let $T$ be the total number of iterations for the algorithm, and denote the trajectory of iterates seeded by $c$ using learning rate $\eta$ as $w\pow{1}_{c, \eta}, ..., w\pow{T}_{c, \eta}$. Then, define the quantity
$   L(\eta) = \frac{1}{5} \sum_{c=1}^5 \msc{R}_\sigma(w\pow{T}_{c, \eta}).$
The learning rate $\eta$ is chosen in the set $\{3\times 10^{-4}, 1\times 10^{-3}, 3\times 10^{-3}, 1\times 10^{-2}, 3\times 10^{-2}, 1\times 10^{-1}, 3\times 10^{-1}, 1\times 10^{0}, 3\times 10^{0}\}$ to minimize $L(\eta)$ for each algorithm. If any of the trajectories diverge, we consider $L(\eta) = +\infty$. Note that $\msc{R}_\sigma$ is computed using the {\it training set}, as we are selecting hyperparameters for optimization.

\subsection{Compute Environment}
\label{sec:a:code}

All experiments were run on a workstation with Intel i9 processor (clock
speed: 2.80GHz) with 32 virtual cores and 126G of memory. We did not use GPUs for any experiments. Code used for this project was written in Python 3. 

\subsection{Experimental Details on Clustering}

Recall that we consider clustering $n$ points $x_1, \ldots, x_n$ into $k$ clusters with centers $C=(c_1, \ldots, c_k)$ by minimizing a weighted average of the distances of each point to its closest center, i.e., problems of the form
\[
\min_{C \in \reals^{d\times k}} \sum_{i=1}^n \sigma_i \ell_{(i)}(C) \quad \mbox{for}\ \ell_i(C) = \min_{\substack{z_i \in \{0, 1\}^k\\ z_i^\top \ones = \ones}}\sum_{j=1}^k z_{ij}\|x_i-c_j\|_2^2.
\]
We consider the weights $\sigma_i$ to be the discretization of a spectrum $s$ such that $\sigma_i= \int_{\frac{i-1}{n}}^{\frac{i}{n}} s(t)dt$ with $s$ being one of the following examples:
\begin{enumerate}[nosep]
    \item uniform spectrum, $s(t) = \ones_{[0, 1]}(t)$ which corresponds to a classical kmeans objective of the form $\min_{C\in \reals^{d\times k}} \frac{1}{n} \sum_{i=1}^n \ell_i(C)$,
    \item a truncated spectrum, $s_q(t) = \ones_{[0, q]}(t)/q$ for $q\in (0, 1)$ that seeks to only consider minimizing losses with small enough values compared to the whole distribution,
    \item an extremile spectrum $s_r(t) = r(1-t)^r$ for $r\geq 1$ that can be interpreted as minimizing the expected minimum of $r$ random variables distributed as the losses~\citep{Daouia2019Extremiles}.
\end{enumerate}

We consider a stochastic subradient descent with constant stepsize with mini-batch estimates given by the empirical L-statitics estimate on the mini-batches as described in Sec.~\ref{sec:optim}.

\subsubsection{Synthetic Data}
As~\cite{Maurer2020RobustUnsupervised} we consider as training data three cloud of Gaussians composed of $100$ two dimensional points each with  variance $0.1$ along both axis and centers $(-3,0), (0, 1)$ and $(3, 0)$ respectively. We add $100$ outliers sampled from a Gaussian with  variance $5$ along both axis and center $(-1, -5)$. 
The test set consists in points sampled from the three aforementioned inlier Gaussians, $100$ points per Gaussian. 
To test our method, we compute the number of correct assignments of the test points in their associated cluster after relabeling the clusters to match the true labeling. Namely, the groups found by a method may be correct but instead of labeling the first cloud of points by 1 the method may have assigned the label 1 to the second group and 2 to the first group for example, so we first find the permutation of the labels that leads to the highest accuracy.

We used mini-batches of size $64$, a learning rate of $1$ found by grid-search on log-10 scale, a uniform spectrum, a truncated spectrum with parameter $q=0.75$ or an extremile spectrum with $r=5$ and we initialize the centers at $0$. In Fig.~\ref{fig:full_synth_kmeans} we present the estimated centers found for each spectrum as well as the training and test losses and the training and test accuracies, where for the training accuracy we only consider the assignment of the inlier points. 

\begin{figure}
    \centering
    \includegraphics[width=0.32\linewidth]{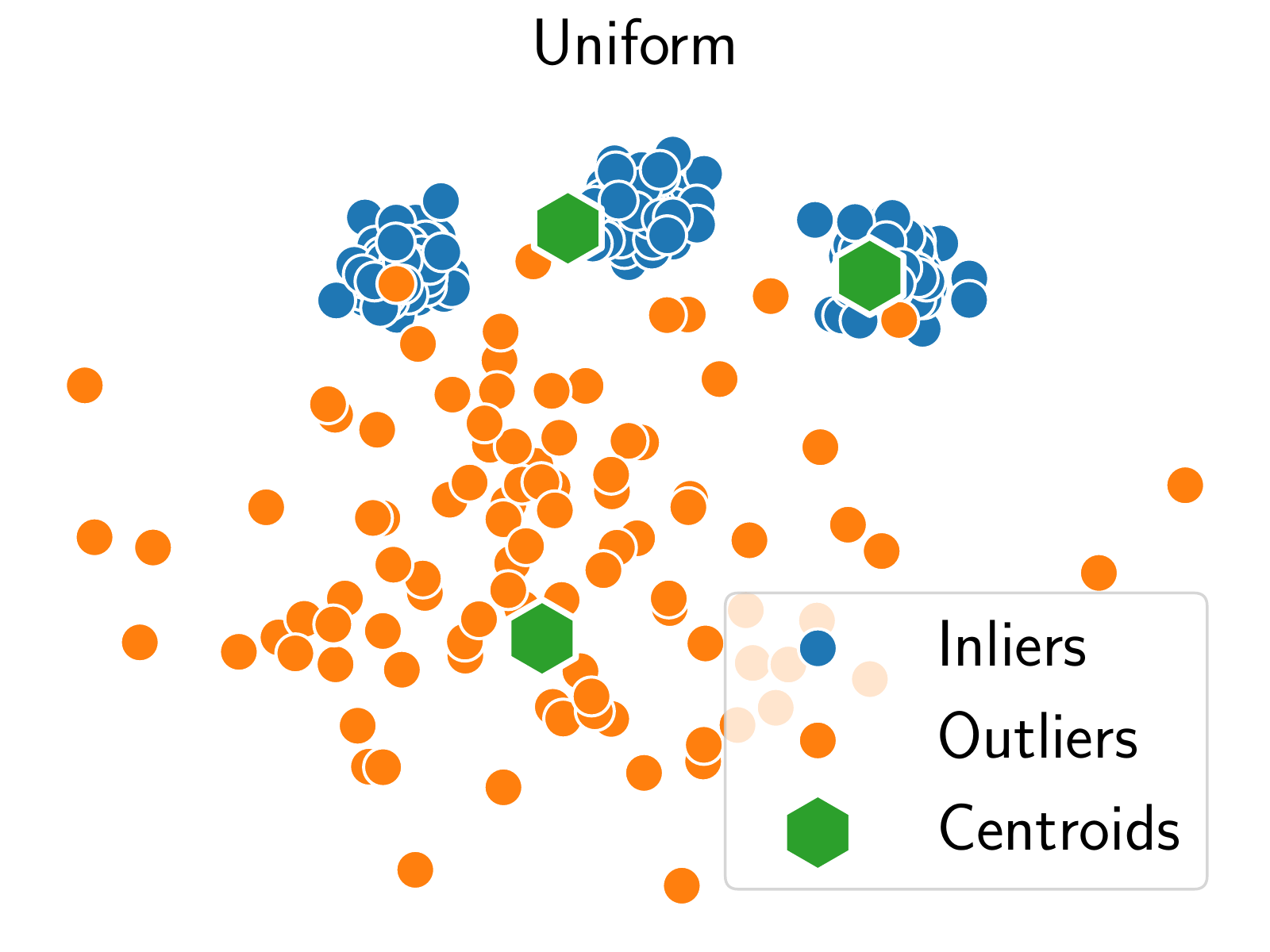}~
    \includegraphics[width=0.32\linewidth]{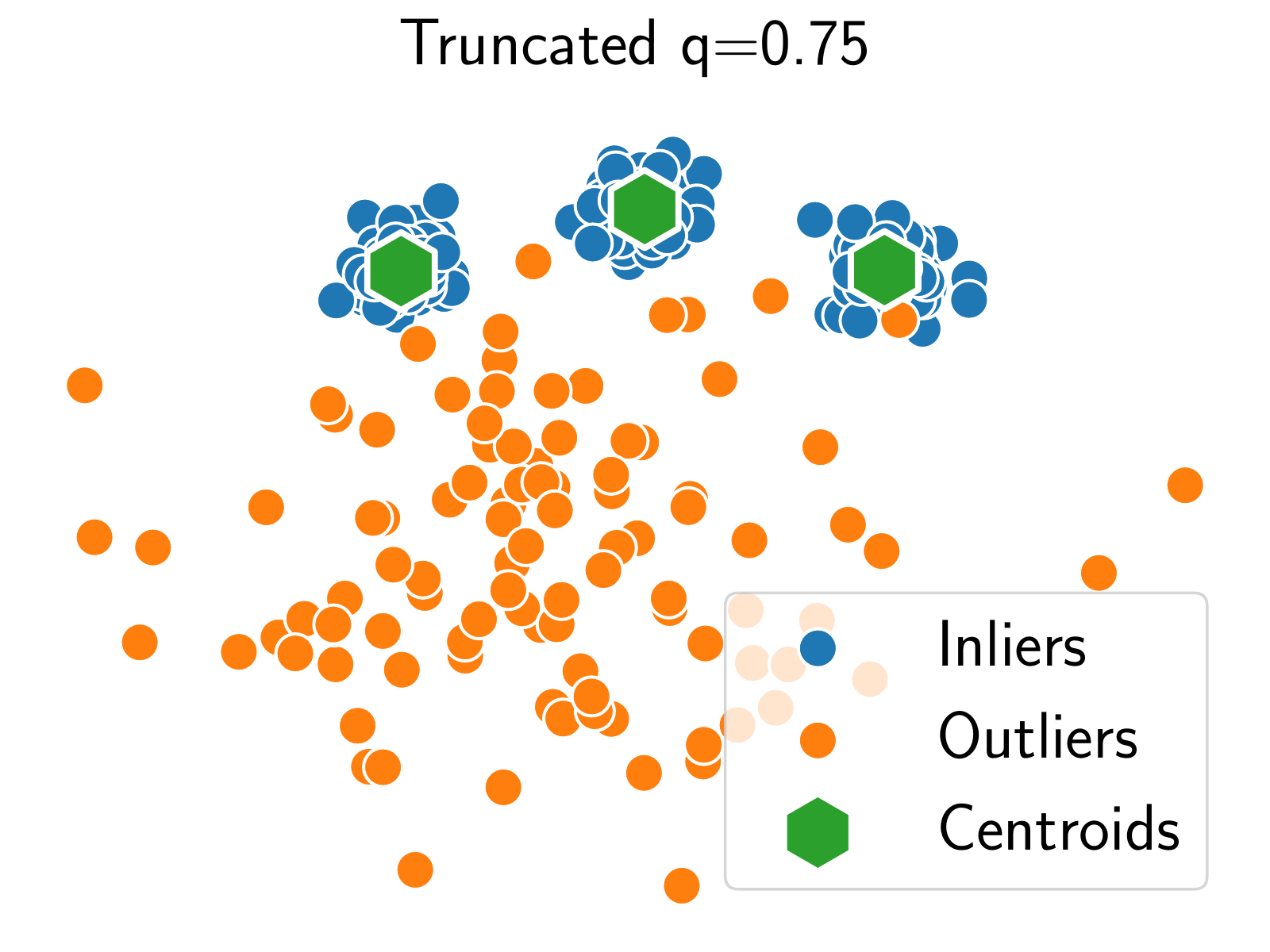}~
    \includegraphics[width=0.32\linewidth]{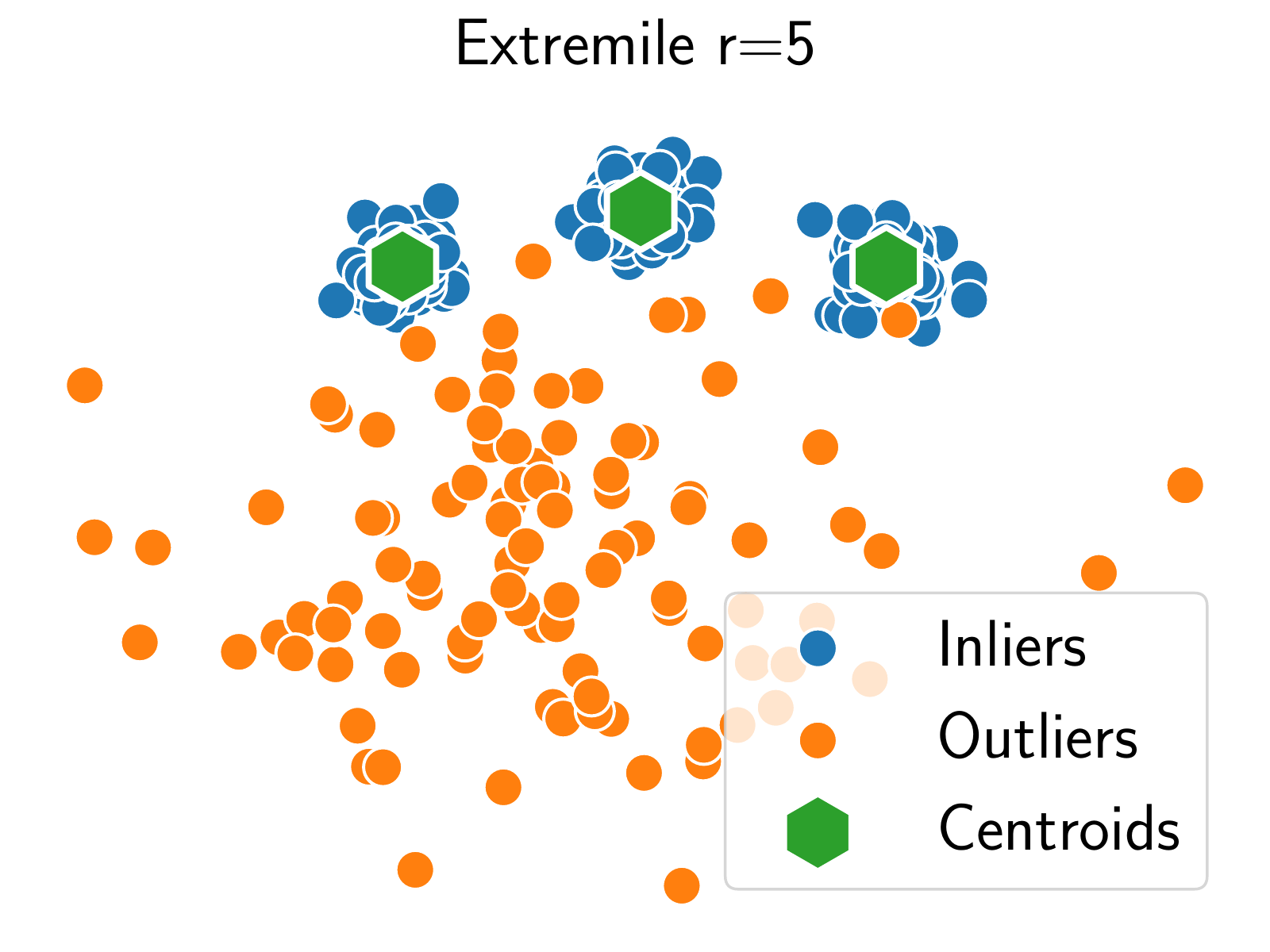}
    \includegraphics[width=0.8\linewidth]{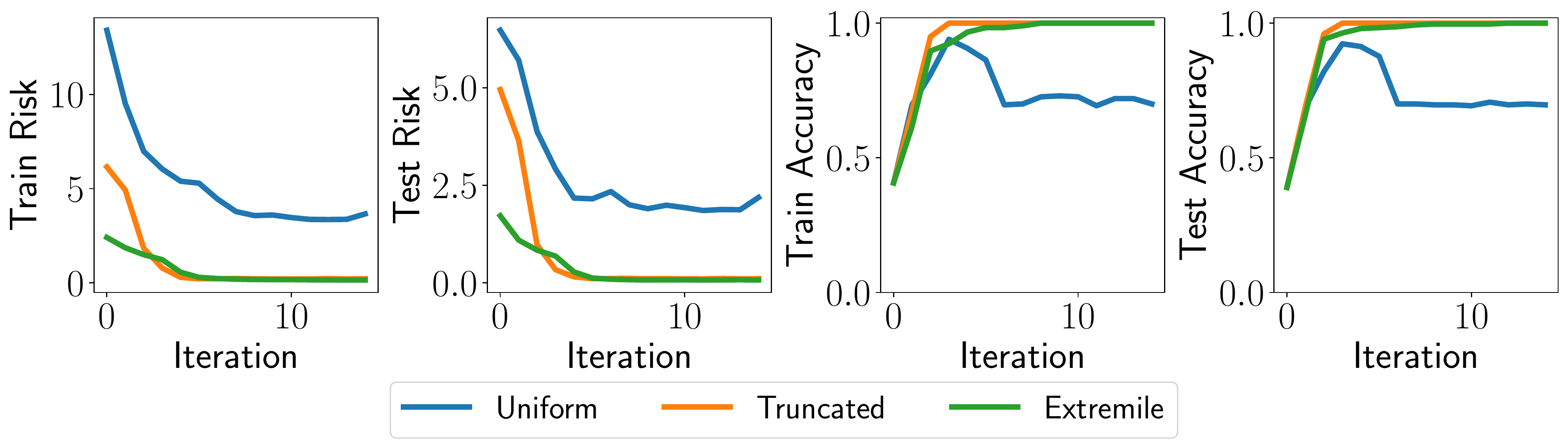}
    \caption{Clustering synthetic data points in the presence of outliers.}
    \label{fig:full_synth_kmeans}
\end{figure}

\subsubsection{Clustering Digits Images}
We consider forming a subset of the MNIST dataset~\citep{lecun1998mnist} of $28\times 28$ black and white images of handwritten digits by selecting $1000$ images of the digit $1$, $1000$ images of the digit $3$ each and $125$ images of each other digit in $\{0, \ldots, 9\}\setminus \{1, 3\}$ for a total of $2000$ inlier examples and $1000$ outlier examples. The images are standardized pixel by pixel. 
Our goal is to cluster the samples from $1$ and $3$ correctly even in the presence of outliers. We test our estimated centers on all images of the digits $1$ and $3$ from the test set of the MNIST database, that is, as in the synthetic experiment we test whether our estimated centers lead to the correct assignments of the test images in their respective group. 

We consider mini-batches of size $256$, a learning rate of $0.1$ found by grid-search on a log-10 scale, a uniform spectrum, a truncated spectrum with parameter $q=0.66$ or an extremile spectrum with $r=2$ and we initialize the centers at $0$.  
In Fig.~\ref{fig:full_mnist_kmeans} we present the estimated centers found for each spectrum as well as the training and test losses and the training and test accuracies, where for the training accuracy we only consider the assignment of the inlier points. 
 
 \begin{figure}
    \centering
    \includegraphics[width=0.32\linewidth]{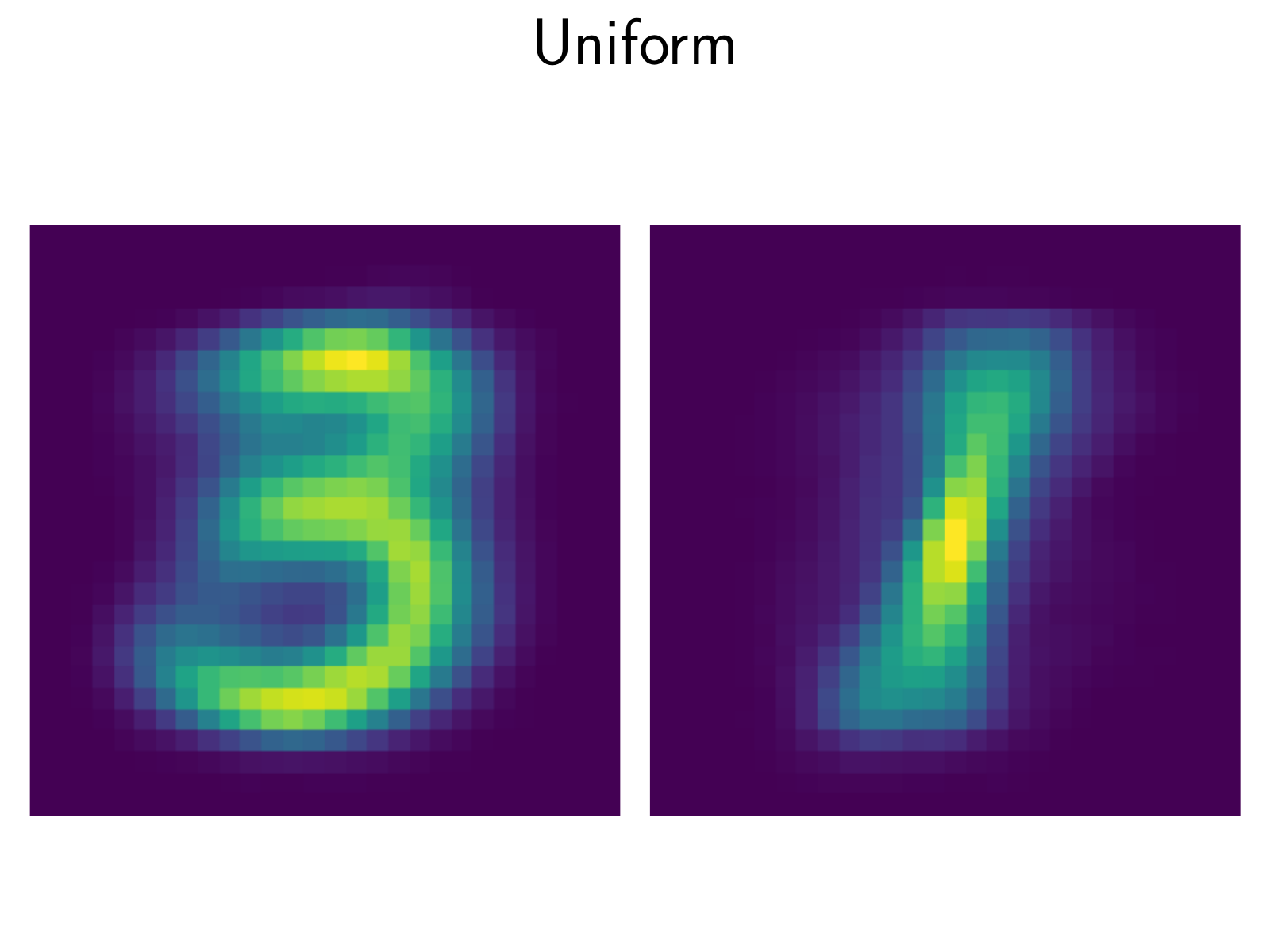}~
    \includegraphics[width=0.32\linewidth]{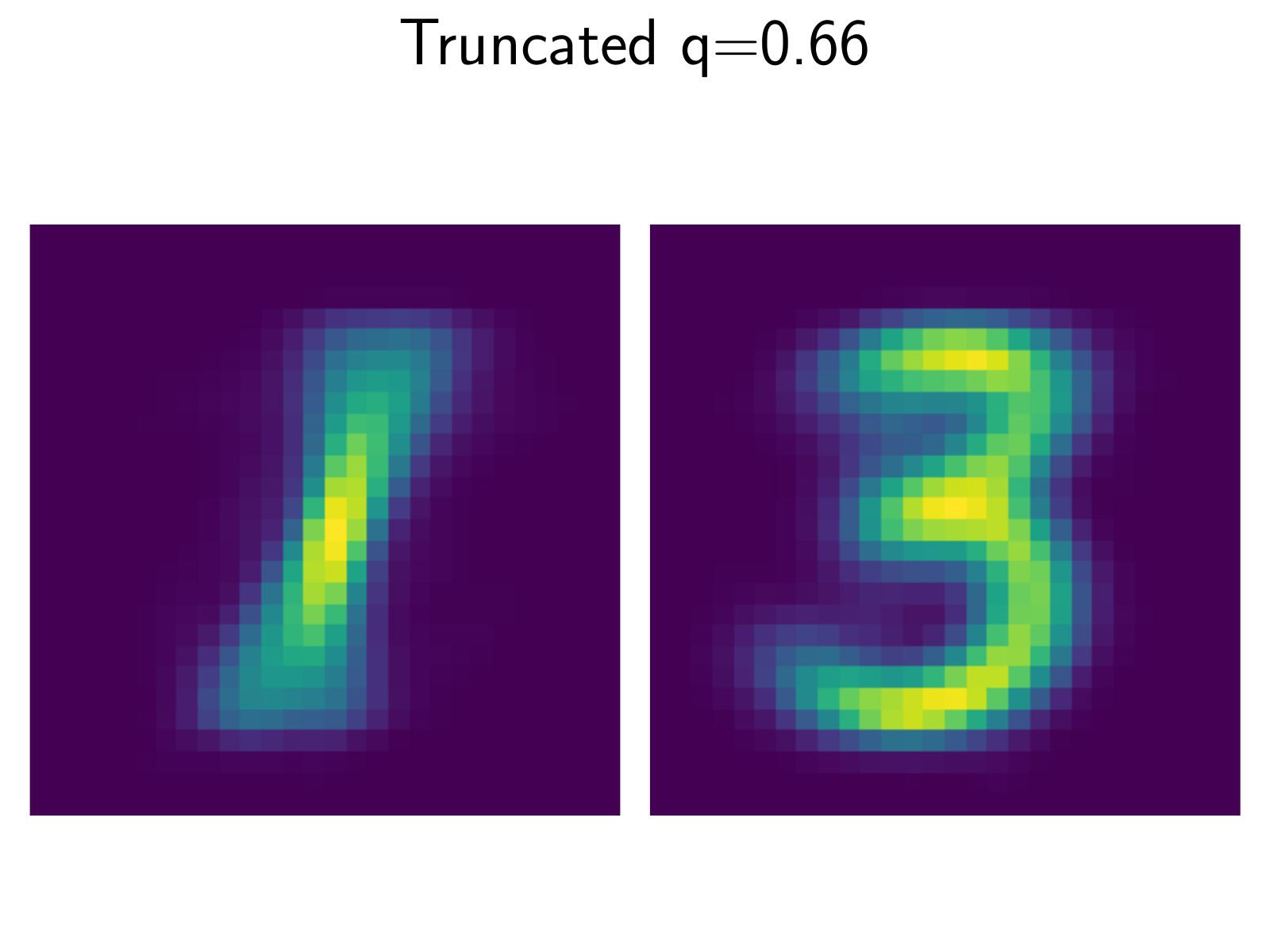}~
    \includegraphics[width=0.32\linewidth]{figs/MNIST_centers_extremile.pdf}
    \includegraphics[width=0.8\linewidth]{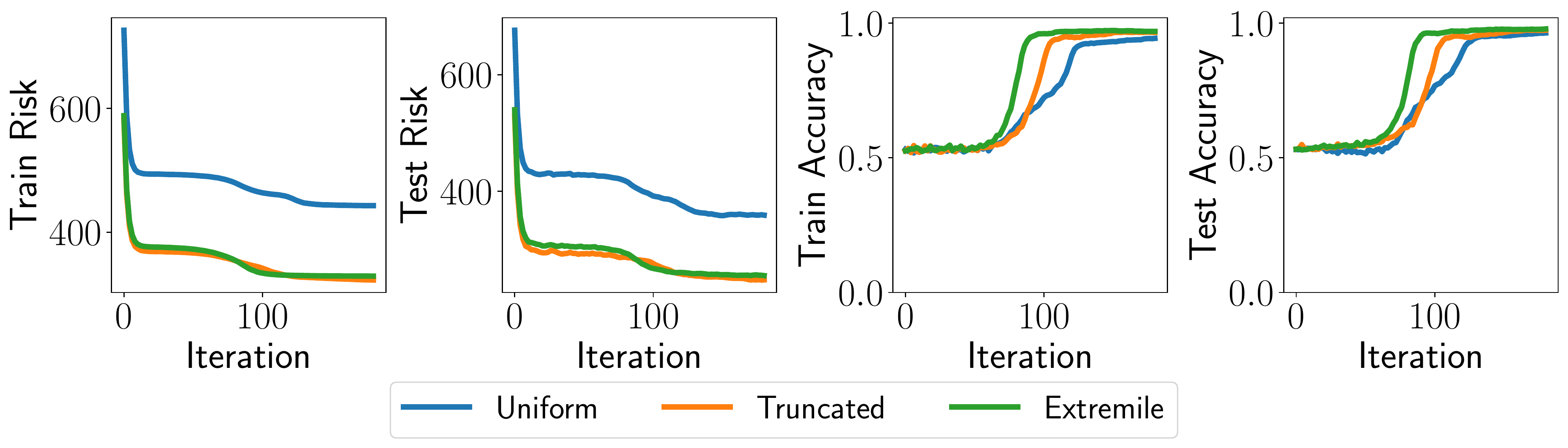}
    \caption{Clustering images of digits in the presence of outliers.}
    \label{fig:full_mnist_kmeans}
\end{figure}

\subsubsection{Clustering Images of Clothes}
As~\citet{Maurer2020RobustUnsupervised} we also consider clustering images of clothes from the dataset FashionMNIST~\citep{xiao2017fashion} that consist in $28\times28$ black and white images of $10$ clases of clothing such as: t-shirt, trouser, pullover, dress, coat, sandal, shirt, sneaker, bag, ankle boot. The images are standardized pixel by pixel. We form a training set composed of $1000$ images of trousers, $1000$ images of sneakers, and $250$ images of each of the other classes for a total of $2000$ inliers and $2000$ outliers. Our goal is to cluster teh trousers and the sneakers in the presence of the outliers. To test our estimators we use all images of trousers and sneakers from the test set of the FashionMNIST dataset. 

We consider mini-batches of size $64$, a learning rate of $1.$ found by grid-search on a log-10 scale, a uniform spectrum, a truncated spectrum with parameter $q=0.5$ or an extremile spectrum with $r=5$ and we initialize the centers at $0$.  
In Fig.~\ref{fig:full_fashion_mnist_kmeans} we present the estimated centers found for each spectrum as well as the training and test losses and the training and test accuracies, where for the training accuracy we only consider the assignment of the inlier points. 

Note that compared to~\citet{Maurer2020RobustUnsupervised} we obtain $100$\% accuracy of these methods on the test set. An approach by stochastic subgradient may be less sensitive to the initialization (performed with K-means++ by~\citet{Maurer2020RobustUnsupervised}). 

\begin{figure}
    \centering
    \includegraphics[width=0.32\linewidth]{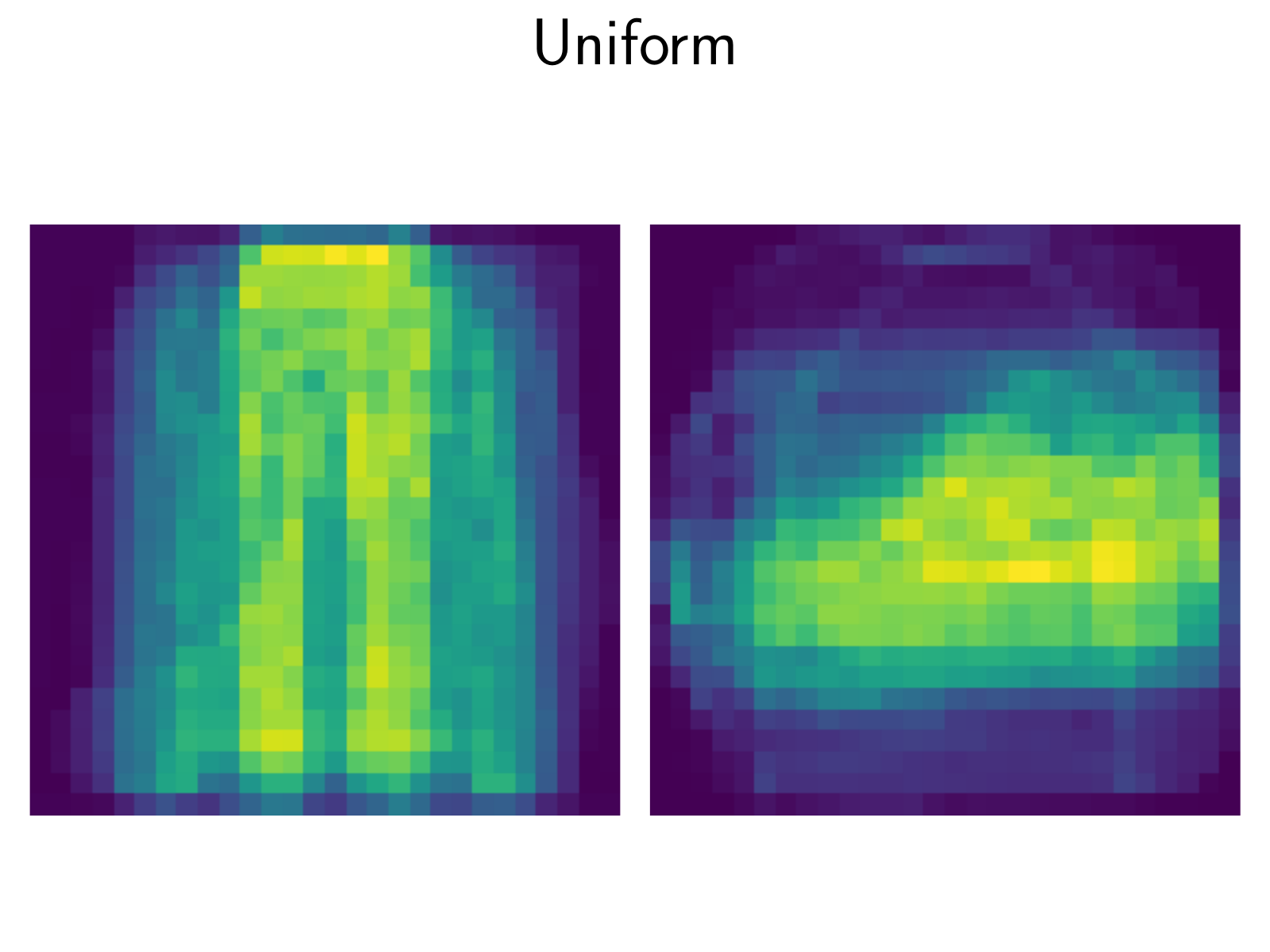}~
    \includegraphics[width=0.32\linewidth]{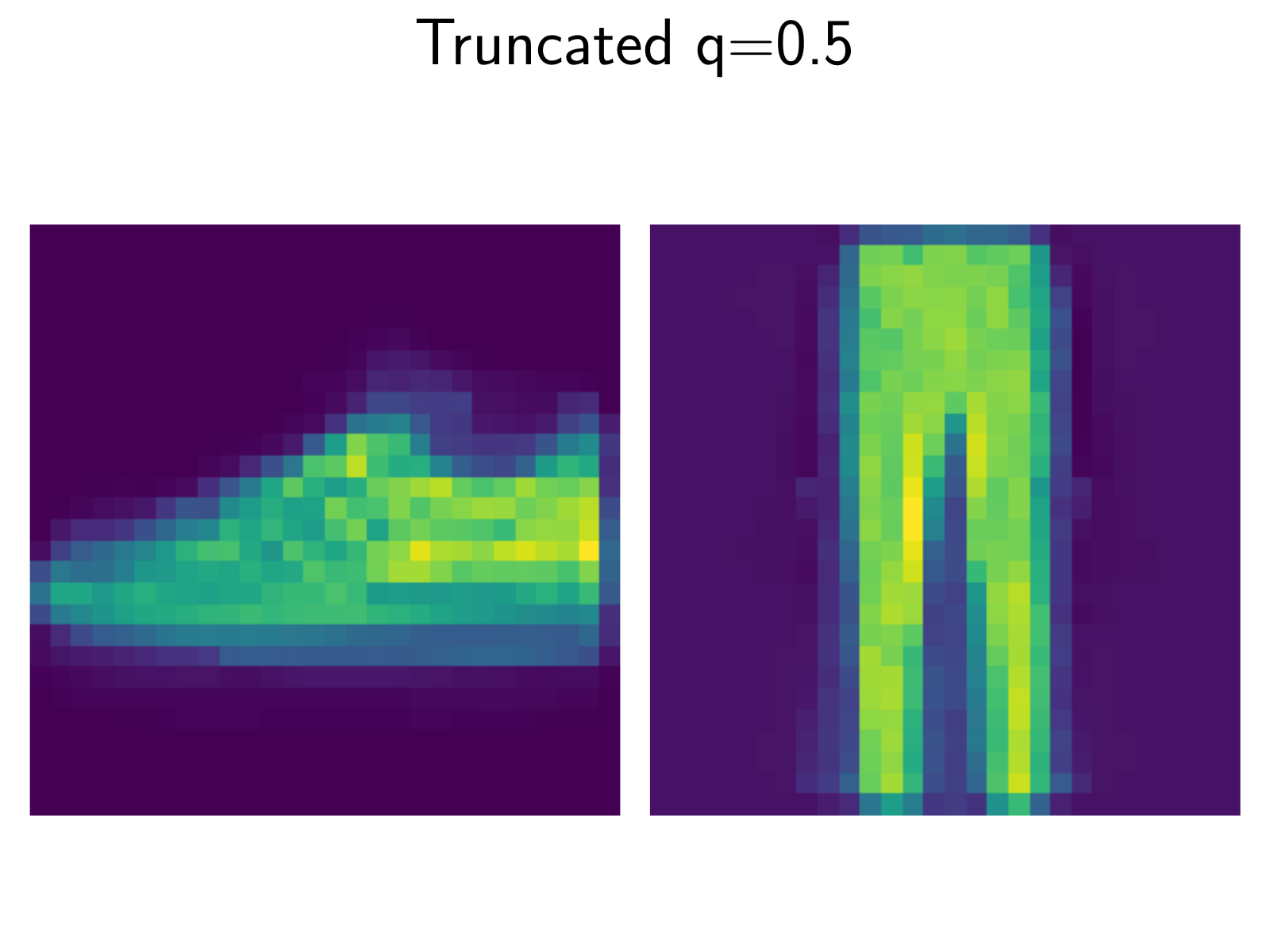}~
    \includegraphics[width=0.32\linewidth]{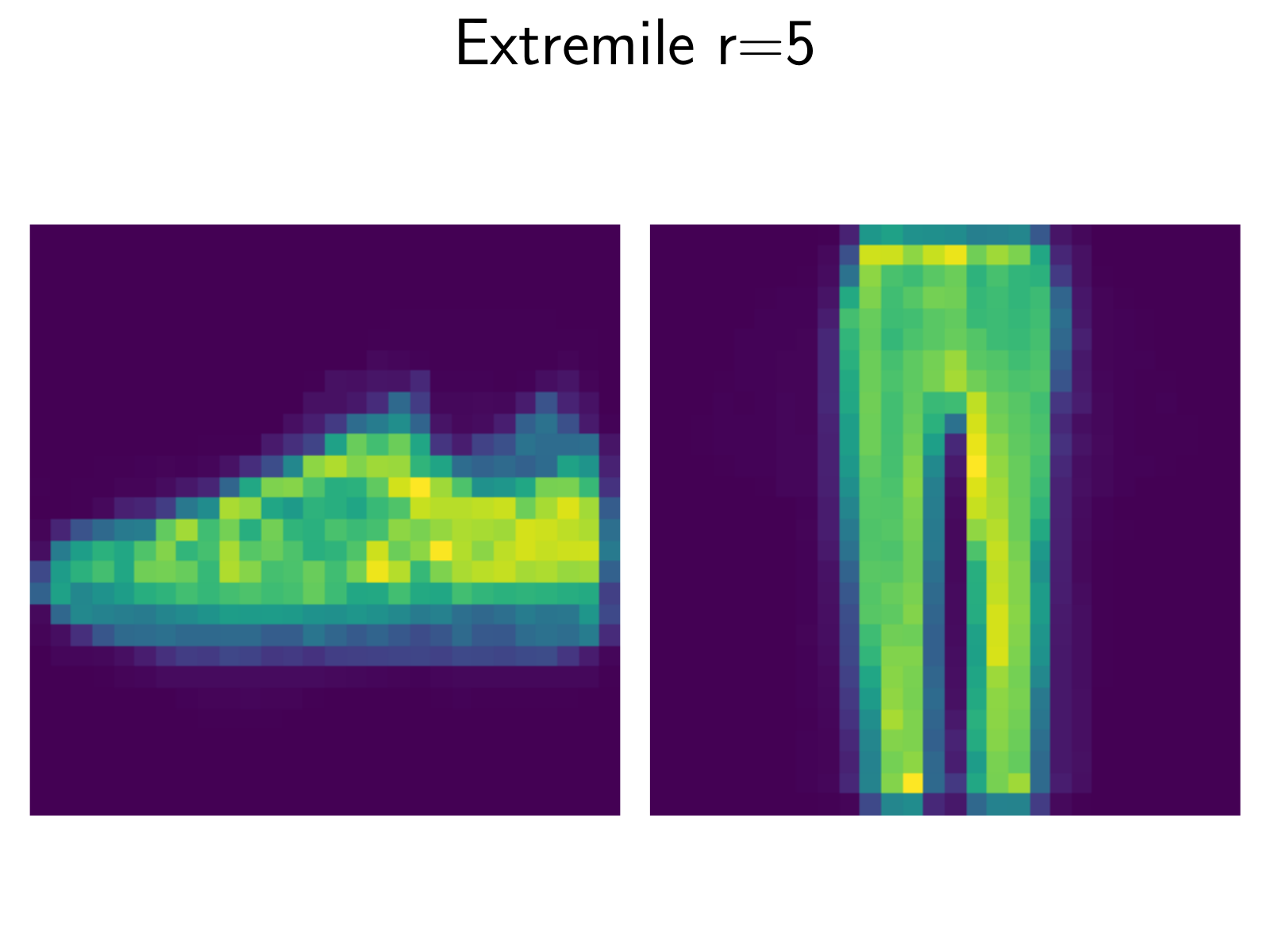}
    \includegraphics[width=0.8\linewidth]{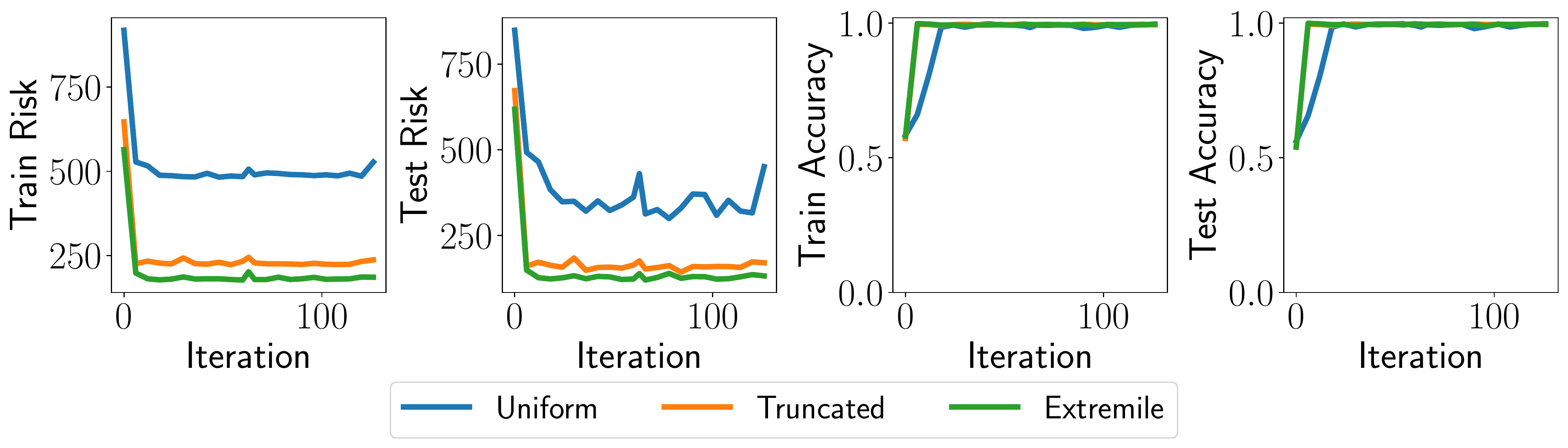}
    \caption{Clustering images of clothes in the presence of outliers.}
    \label{fig:full_fashion_mnist_kmeans}
\end{figure}

\section{Additional Experiments}
\label{sec:a:additional}
\myparagraph{Optimization effect of varying regularization parameter}
We demonstrate the robustness of the algorithm comparison with respect to the \emph{statistical regularization parameter} $\mu$. Hyperparameters are selected in accordance with \Cref{sec:a:hyperparam}. \Cref{fig:training_curves1}, \Cref{fig:training_curves10}, and \Cref{fig:training_curves0.1} show the suboptimality trajectories for $\mu = 1/n$, $10/n$, and $0.1/n$, respectively. The same rankings of algorithms result from each of the three figures, that \osvrg generally outperforms SGD and SRDA.

\begin{figure}[t!]
    \centering
    \includegraphics[width=\linewidth]{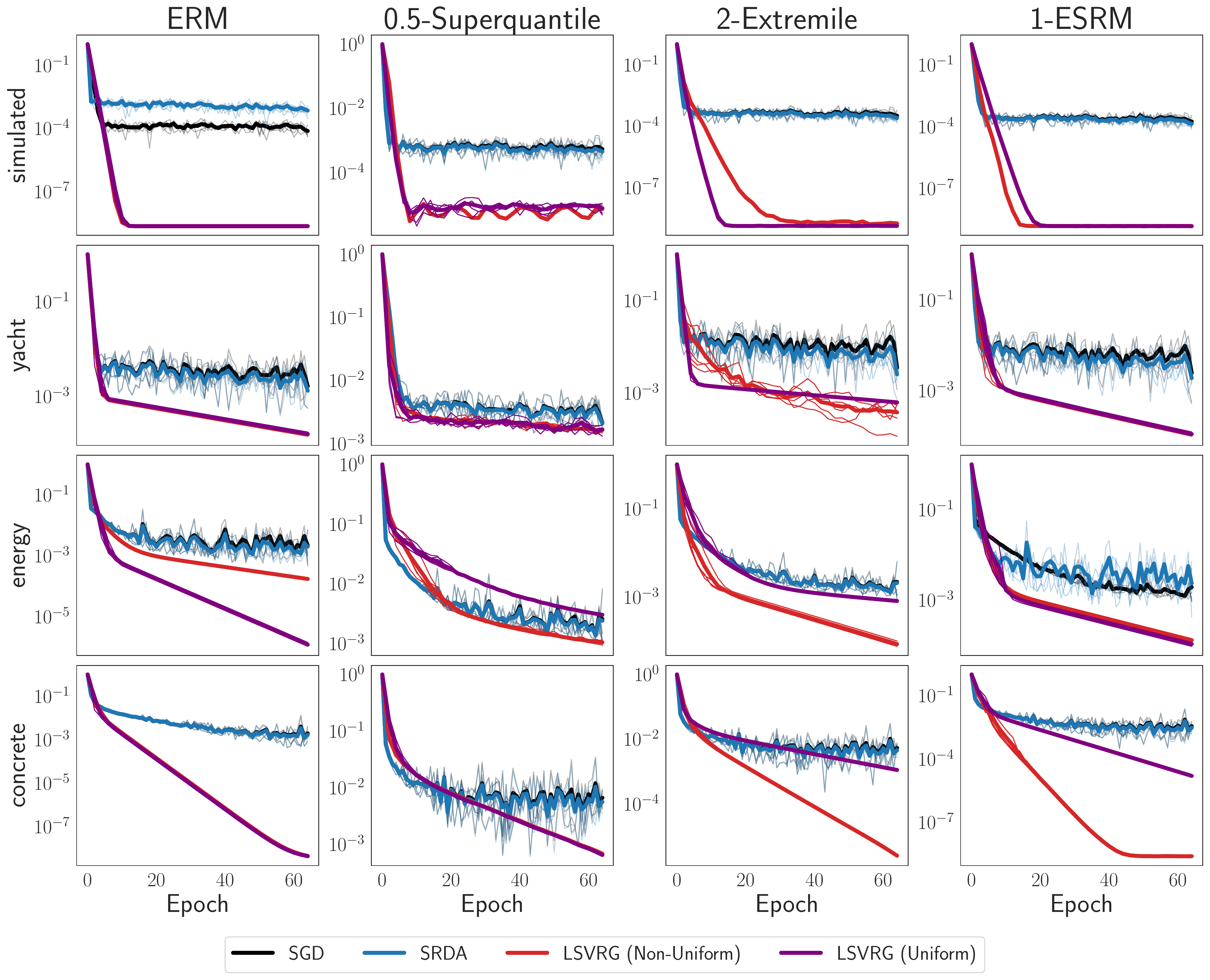}
    \caption{The suboptimality gap (base 10) for various optimization algorithms on spectral risk objectives for $\mu = 1/n$. The $x$-axis shows the number of effective passes through the data.}
    \label{fig:training_curves1}
\end{figure}
\begin{figure}[t!]
    \centering
    \includegraphics[width=\linewidth]{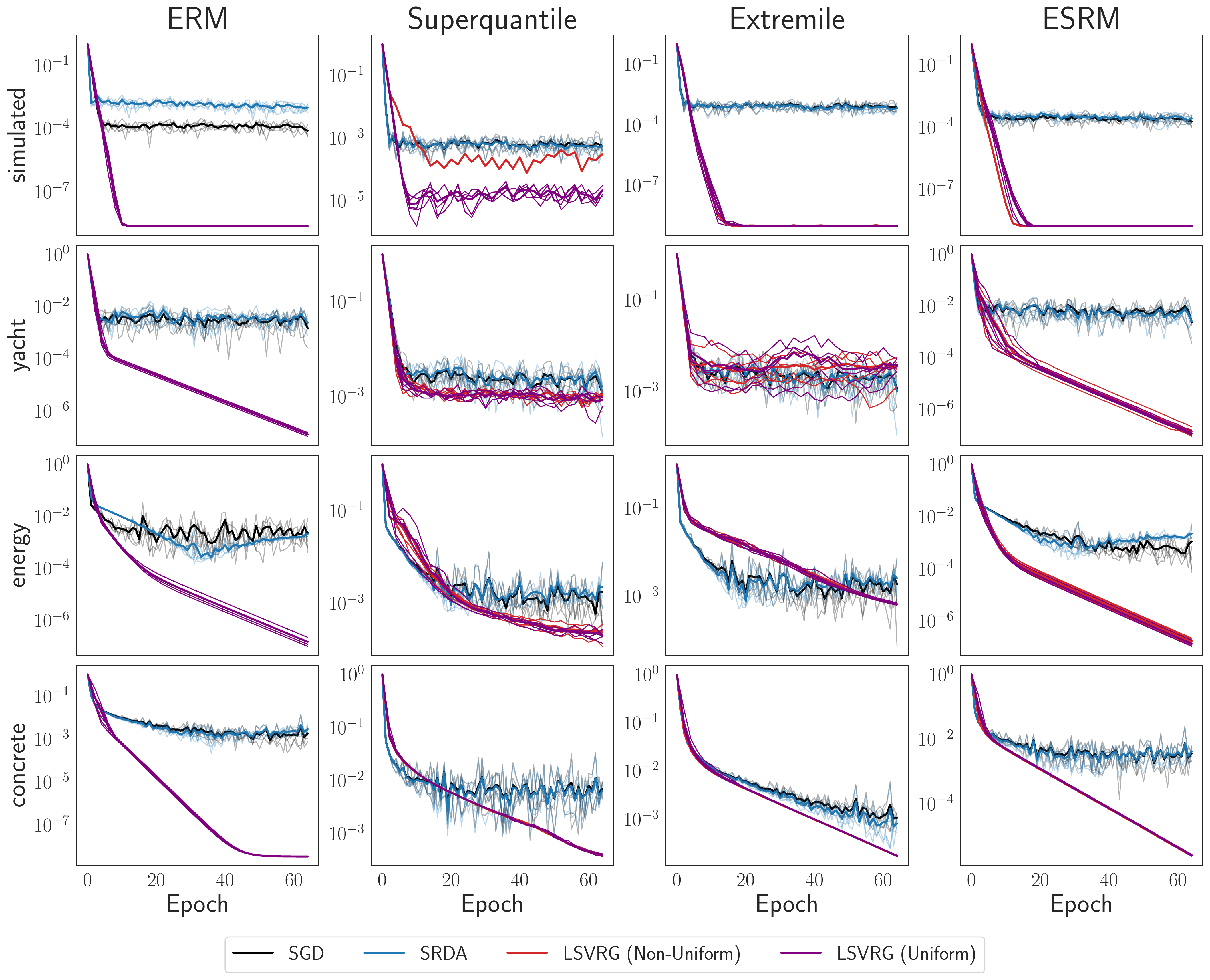}
    \caption{The suboptimality gap (base 10) for various optimization algorithms on spectral risk objectives for $\mu = 10/n$. The $x$-axis shows the number of effective passes through the data.}
    \label{fig:training_curves10}
\end{figure}
\begin{figure}[t!]
    \centering
    \includegraphics[width=\linewidth]{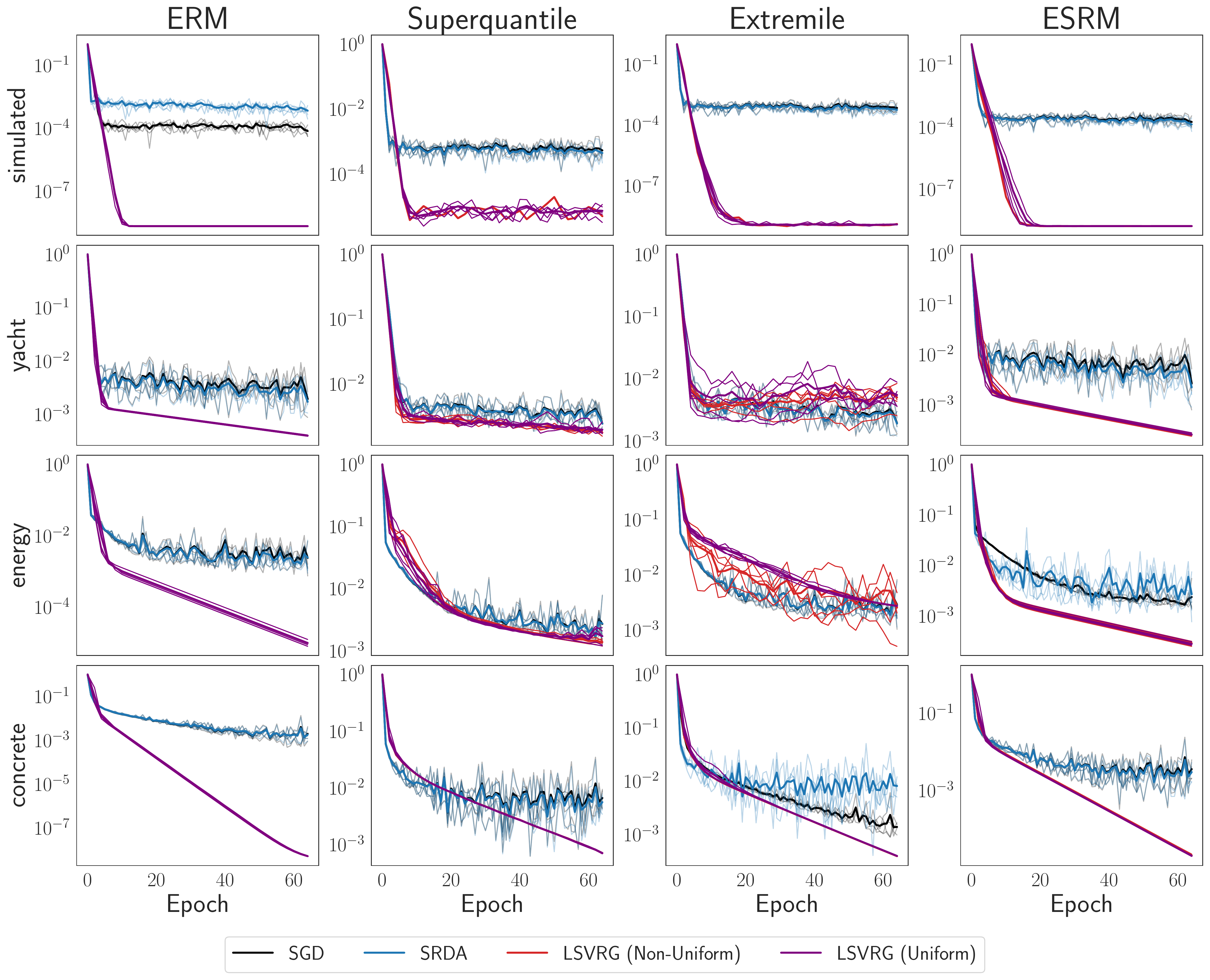}
    \caption{The suboptimality gap (base 10) for various optimization algorithms on spectral risk objectives for $\mu = 0.1/n$. The $x$-axis shows the number of effective passes through the data.}
    \label{fig:training_curves0.1}
\end{figure}

\myparagraph{Optimization effect of varying risk parameter}
We demonstrate the robustness of the algorithm comparison with respect to the statistical regularization parameter $\mu$. Hyperparameters are selected in accordance with \Cref{sec:a:hyperparam}. \Cref{fig:riskopti1}, \Cref{fig:riskopti2}, and \Cref{fig:riskopti3} show the suboptimality trajectories for $(q, r, \rho)$ set to $(0.25, 1.5, 0.5)$, $(0.5, 2, 1)$, and $(0.75, 2.5, 2)$, respectively. The same rankings of algorithms result from each of the three figures, that \osvrg generally outperforms SGD and SRDA. It should be noted that for the $0.75$-superquantile, \osvrg suffers from slow convergence and is outperformed by SGD and SRDA, suggesting that the superquantile is a particularly difficult learning objective.

\begin{figure}
    \centering
    \includegraphics[width=0.9\linewidth]{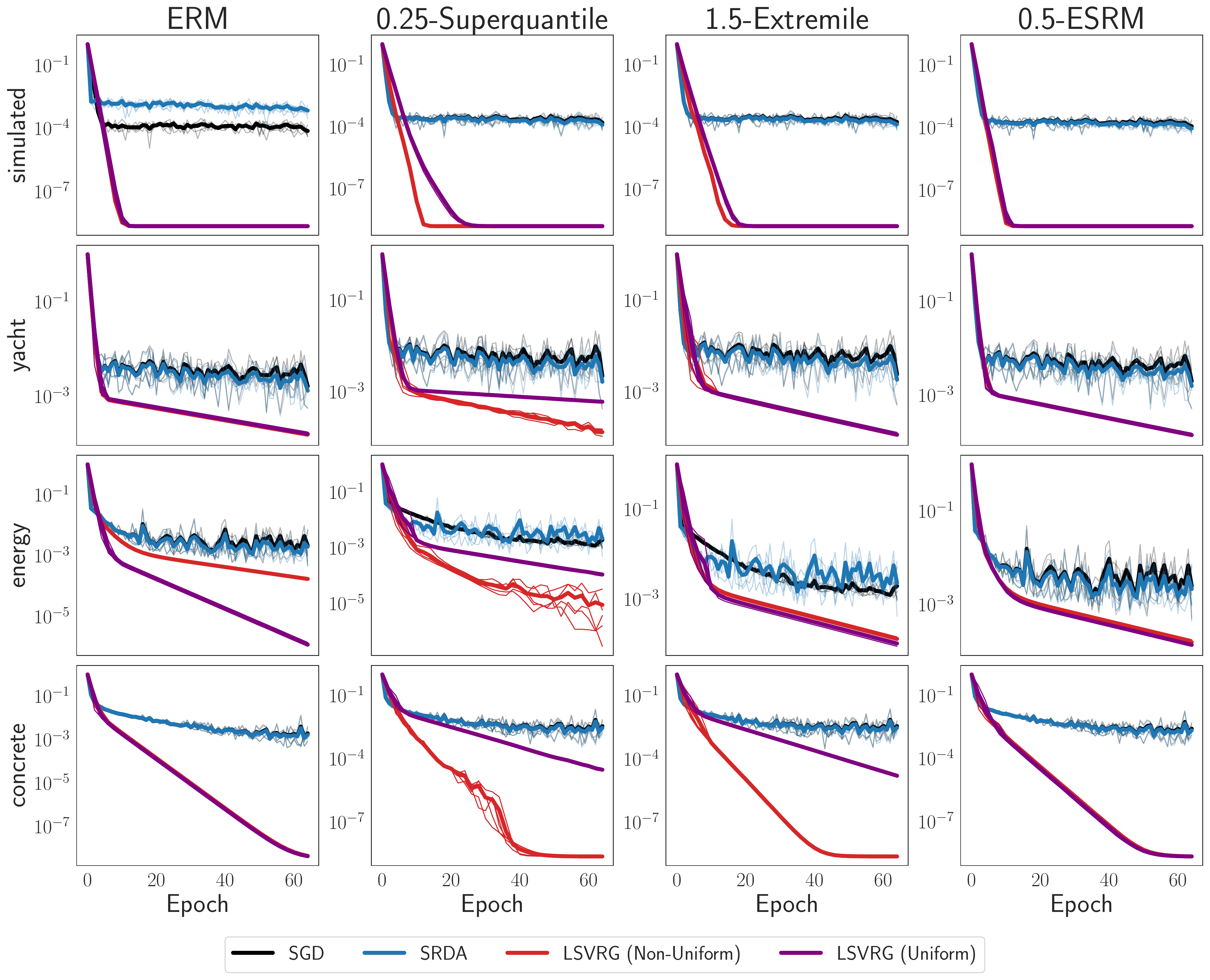}
    \caption{The suboptimality gap (base 10) for various optimization algorithms on ERM, $q$-superquantile, $r$-extremile, and $\rho$-ESRM objectives for$(q, r, \rho)$ set to $(0.25, 1.5, 0.5)$. The $x$-axis shows the number of effective passes through the data.}
    \label{fig:riskopti1}
\end{figure}
\begin{figure}
    \centering
    \includegraphics[width=0.9\linewidth]{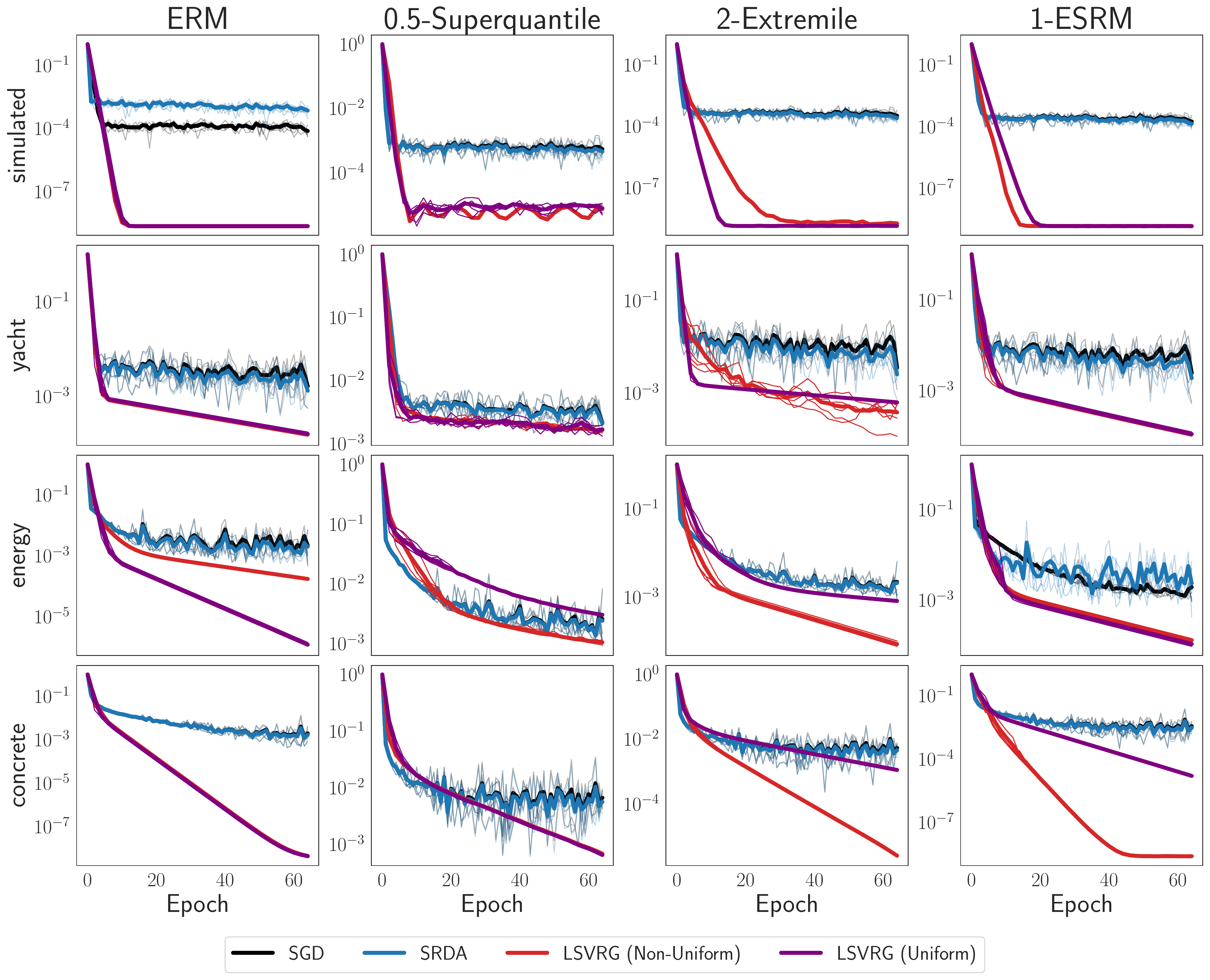}
    \caption{The suboptimality gap (base 10) for various optimization algorithms on ERM, $q$-superquantile, $r$-extremile, and $\rho$-ESRM objectives for$(q, r, \rho)$ set to $(0.5, 2, 1)$. The $x$-axis shows the number of effective passes through the data.}
    \label{fig:riskopti2}
\end{figure}
\begin{figure}
    \centering
    \includegraphics[width=0.9\linewidth]{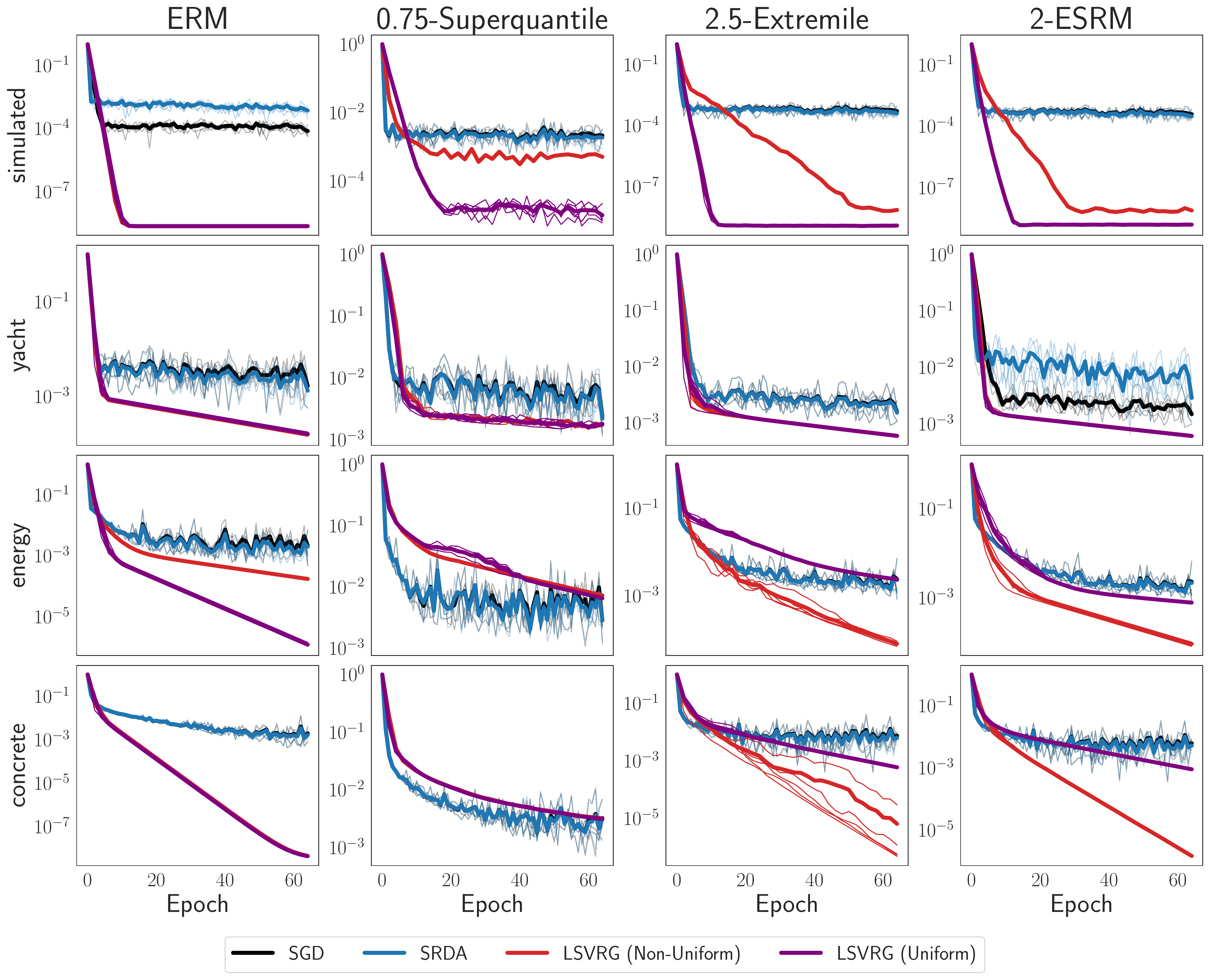}
    \caption{The suboptimality gap (base 10) for various optimization algorithms on ERM, $q$-superquantile, $r$-extremile, and $\rho$-ESRM objectives for$(q, r, \rho)$ set to $(0.75, 2.5, 2)$. The $x$-axis shows the number of effective passes through the data.}
    \label{fig:riskopti3}
\end{figure}

\myparagraph{Statistical effect of varying risk parameter}
We inspect how the test losses of the $L$-risk minimizers behave compared to the corresponding ERM solutions. Letting $\hat{w}_{\text{ERM}}$ be the approximate solution of ERM, whereas $\hat{w}_{\text{LRM}}$ is the approximate solution of an $L$-Risk minimization problem other than ERM, \Cref{fig:robustness1}, \Cref{fig:robustness2}, and \Cref{fig:robustness3} plot the following against $p$:
\begin{equation}
    \ell_{(\ceil{np})}\p{\hat{w}_{\text{ERM}}} - \ell_{(\ceil{np})}\p{\hat{w}_{\text{LRM}}},
    \label{eqn:emp_quantile_diff}
\end{equation}
that is, the difference in the $p$-th quantile of the test loss of $\hat{w}_{\text{ERM}}$ and the $p$-th quantile of the test loss of $\hat{w}_{\text{LRM}}$. The plots are in order of ``easy'', ``medium'', and ``hard'' values of the risk parameters, corresponding to $(q, r, \rho)$ being $(0.25, 1.5, 0.5)$, $(0.5, 2, 1)$, and $(0.75, 2.5, 2)$, respectively. The medium settings are shown primarily in the main text. The median test loss ($p = 0.5$) is similar between the $L$-risk minimizers and standard ERM across risk parameters. However, for $p > 0.5$, the ERM solution can make predictions with much higher loss, indicating that the tail is not controlled. The superquantile at parameters $q = 0.5$ generally fails to control test risk, even substantially underperforms in comparison to ERM in {\tt energy}. On the other hand, the extremile and ESRM convincingly dominate ERM in the region $(0.9, 1)$ of the empirical quantile function for each of the risk parameters, with the extremile having a more pronounced effect.
\begin{figure}
    \centering
    \includegraphics[width=0.8\linewidth]{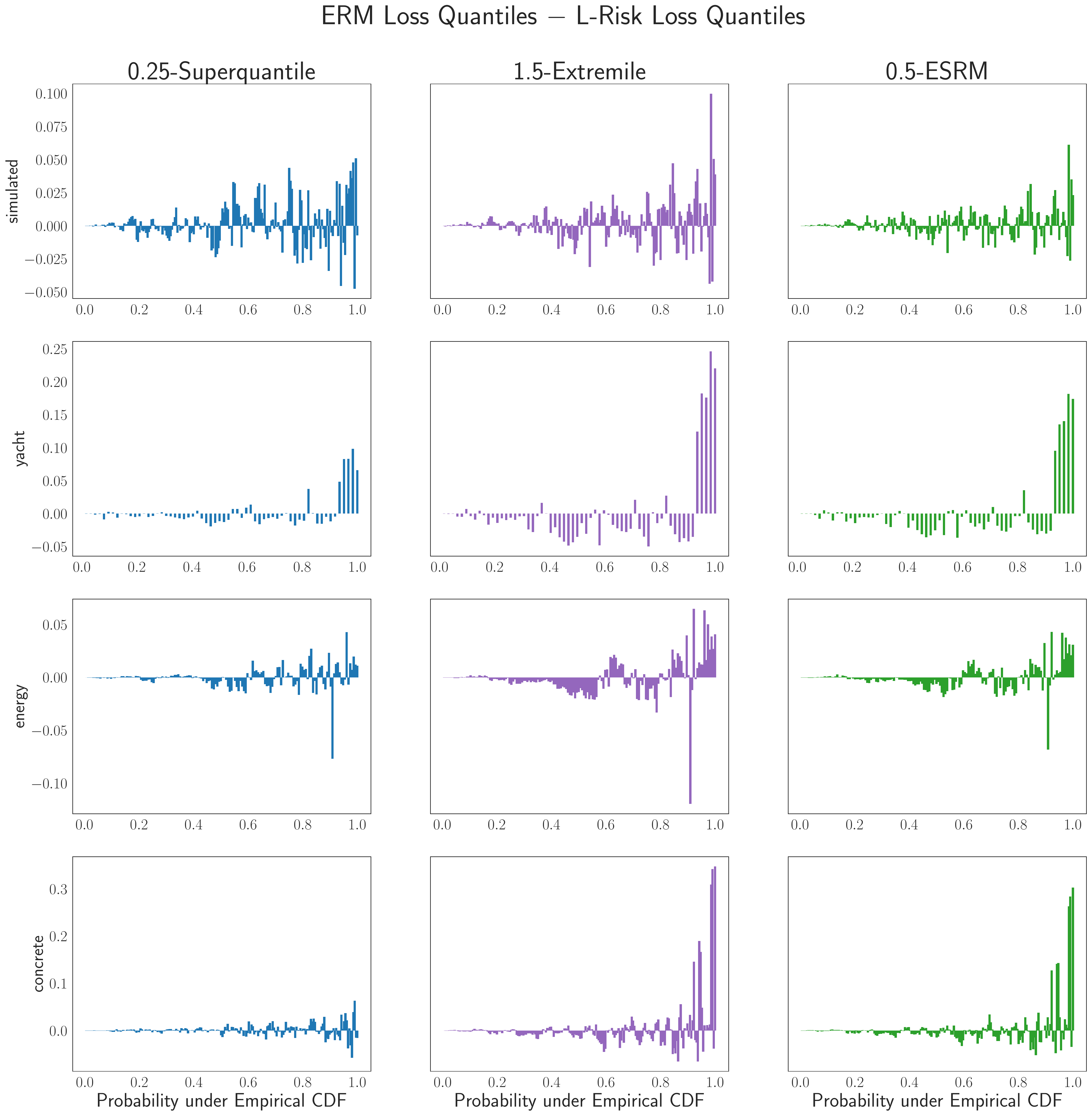}
    \caption{The difference between the empirical quantile function given by $\ell_{(1)}(\hat{w}_{\text{ERM}}), \ldots, \ell_{(n)}(\hat{w}_{\text{ERM}})$ and the empirical quantile function of an $L$-risk minimizer $\ell_{(1)}(\hat{w}_{\text{LRM}}), \ldots, \ell_{(n)}(\hat{w}_{\text{LRM}})$, where the $L$-risk is the $q$-superquantile (left column), $r$-extremile (middle column), or $\rho$-exponential spectral risk measure (right column). Each row represents a dataset out of {\tt simulated}, {\tt yacht}, {\tt energy}, and {\tt concrete}. Here, $(q, r, \rho) = (0.25, 1.5, 0.5)$, constituting $L$-risks that are ``close'' to ERM.}
    \label{fig:robustness1}
\end{figure}
\begin{figure}
    \centering
    \includegraphics[width=0.8\linewidth]{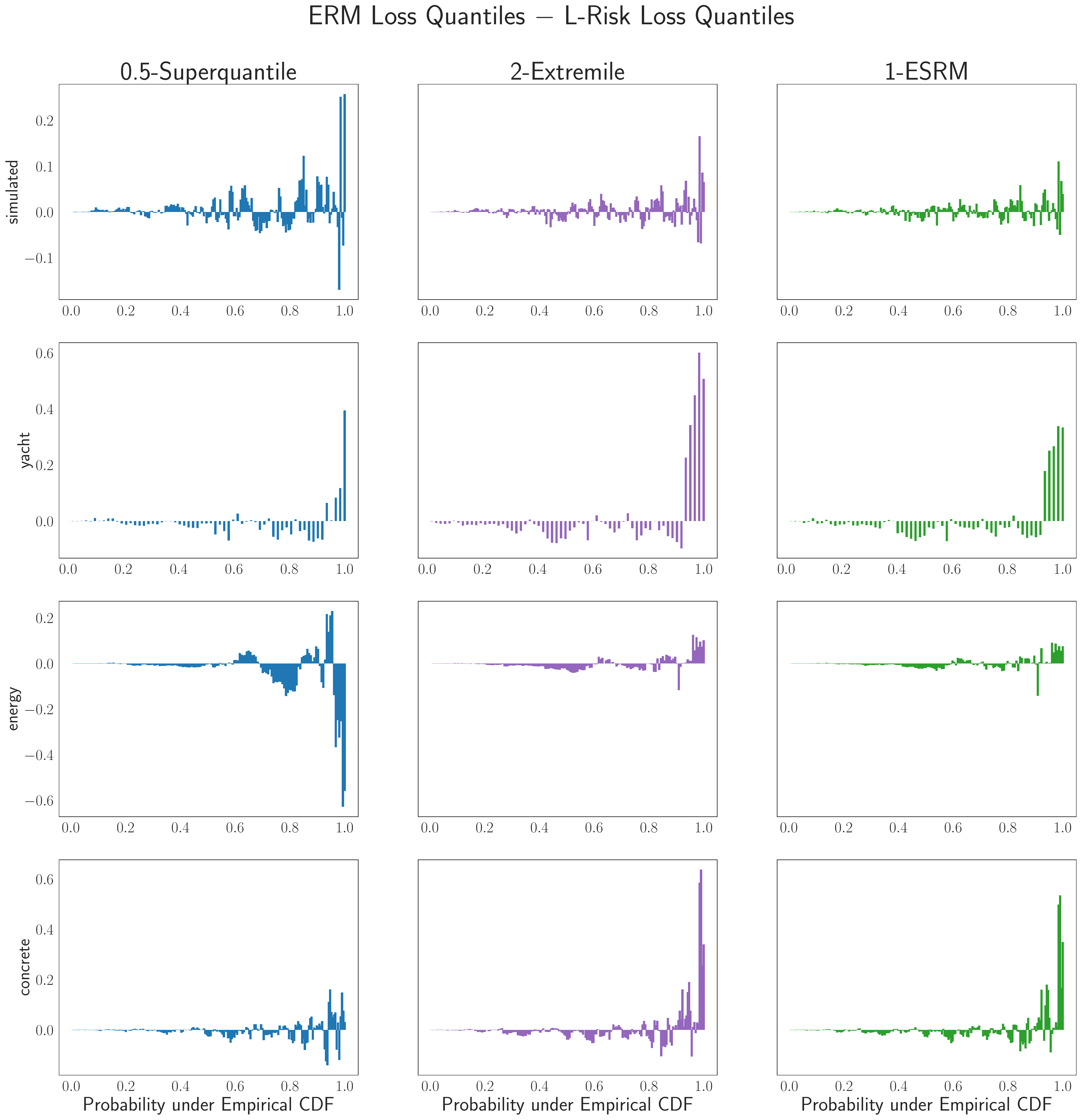}
    \caption{The difference between the empirical quantile function given by $\ell_{(1)}(\hat{w}_{\text{ERM}}), \ldots, \ell_{(n)}(\hat{w}_{\text{ERM}})$ and the empirical quantile function of an $L$-risk minimizer $\ell_{(1)}(\hat{w}_{\text{LRM}}), \ldots, \ell_{(n)}(\hat{w}_{\text{LRM}})$, where the $L$-risk is the $q$-superquantile (left column), $r$-extremile (middle column), or $\rho$-exponential spectral risk measure (right column). Each row represents a dataset out of {\tt simulated}, {\tt yacht}, {\tt energy}, and {\tt concrete}. Here, $(q, r, \rho) = (0.5, 2, 1)$, constituting $L$-risks that are ``moderately far'' from ERM.}
    \label{fig:robustness2}
\end{figure}
\begin{figure}
    \centering
    \includegraphics[width=0.8\linewidth]{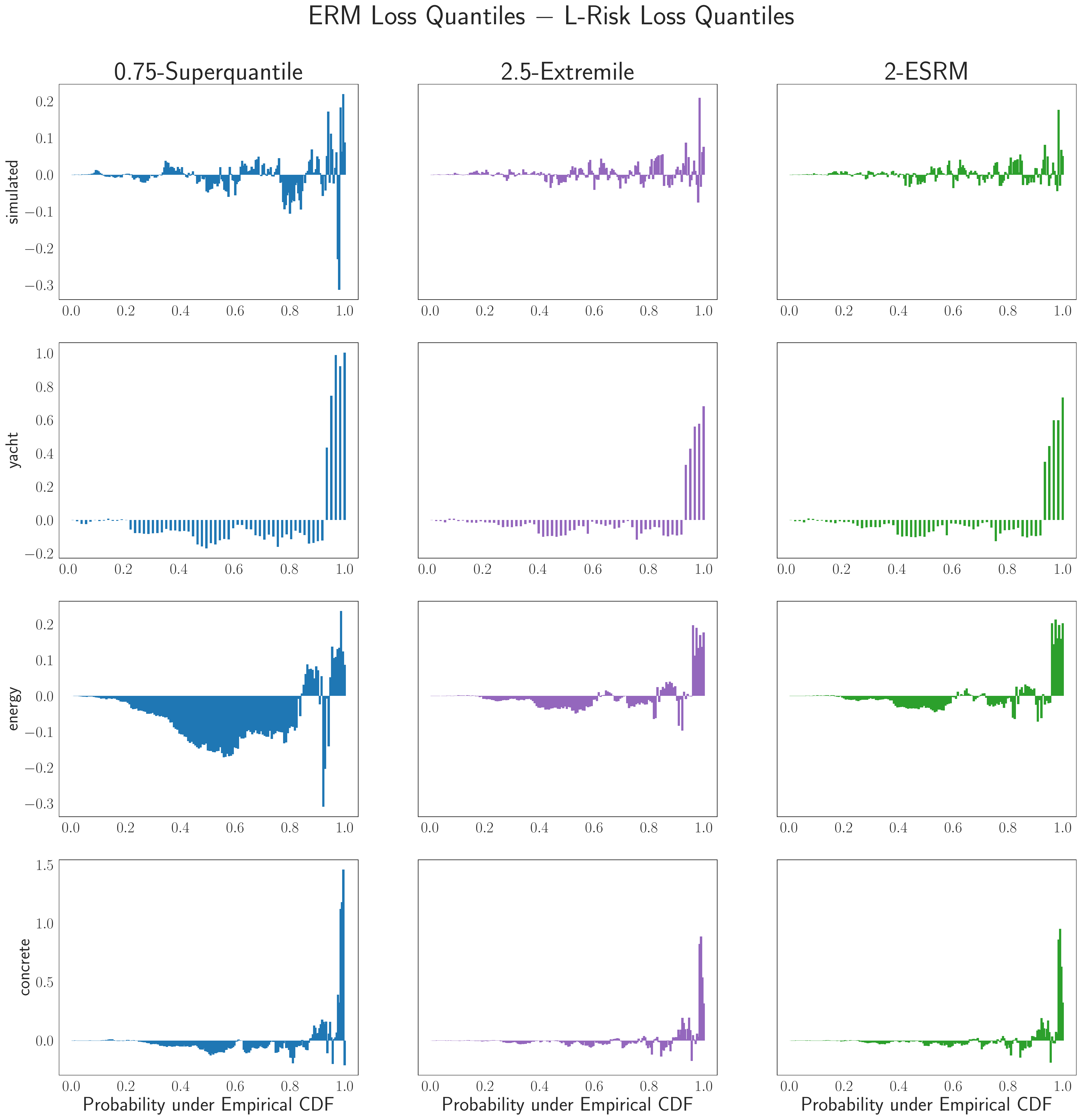}
    \caption{The difference between the empirical quantile function given by $\ell_{(1)}(\hat{w}_{\text{ERM}}), \ldots, \ell_{(n)}(\hat{w}_{\text{ERM}})$ and the empirical quantile function of an $L$-risk minimizer $\ell_{(1)}(\hat{w}_{\text{LRM}}), \ldots, \ell_{(n)}(\hat{w}_{\text{LRM}})$, where the $L$-risk is the $q$-superquantile (left column), $r$-extremile (middle column), or $\rho$-exponential spectral risk measure (right column). Each row represents a dataset out of {\tt simulated}, {\tt yacht}, {\tt energy}, and {\tt concrete}. Here, $(q, r, \rho) = (0.75, 2.5, 2)$, constituting $L$-risks that are ``significantly far'' from ERM.}
    \label{fig:robustness3}
\end{figure}

\myparagraph{Comparison between smoothed and non-smooth \osvrg}
We compare the implementation of \osvrg with smoothing presented in \cref{alg:osvrg-u:theory-1} to the non-smooth epoch-based implementation of \osvrg presented in \cref{alg:extr_svrg}.
We consider the datasets \texttt{simulated}, \texttt{yacht}, \texttt{energy} and \texttt{concrete} presented in \cref{sec:a:task} and spectral risk measure objectives~\eqref{eqn:orm} defined by the empirical superquantile ($q = 0.5$), extremile ($r = 2$), and ESRM ($\rho=1$) of the losses, plus an $\ell_2^2$ regularization term of magnitude $1/n$.

We implemented the smoothed \osvrg algorithm (\cref{alg:osvrg-u:theory-1}) with $N=n$, $q^*=0$ and a smoothing given by either a centered negative entropy regularizer $\Omega_1$ or a centered square Euclidean norm $\Omega_2$, with $\Omega_1$ and $\Omega_2$ from Eq.~\eqref{eq:chosen_smoothing}. We consider using a smoothing coefficient of $\nu_1 = 10^{-3}$ for $\Omega_1$ and $\nu_2 = n10^{-3}$ for $\Omega_2$ (using the fact that the approximation done by $\Omega_2$ has an approximation error of $\chi^2(s||u)/n$ as detailed in \cref{sec:a:smoothing}). On the vertical axis we consider is the suboptimality gap  $\frac{\mc{R}_\sigma(w\pow{t}) - \mc{R}_\sigma(w^*)}{\mc{R}_\sigma(w\pow{0}) - \mc{R}_\sigma(w^*)}$ for $w^*$ computed by L-BFGS. 

In \cref{fig:smooth_osvrg_conv}, we observe that the non-smooth and smooth implementations of \osvrg generally match. For the ERM objective, this observation was expected since the permutahedron associated with the vector $\avg = \ones/n$ reduces to $\{\avg\}$ since all entries of $\avg$ are equal. Hence the maximization defining the smooth approximations $\lstat_{\nu \Omega}$ given in \cref{sec:a:smoothing} have a maximizer independent of the values of the losses and naturally given by $\avg$ such that the smooth approximation of $\lstat$ reduces exactly to $\lstat$ for any choice of $\nu$ and $\Omega$. 
For the other spectral risk measures, we observe some discrepancies between the non-smooth and the smooth implementations with the smooth implementation giving generally smoother curves as it is the case for the superquantile on the \texttt{simulated} dataset or the ESRM on the \texttt{concrete} dataset. However, such differences are not observed for, e.g., the superquantile on the \texttt{yacht}, \texttt{energy}, \texttt{concrete} datasets or the extremile and the ESRM objectives on the \texttt{simulated} and \texttt{yacht} datasets. Overall these experiments suggest that the non-smooth nature of the problem has moderate impact on the performance of \osvrg. This behavior may be explained by the fact that the non-smoothness of the losses only intervene if the minimizer of the objective produces a vector of losses with ties which may not happen in practice. In addition note, that the negative entropy or the squared Euclidean smoothing generally give the same results (after appropriately scaling the smoothing coefficient of $\Omega_2$ by $n$ as suggested by the approximation errors given in~\Cref{cor:smooth_approx} (\Cref{sec:a:smoothing}).

In \cref{fig:smooth_osvrg_conv}, we also consider \cref{alg:osvrg-u:theory-1} with $N= 2n$, $q^*=1/n$ and the same smoothing method as presented above. We scaled the horizontal axis by multiplying all algorithms by the total  number of calls to the gradient oracles of the losses such that \osvrg in \cref{alg:extr_svrg} is scaled by a factor 2 while \cref{alg:osvrg-u:theory-1} is scaled by a factor $\rho\geq 2$. We observe that the non-smooth implementation of \osvrg in \cref{alg:extr_svrg} compares generally on par or better than the implementation of \cref{alg:osvrg-u:theory-1} after taking into account the total number of passes over the data, except for the ESRM risk on \texttt{concrete} and the extremile on \texttt{yacht}.

\begin{figure}
	\begin{center}
    \includegraphics[width=0.8\linewidth]{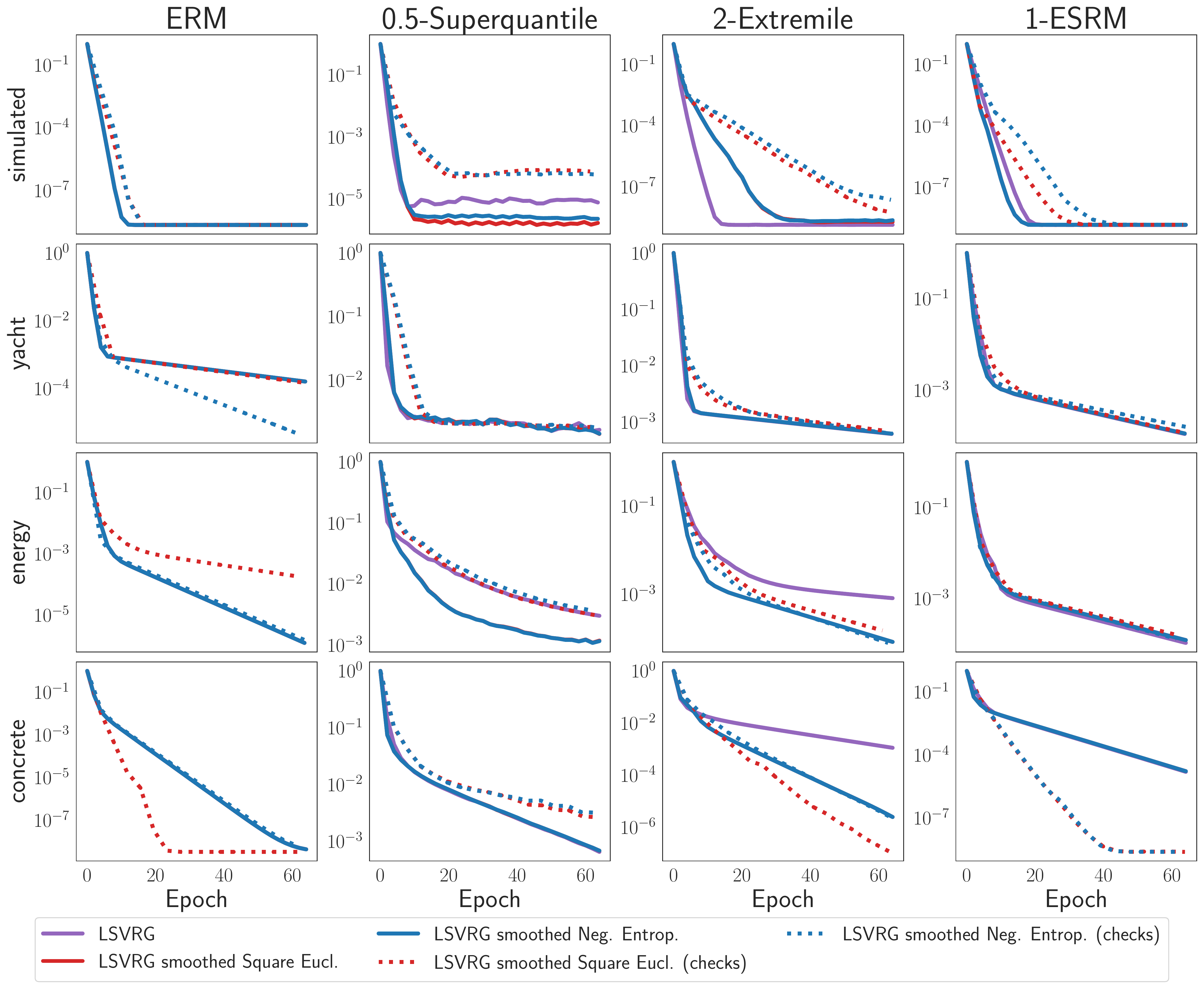}
	\end{center}
\caption{Comparison of the non-smooth implementation of \osvrg in \cref{alg:extr_svrg} and the smoothed implementation of \osvrg in \cref{alg:osvrg-u:theory-1} with $N=n$, $q^*=0$, and with a centered non-negative entropy smoothing function $\Omega_1$ and $\nu_1 = 10^{-3}$ or a centered Euclidean smoothing function $\Omega_2$ with $\nu_2 = n 10^{-3}$ (see \cref{eq:chosen_smoothing} for the exact definitions of $\Omega_1, \Omega_2$).\label{fig:smooth_osvrg_conv}}
\end{figure}

\myparagraph{Run time experiments}
\Cref{fig:runtime} contains plots of optimizer runtimes in each of the datasets considered. The values are calculated using the {\tt time} module in Python 3 with logging disabled on the compute environment described in \Cref{sec:a:code}. The two variants of \osvrg trade off run time for precision, as their suboptimality achieves $\sim 1$ order of magnitude improvement on {\tt yacht}, up to $\sim 4$ orders of magnitude improvement on {\tt concrete} over the SGD and SRDA baseline. SGD and SRDA also run $\sim 2$ orders of magnitude faster across datasets, but fail to converge due to both bias and variance.

\begin{figure}
    \centering
    \includegraphics[width=\linewidth]{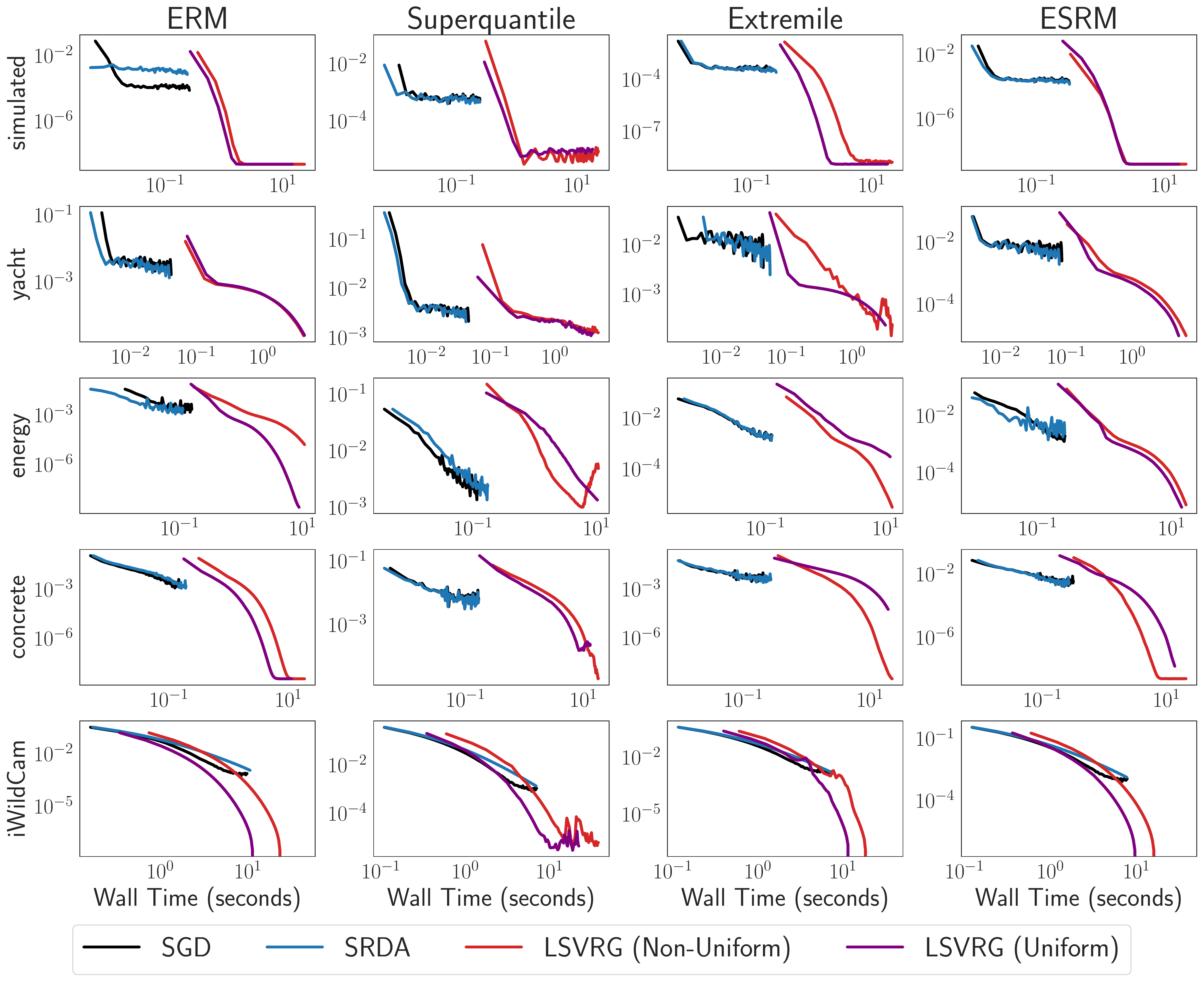}
    \caption{Algorithm run time of SGD, SRDA, and LSVRG optimizers on fives datasets (rows) and four objectives (columns). The $y$-axis plots the suboptimality in log scale, whereas the $x$-axis contains wall time in seconds in log scale.}
    \label{fig:runtime}
\end{figure} 

\end{document}